\documentclass{article}

\usepackage{microtype}
\usepackage{graphicx}
\usepackage{subcaption}
\usepackage{booktabs} 
\usepackage{multirow} 
\usepackage[table]{xcolor}
\usepackage{multirow}
\usepackage{booktabs}
\usepackage{graphicx}
\usepackage{amsmath}
\usepackage{fontawesome5}
\usepackage{wrapfig} 

\usepackage{fontawesome5}
\usepackage{hyperref}

\usepackage[preprint]{icml2026}

\usepackage{amsmath}
\usepackage{amssymb}
\usepackage{mathtools}
\usepackage{amsthm}

\usepackage[capitalize,noabbrev]{cleveref}

\theoremstyle{plain}
\newtheorem{theorem}{Theorem}[section]
\newtheorem{proposition}[theorem]{Proposition}
\newtheorem{lemma}[theorem]{Lemma}

\theoremstyle{definition}
\newtheorem{definition}[theorem]{Definition}
\newtheorem{assumption}[theorem]{Assumption}
\theoremstyle{remark}

\usepackage{tcolorbox}
\tcbuselibrary{skins, breakable, theorems}
\usepackage{soul}

\usepackage[textsize=tiny]{todonotes}

\definecolor{myred}{RGB}{156, 39, 33}
\definecolor{myblue}{RGB}{31, 90, 153}

\definecolor{TakeawayBg}{HTML}{F4F5FF} 
\definecolor{TakeawayLabel}{HTML}{666666} 
\definecolor{TakeawayText}{HTML}{000000}  

\sethlcolor{TakeawayBg}
\definecolor{CorrectGreen}{HTML}{2EA121}
\definecolor{WrongRed}{HTML}{D12121}
\tcbset{
  takeaway/.style={
    enhanced,
    colback=TakeawayBg,
    colframe=TakeawayText,
    arc=2mm,
    boxrule=0.8pt,
    left=3mm,right=3mm,
    top=2mm,bottom=2mm,
    fonttitle=\bfseries,
    coltitle=white,
    attach boxed title to top left={
      yshift=-2mm,
      xshift=3mm
    },
    boxed title style={
      colback=TakeawayLabel,
      colframe=TakeawayLabel,
      arc=1mm,            
      left=2mm,right=2mm,
      top=0.1mm,bottom=0.1mm
    },
  }
}
\icmltitlerunning{SFT Conflicts, RL Coexists: A Theoretical and Empirical Analysis of Multi-Task Learning for LLMs}

\begin{document}

\twocolumn[
  \icmltitle{SFT Conflicts, RL Coexists: A Theoretical and Empirical \\ Analysis of Multi-Task Learning Paradigms for LLMs}



  \icmlsetsymbol{equal}{*}





  \begin{icmlauthorlist}
    \icmlauthor{Kejian Zhu}{cas,ucas}
    \icmlauthor{Zhuoran Jin}{cas,ucas}
    \icmlauthor{Shangqing Tu}{thu}
    \icmlauthor{Hongbang Yuan}{cas,ucas}
    \icmlauthor{Yushi Bai}{thu} \\
    \icmlauthor{Kang Liu}{cas,ucas}
    \icmlauthor{Juanzi Li}{thu}
    \icmlauthor{Jun Zhao}{cas,ucas}
\end{icmlauthorlist}

\icmlaffiliation{cas}{The Key Laboratory of Cognition and Decision Intelligence for Complex Systems, Institute of Automation, Chinese Academy of Sciences, Beijing, China}
\icmlaffiliation{ucas}{School of Artificial Intelligence, University of Chinese Academy of Sciences, Beijing, China}
\icmlaffiliation{thu}{Tsinghua University, Beijing, China}

\icmlcorrespondingauthor{Jun Zhao}{jzhao@nlpr.ia.ac.cn}

  \icmlkeywords{Machine Learning, ICML}

  \vskip 0.3in
]



\printAffiliationsAndNotice{}  

\begin{abstract}

Supervised Fine-Tuning (SFT) and Reinforcement Learning (RL) exhibit fundamentally different behaviors in enhancing multi-task reasoning for large language models (LLMs). Our preliminary experiments revealed a phenomenon: SFT suffers from severe task conflicts under multi-stage training, whereas RL enables stable coexistence across diverse tasks. Empirically, we trace this to the parameter level, observing that RL induces sparse and approximately orthogonal updates across tasks. We provide a theoretical explanation for this mechanism by analyzing multi-task gradient interference. Our results reveal a distinction: interference in SFT is norm-limited, scaling with the absolute gradient magnitude, whereas interference in RL is variance-limited, bounded by the gradient variance induced by advantage normalization and on-policy optimization. This small variance bound yields near-orthogonal optimization directions across tasks. Leveraging this insight, we propose Parallel-RL, a paradigm that decouples multi-task training, significantly improving efficiency and flexibility. 
\href{https://github.com/GaryStack/Parallel-RL}
{\faGithub\ \textbf{Code}}

\end{abstract}

\section{Introduction}


Enhancing the reasoning capabilities of large language models (LLMs) is a central research focus, with Supervised Fine-Tuning (SFT) and Reinforcement Learning (RL) as the dominant paradigms~\cite{jaech2024openai,guo2025deepseek}. Recent works have explored their differences in single-task training~\cite{chu2025sft, matsutani2025rl}. But their behaviors in multi-task settings, which is also crucial for achieving Artificial General Intelligence (AGI), remain underexplored. We observe that in most existing works, researchers typically construct mixed datasets from multiple tasks for SFT~\cite{dong2023abilities, park2025instruct}. In contrast, for RL, apart from using mixed-data training, many works commonly adopts multi-stage training, where each stage focuses on a single task~\cite{narvekar2020curriculum, cho2024hard}. To explain this divergence, we conduct a comprehensive \textbf{empirical and theoretical analysis} comparing SFT and RL in multi-task training.

\begin{figure}[t] 
    \centering
  \includegraphics[width=\linewidth]{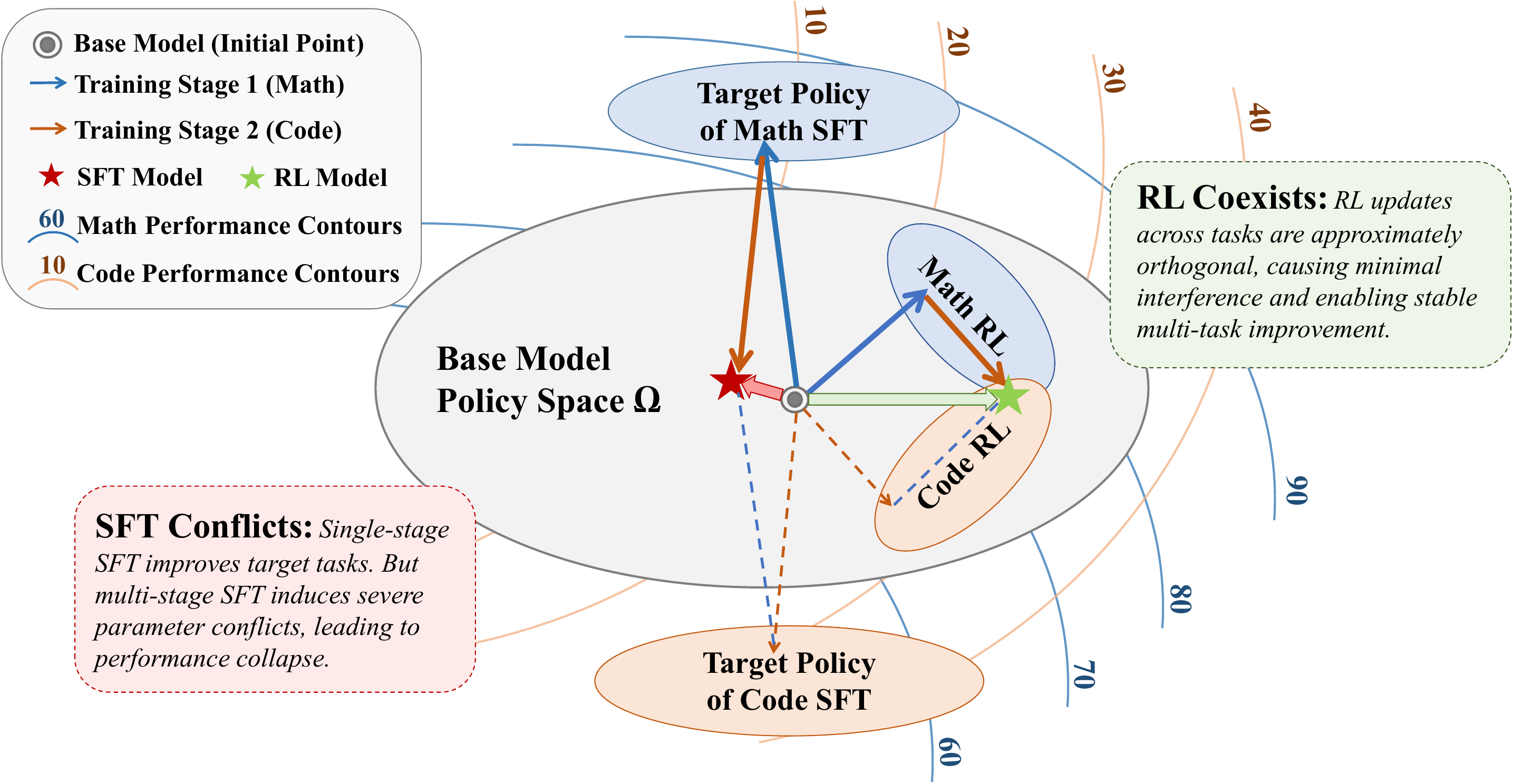}
  \caption{SFT and RL exhibit significant differences during \textbf{multi-stage} training. SFT shows task conflicts, which impairs performance, while RL can steadily improve performance across tasks.}
  \label{fig:intro}
  \vspace{-8pt}
\end{figure}

We first tested the performance differences between SFT and RL when using mixed-data and multi-stage strategy for multi-task training. Table~\ref{tab:dif-strategy} reveals a striking contrast: SFT suffers from performance collapses during multi-stage training, showing an average decline of 23.1\% compared to the base model. Conversely, RL demonstrates robustness improvement (with gains of 24.9\%) in multi-stage settings. Therefore, one of the key questions of our research is this significant difference: \textbf{why SFT collapses under multi-stage training whereas RL achieves stable, cumulative performance growth on different tasks across multi-stage training?} To answer this question, we examined the impact of single-task training on untrained tasks. As shown in Table \ref{tab:dif-task-generalization}, while SFT improves the target task, it causes substantial degradation in others. Conversely, RL enhances the target task while preserving performance on untrained tasks. We term this phenomenon \textbf{SFT Conflicts} and \textbf{RL Coexists}. This distinction in generalization directly answers our inquiry. As illustrated in Figure \ref{fig:intro}, although SFT parameters updates are toward improving performance of each single task during single-stage training, severe conflicts arise across multi-stages, leading to performance collapse in the final multi-stage SFT model. In contrast, parameter updates induced by RL across different tasks are approximately orthogonal and can therefore coexist, enabling multi-stage RL to stably improve performance across tasks over stages.




To understand the mechanisms behind these phenomena, we conduct an \textbf{empirical analysis} of the parameter updates $\Delta W$, as detailed in Section \ref{sec:parameter-level-empirical}. As illustrated in Figure \ref{fig:intro}, we highlight two observations. First, regarding the \textit{magnitude}, RL induces significantly smaller updates, with an average magnitude over two orders of magnitude smaller than that of SFT. Second, regarding the interference between different $\Delta W$, the similarity of $\Delta W$ across distinct tasks is negligible for RL ($\sim 10^{-5}$) but substantial for SFT ($\sim 10^{-1}$).


We then provide a theoretical analysis in Section \ref{sec:theoretical_analysis} to explain these observations by examining their respective gradient descent directions. Specifically, the two key differences between their gradients lie in the \textbf{Advantage Function} and \textbf{Policy Source}. Thus in Sections \ref{sec:sparsity_magnitude} and \ref{sec:orthogonality}, we qualitatively analyze how the advantage function and the on-policy nature of RL drive the observed sparsity of $\Delta W$ and small interference between different tasks. Furthermore, in Section \ref{sec:upper bound}, we derive the upper bound of gradient interference between tasks. Our proofs demonstrate that SFT interference is \textbf{norm-limited} and depends on the absolute gradient magnitude. Conversely, RL interference is \textbf{variance-limited} and is bounded by the variance of multiple rollouts (caused by both advantage function and on-policy). This theorem clarifies the distinct interference levels in multi-task settings and reflects the fundamental objectives of SFT and RL.


Building on the insight that RL optimization directions across tasks exhibit approximate orthogonality, we propose \textbf{Parallel-RL}, a novel paradigm for multi-task RL. In this framework, tasks are trained independently in parallel and their updates are subsequently merged. Experiments show that Parallel-RL matches single-task performance while improving training efficiency, and ablation studies further confirm its ability to effectively decouple task-specific capabilities. Guided by our analysis, we note that Parallel-RL goes beyond model merging and additionally requires techniques for task selection and single-task training.



\vspace{-5pt}
\section{Preliminaries}

In this section, we aim to provide preliminary evidence revealing the fundamental differences between SFT and RL in multi-task training paradigms, and to demonstrate their generalization on other tasks when trained on single task.

\subsection{Experiment Setup}
\label{sec:setup for preliminaries}

We employ DeepSeek-R1-Distill-Qwen-1.5B as our base model. Our experiments cover four representative reasoning domains: math, science, coding, and logic. For RL, we adopt the GRPO algorithm \cite{guo2025deepseek}. We utilized LoRA to facilitate efficient fine-tuning and analysis for both RL and SFT. Here we select MATH500, MMLU, Knights \& Knaves, and LiveCodeBench as benchmarks for Math, Science, Logic, and Code tasks, respectively. Implementation details are provided in Appendix \ref{appendix:experiment details}.

\subsection{Differences of Training Paradigms}

There are two paradigms for multi-task learning: \textbf{Mixed-Data}, which aggregates datasets from diverse tasks for joint training, and \textbf{Multi-Stage}, where the model sequentially learns distinct tasks in separate phases.
For SFT, Mixed-Data training has become the consensus approach to mitigate forgetting~\cite{dong2023abilities, park2025instruct}. In contrast, while also utilizing Mixed-Data, RL often follows a Multi-Stage paradigm~\cite{narvekar2020curriculum, cho2024hard}. To understand this difference, we conducted an experiment comparing two strategies, with results presented in Table~\ref{tab:dif-strategy}.

\begin{table}[!htbp]
\caption{Accuracy (\%) of SFT and RL with different strategies. Superscripts denote performance changes relative to the Base Model.}
\vspace{0.1in}
\scriptsize 
\centering
\renewcommand{\arraystretch}{1.6} 

\resizebox{\linewidth}{!}{
\begin{tabular}{lcccc}
\toprule[1pt]

\scriptsize \textbf{Strategy} & \scriptsize \textbf{Math} & \scriptsize \textbf{Science} & \scriptsize \textbf{Logic} & \scriptsize \textbf{Code} \\ \hline
\noalign{\vspace{-3pt}}
\multicolumn{5}{c}{\emph{Baseline}} \\
\noalign{\vspace{-3pt}}
Base Model  & 
$83.1_{\raisebox{-0.5pt}{\tiny $\pm$1.6}}$ & 
$34.9_{\raisebox{-0.5pt}{\tiny $\pm$1.1}}$ & 
$31.0_{\raisebox{-0.5pt}{\tiny $\pm$2.8}}$ & 
$15.0_{\raisebox{-0.5pt}{\tiny $\pm$2.3}}$ \\ \hline

Mixed Data SFT  & 
$84.6_{\raisebox{-0.5pt}{\tiny $\pm$1.6}}^{\raisebox{0.5pt}{\tiny \textcolor[HTML]{00008B}{$\uparrow$1.5}}}$ & 
$38.9_{\raisebox{-0.5pt}{\tiny $\pm$1.1}}^{\raisebox{0.5pt}{\tiny \textcolor[HTML]{00008B}{$\uparrow$4.0}}}$ & 
$34.0_{\raisebox{-0.5pt}{\tiny $\pm$2.5}}^{\raisebox{0.5pt}{\tiny \textcolor[HTML]{00008B}{$\uparrow$3.0}}}$ & 
$16.0_{\raisebox{-0.5pt}{\tiny $\pm$2.2}}^{\raisebox{0.5pt}{\tiny \textcolor[HTML]{00008B}{$\uparrow$1.0}}}$ \\

Multi Stage SFT & 
$78.2_{\raisebox{-0.5pt}{\tiny $\pm$1.7}}^{\raisebox{0.5pt}{\tiny \textcolor[HTML]{8B0000}{$\downarrow$4.9}}}$ & 
$31.1_{\raisebox{-0.5pt}{\tiny $\pm$1.2}}^{\raisebox{0.5pt}{\tiny \textcolor[HTML]{8B0000}{$\downarrow$3.8}}}$ & 
$9.0_{\raisebox{-0.5pt}{\tiny $\pm$2.1}}^{\raisebox{0.5pt}{\tiny \textcolor[HTML]{8B0000}{$\downarrow$22.0}}}$  & 
$14.3_{\raisebox{-0.5pt}{\tiny $\pm$2.4}}^{\raisebox{0.5pt}{\tiny \textcolor[HTML]{8B0000}{$\downarrow$0.7}}}$ \\ \hline

Mixed Data RL  & 
$85.2_{\raisebox{-0.5pt}{\tiny $\pm$1.5}}^{\raisebox{0.5pt}{\tiny \textcolor[HTML]{00008B}{$\uparrow$2.1}}}$ & 
$43.2_{\raisebox{-0.5pt}{\tiny $\pm$1.1}}^{\raisebox{0.5pt}{\tiny \textcolor[HTML]{00008B}{$\uparrow$8.3}}}$ & 
$37.0_{\raisebox{-0.5pt}{\tiny $\pm$2.4}}^{\raisebox{0.5pt}{\tiny \textcolor[HTML]{00008B}{$\uparrow$6.0}}}$ & 
$15.7_{\raisebox{-0.5pt}{\tiny $\pm$2.3}}^{\raisebox{0.5pt}{\tiny \textcolor[HTML]{00008B}{$\uparrow$0.7}}}$ \\

Multi Stage RL & 
$86.6_{\raisebox{-0.5pt}{\tiny $\pm$1.5}}^{\raisebox{0.5pt}{\tiny \textcolor[HTML]{00008B}{$\uparrow$3.5}}}$ & 
$49.3_{\raisebox{-0.5pt}{\tiny $\pm$1.0}}^{\raisebox{0.5pt}{\tiny \textcolor[HTML]{00008B}{$\uparrow$14.4}}}$ & 
$43.0_{\raisebox{-0.5pt}{\tiny $\pm$2.6}}^{\raisebox{0.5pt}{\tiny \textcolor[HTML]{00008B}{$\uparrow$12.0}}}$ & 
$17.3_{\raisebox{-0.5pt}{\tiny $\pm$2.1}}^{\raisebox{0.5pt}{\tiny \textcolor[HTML]{00008B}{$\uparrow$2.3}}}$ \\ 

\bottomrule[1pt]
\end{tabular}}
\label{tab:dif-strategy}
\end{table}
\vspace{-5pt}
We observe that under multi-stage training, SFT suffers a significant decline across tasks, averaging \textbf{23.1\% below} the base model, whereas mixed-data SFT yields a \textbf{7.4\% gain}. Conversely, RL exhibits robust improvements in both multi-stage and mixed-data settings, achieving average gains of \textbf{24.9\% and 12.6\%}, respectively. The relatively smaller gain in mixed-data RL may stem from factors such as gradient imbalance across tasks~\cite{wu2025imbalancedgradientsrlposttraining}, which is not the key focus of our work. We focus on the more striking disparity: \textbf{why multi-stage SFT leads to performance collapse while RL sustains stable and significant growth?} 

\vspace{-2pt}
\begin{tcolorbox}[takeaway,title={Takeaway 2.2}]
Multi-stage SFT leads to performance collapse, whereas RL enables stable and cumulative performance growth across distinct tasks.
\end{tcolorbox}
\vspace{-5pt}

\subsection{Generalization of Single-Task Training}

\label{sec:preliminary single task generalization}

To investigate the underlying reason for the question above, we analyze the impact of single-task training on the model's performance across other tasks. 

As shown in Table ~\ref{tab:dif-task-generalization}, while SFT improves the target task by an average of 4.0\%, it comes at the cost of degradation in untrained tasks, resulting in an average decline of 5.1\% across other tasks. In contrast, RL achieves a superior average gain of 6.8\% on the target task while simultaneously exerting a positive influence on others, yielding an average improvement of 2.3\% on untrained tasks. 
These observations indicate that SFT and RL differ fundamentally in how they affect other capabilities during learning on a single task. This phenomenon, which we term \textbf{SFT Conflicts} and \textbf{RL Coexists}, directly determines their suitability for different paradigms in multi-task settings.

\begin{table}[!htbp]
\caption{Accuracy (\%) on all tasks after single-task training using RL and SFT. Superscripts: performance changes to Base Model.}
\vspace{0.1in}
\scriptsize
\centering
\renewcommand{\arraystretch}{1.65} 
\resizebox{\linewidth}{!}{
\begin{tabular}{lcccc}
\toprule[1pt]
\scriptsize \textbf{Train-Set} & \scriptsize \textbf{Math} & \scriptsize \textbf{Science} & \scriptsize \textbf{Logic} & \scriptsize \textbf{Code} \\ \hline
\noalign{\vspace{-3pt}}
\multicolumn{5}{c}{\emph{Baseline}} \\ 
\noalign{\vspace{-3pt}}
Base Model  & 
$82.0_{\raisebox{-0.5pt}{\tiny $\pm$1.6}}$ & 
$34.9_{\raisebox{-0.5pt}{\tiny $\pm$1.1}}$ & 
$31.0_{\raisebox{-0.5pt}{\tiny $\pm$2.8}}$ & 
$15.0_{\raisebox{-0.5pt}{\tiny $\pm$2.3}}$ \\ \hline

Math SFT & 
$84.4_{\raisebox{-0.5pt}{\tiny $\pm$1.5}}^{\raisebox{0.5pt}{\tiny \textcolor[HTML]{00008B}{$\uparrow$2.4}}}$ & 
$28.0_{\raisebox{-0.5pt}{\tiny $\pm$1.1}}^{\raisebox{0.5pt}{\tiny \textcolor[HTML]{8B0000}{$\downarrow$6.9}}}$ & 
$15.0_{\raisebox{-0.5pt}{\tiny $\pm$2.3}}^{\raisebox{0.5pt}{\tiny \textcolor[HTML]{8B0000}{$\downarrow$16.0}}}$ & 
$12.4_{\raisebox{-0.5pt}{\tiny $\pm$2.4}}^{\raisebox{0.5pt}{\tiny \textcolor[HTML]{8B0000}{$\downarrow$2.6}}}$ \\

Science SFT    & 
$74.6_{\raisebox{-0.5pt}{\tiny $\pm$1.8}}^{\raisebox{0.5pt}{\tiny \textcolor[HTML]{8B0000}{$\downarrow$7.4}}}$ & 
$43.5_{\raisebox{-0.5pt}{\tiny $\pm$1.2}}^{\raisebox{0.5pt}{\tiny \textcolor[HTML]{00008B}{$\uparrow$8.6}}}$ & 
$10.0_{\raisebox{-0.5pt}{\tiny $\pm$2.3}}^{\raisebox{0.5pt}{\tiny \textcolor[HTML]{8B0000}{$\downarrow$21.0}}}$ & 
$10.5_{\raisebox{-0.5pt}{\tiny $\pm$2.1}}^{\raisebox{0.5pt}{\tiny \textcolor[HTML]{8B0000}{$\downarrow$4.5}}}$ \\

Logic SFT       & 
$78.4_{\raisebox{-0.5pt}{\tiny $\pm$1.7}}^{\raisebox{0.5pt}{\tiny \textcolor[HTML]{8B0000}{$\downarrow$3.6}}}$ & 
$36.0_{\raisebox{-0.5pt}{\tiny $\pm$1.1}}^{\raisebox{0.5pt}{\tiny \textcolor[HTML]{00008B}{$\uparrow$1.1}}}$ & 
$35.0_{\raisebox{-0.5pt}{\tiny $\pm$2.3}}^{\raisebox{0.5pt}{\tiny \textcolor[HTML]{00008B}{$\uparrow$4.0}}}$ & 
$14.1_{\raisebox{-0.5pt}{\tiny $\pm$2.5}}^{\raisebox{0.5pt}{\tiny \textcolor[HTML]{8B0000}{$\downarrow$0.9}}}$ \\

Code SFT        & 
$79.6_{\raisebox{-0.5pt}{\tiny $\pm$1.6}}^{\raisebox{0.5pt}{\tiny \textcolor[HTML]{8B0000}{$\downarrow$2.4}}}$ & 
$38.4_{\raisebox{-0.5pt}{\tiny $\pm$1.1}}^{\raisebox{0.5pt}{\tiny \textcolor[HTML]{00008B}{$\uparrow$3.5}}}$ & 
$31.0_{\raisebox{-0.5pt}{\tiny $\pm$3.3}}^{\raisebox{0.5pt}{\tiny \textcolor[HTML]{00008B}{$\uparrow$0.0}}}$ & 
$15.8_{\raisebox{-0.5pt}{\tiny $\pm$2.2}}^{\raisebox{0.5pt}{\tiny \textcolor[HTML]{00008B}{$\uparrow$0.8}}}$ \\ \hline

Math RL & 
$85.6_{\raisebox{-0.5pt}{\tiny $\pm$1.5}}^{\raisebox{0.5pt}{\tiny \textcolor[HTML]{00008B}{$\uparrow$3.6}}}$ & 
$42.5_{\raisebox{-0.5pt}{\tiny $\pm$1.1}}^{\raisebox{0.5pt}{\tiny \textcolor[HTML]{00008B}{$\uparrow$7.6}}}$ & 
$33.0_{\raisebox{-0.5pt}{\tiny $\pm$2.1}}^{\raisebox{0.5pt}{\tiny \textcolor[HTML]{00008B}{$\uparrow$2.0}}}$ & 
$15.6_{\raisebox{-0.5pt}{\tiny $\pm$2.4}}^{\raisebox{0.5pt}{\tiny \textcolor[HTML]{00008B}{$\uparrow$0.6}}}$ \\

Science RL     & 
$83.2_{\raisebox{-0.5pt}{\tiny $\pm$1.6}}^{\raisebox{0.5pt}{\tiny \textcolor[HTML]{00008B}{$\uparrow$1.2}}}$ & 
$48.5_{\raisebox{-0.5pt}{\tiny $\pm$1.1}}^{\raisebox{0.5pt}{\tiny \textcolor[HTML]{00008B}{$\uparrow$13.6}}}$ & 
$31.0_{\raisebox{-0.5pt}{\tiny $\pm$2.3}}^{\raisebox{0.5pt}{\tiny \textcolor[HTML]{00008B}{$\uparrow$0.0}}}$ & 
$15.9_{\raisebox{-0.5pt}{\tiny $\pm$2.3}}^{\raisebox{0.5pt}{\tiny \textcolor[HTML]{00008B}{$\uparrow$0.9}}}$ \\

Logic RL       & 
$83.8_{\raisebox{-0.5pt}{\tiny $\pm$1.5}}^{\raisebox{0.5pt}{\tiny \textcolor[HTML]{00008B}{$\uparrow$1.8}}}$ & 
$41.0_{\raisebox{-0.5pt}{\tiny $\pm$1.1}}^{\raisebox{0.5pt}{\tiny \textcolor[HTML]{00008B}{$\uparrow$6.1}}}$ & 
$38.0_{\raisebox{-0.5pt}{\tiny $\pm$2.3}}^{\raisebox{0.5pt}{\tiny \textcolor[HTML]{00008B}{$\uparrow$7.0}}}$ & 
$16.1_{\raisebox{-0.5pt}{\tiny $\pm$2.1}}^{\raisebox{0.5pt}{\tiny \textcolor[HTML]{00008B}{$\uparrow$1.1}}}$ \\

Code RL        & 
$81.6_{\raisebox{-0.5pt}{\tiny $\pm$1.7}}^{\raisebox{0.5pt}{\tiny \textcolor[HTML]{8B0000}{$\downarrow$0.4}}}$ & 
$38.7_{\raisebox{-0.5pt}{\tiny $\pm$1.1}}^{\raisebox{0.5pt}{\tiny \textcolor[HTML]{00008B}{$\uparrow$3.8}}}$ & 
$34.0_{\raisebox{-0.5pt}{\tiny $\pm$3.1}}^{\raisebox{0.5pt}{\tiny \textcolor[HTML]{00008B}{$\uparrow$3.0}}}$ & 
$17.8_{\raisebox{-0.5pt}{\tiny $\pm$2.0}}^{\raisebox{0.5pt}{\tiny \textcolor[HTML]{00008B}{$\uparrow$2.8}}}$ \\ 

\bottomrule[1pt]
\end{tabular}}
\label{tab:dif-task-generalization}
\end{table}
\vspace{-2pt}
\begin{tcolorbox}[takeaway,title={Takeaway 2.3}]
SFT suffers from task \textbf{conflicts}, where optimizing one capability severely compromises others; in contrast, RL exhibits task \textbf{coexistence}, allowing the model to learn new tasks while preserving existing ones.
\end{tcolorbox}
\vspace{-5pt}

\section{Parameter-Level Empirical Analysis}
\label{sec:parameter-level-empirical}
\label{sec:exp of para updates}

To investigate the mechanism underlying the contrasting generalization behaviors during single-task training observed in Section~\ref{sec:preliminary single task generalization}, we analyze the parameter update dynamics. For each task $T_i$, we compute the parameter update vector $ \Delta W_i $ and evaluate two key geometric properties: the \textbf{magnitude} ($L_2$ norm) and the \textbf{pairwise cosine similarity} between tasks. As visualized in Figure~\ref{fig:similarity_norm}, we identify two fundamental distinctions between SFT and RL updates.

\textbf{Observation 1: The magnitude of RL updates is minimal and sparse.} The heatmaps reveal a stark difference between $\Delta W_{RL}$ and $\Delta W_{SFT}$ in scale: the average $L_2$ norm of $\Delta W$ is approximately $3 \times 10^{-2}$ for RL, whereas it reaches $7.4$ for SFT, showing a difference of over two orders of magnitude. Furthermore, RL updates exhibit high sparsity; only about 20\% of parameters in RL have magnitudes exceeding $10^{-5}$, compared to 93\% in SFT. These results align with recent findings~\cite{mukherjee2025reinforcement, shenfeld2025rl} that RL tends to make \textbf{minimal adjustments}. The sparse and small parameter changes in the high-dimensional parameter space are a significant factor that makes the training of RL for different tasks less disruptive.


\begin{figure*}
\centering
  \includegraphics[width=0.85\linewidth]{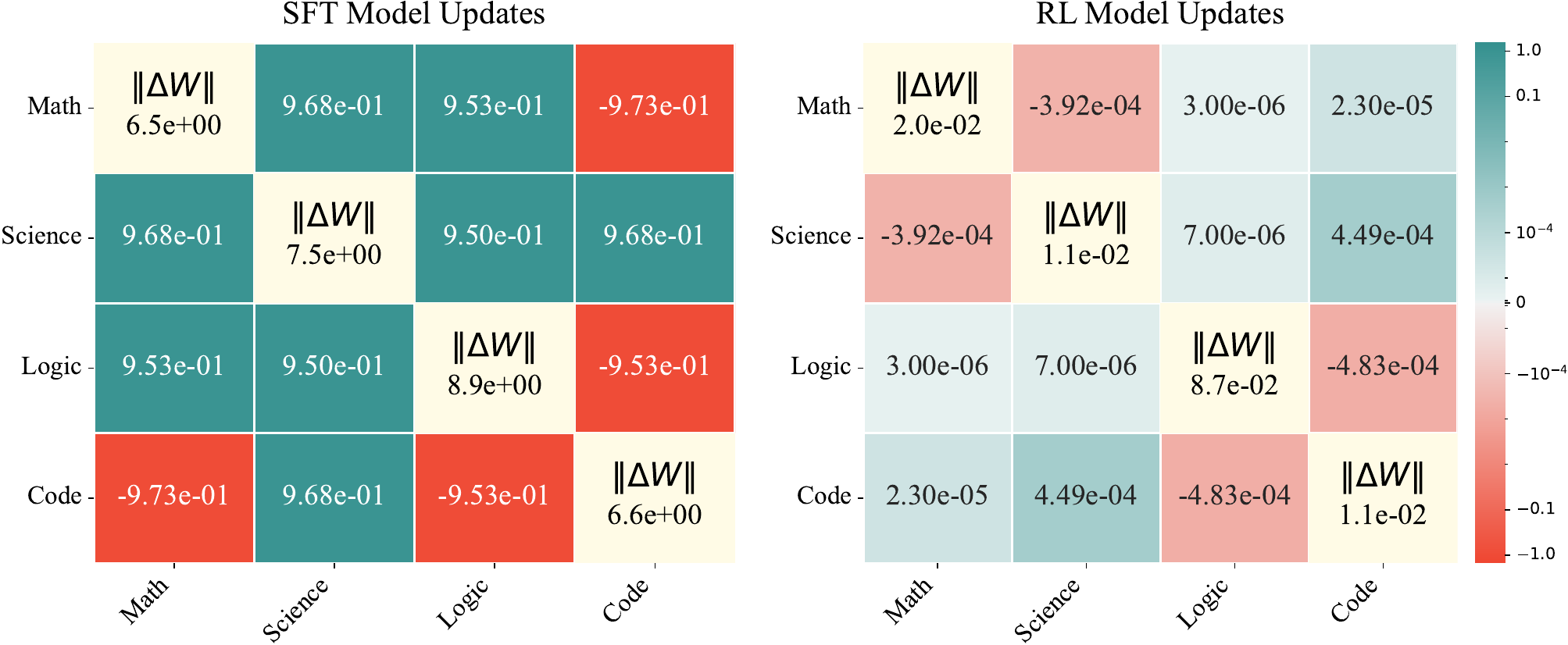}
  \caption{Analysis of Parameter Update Dynamics across SFT and RL. The heatmaps display the pairwise cosine similarity (off-diagonal) and the $L_2$ norm magnitude (diagonal) of parameter update vectors for different tasks. Left: SFT update has a larger absolute similarity value. ($\approx 1.0$). Right: RL updates show near-zero similarity, indicating orthogonal and task-specific parameter adjustments.}
  \vspace{-5pt}
  \label{fig:similarity_norm}
\end{figure*}




\textbf{Observation 2: The optimization directions of RL across different tasks are approximately orthogonal.} We observe that the pairwise cosine similarity of $\Delta W$ between different RL tasks is negligible, averaging around $10^{-5}$. In contrast, SFT exhibits high similarity across tasks (on the order of $10^{-1}$ to $1.0$), with some updates even pointing in opposite directions (e.g., Math vs. Code). Consequently, the $\Delta W_i$ for task $T_i$ obtained from RL training can be considered practically orthogonal, resulting in minimal interference with the optimization landscapes of other tasks. Whereas the $\Delta W_i$ from SFT exhibit overlapping optimization directions, which leads to significant mutual interference and the observed catastrophic forgetting.

Overall, this empirical analysis shows that RL induces \textbf{sparser} and \textbf{nearly orthogonal} parameter updates across tasks, which are two key factors that explain its minimal interference with other tasks.

\vspace{-2pt}
\begin{tcolorbox}[takeaway,title={Takeaway 3}]
Unlike SFT, RL induces \textbf{sparse, minimal updates that are approximately orthogonal across tasks}, naturally decoupling the optimization of different tasks. This parametric mechanism fundamentally explains why RL supports robust multi-stage training.\end{tcolorbox}
\vspace{-5pt}







\section{Theoretical Analysis}
\label{sec:theoretical_analysis}

Our empirical findings reveal two fundamental properties of parameter updates from RL which enable the harmonious coexistence of different task updates within a single model: (1) \textbf{Sparsity and Small Magnitude}, where updates are confined to a small subset of parameters with diminishing norms compared to SFT; and (2) \textbf{Orthogonality}, where gradients for different tasks occupy nearly disjoint subspaces. In this section, we will use theoretical qualitative analysis to explain these two properties. To analyze the distinct parameter updates $\Delta W$ of SFT and RL, we examine their respective gradient descent directions. Let $\pi_\theta$ denote the policy parameterized by $\theta$, $\pi_{\text{expert}}$ be the supervision distribution. The expected gradients can be formalized as follows:

\begin{equation}
\label{eq:gradient comparison}
\begin{aligned}
    & g_{\text{SFT}} = \mathbb{E}_{x \sim \mathcal{D}, \textcolor{myblue}{y \sim \pi_{\text{expert}}}} \left[ \nabla_\theta \log \pi_\theta(y|x) \right] \\
    & g_{\text{RL}} = \mathbb{E}_{x \sim \mathcal{D}, \textcolor{myblue}{y \sim \pi_\theta}} \left[ \textcolor{myred}{A(x, y)} \nabla_\theta \log \pi_\theta(y|x) \right]
\end{aligned}
\end{equation}

There are two fundamental distinctions:
\textcolor{myblue}{(1) Policy Source:} SFT is off-policy where the target response $y$ is sampled from the fixed expert distribution ($\pi_{\text{expert}}$ (ground truth). In contrast, RL is on-policy, where $y$ is sampled from $\pi_\theta$.
\textcolor{myred}{(2) Advantage (or Reward) Function:} The RL gradient incorporates a scalar weighting term $A(x, y)$.

In the following sections, we will conduct a qualitative theoretical analysis centering on these two mechanisms. Furthermore, in Section \ref{sec:quantify}, we will perform a quantitative analysis of the key factors obtained in the theoretical analysis.

\subsection{Sparsity and Small Magnitude of RL Updates}
\label{sec:sparsity_magnitude}

Observation 1 of Section \ref{sec:parameter-level-empirical} highlights the conservative nature of RL updates. Given the high-dimensional parameter space of LLMs, this minimal variation implies that update vectors from different tasks are likely to be orthogonal. This inference is underpinned by the mathematical principle that sparse vectors in high-dimensional spaces tend to be mutually orthogonal with high probability~\citep{cai2013distributionsanglesrandompacking, luisto2025shortsurveyorthogonalvectors}. Recent theoretical advancements have offered numerous explanations for the sparsity of parameter changes in RL. A representative work is RL's Razor~\cite{shenfeld2025rl}, which demonstrates that even in the absence of an explicit KL penalty term, on-policy RL optimization is implicitly biased towards solutions that are closest in KL-divergence to the initial policy. We restate this theory as follows:

\begin{proposition}[Convergence to KL-Minimal Solution, \citealp{shenfeld2025rl}]
\label{prop:kl_minimal}
Let $\mathcal{Y}$ be a finite set of outputs and let $\Pi \subseteq \Delta(\mathcal{Y})$ be a convex family of feasible policies. Let $R : \mathcal{Y} \to \{0,1\}$ be a binary reward function and $\mathcal{P}^* = \{q : \mathbb{E}_{q}[R] = 1\}$ be the set of optimal policies. Then, under suitable regularity conditions, policy gradient converges to:
\begin{equation}
    \pi^{updated} = \mathop{\arg\min}\limits_{\pi \in \mathcal{P}^* \cap \Pi} D_{KL}(\pi \,||\, \pi_0).
\end{equation}
In other words, it selects the optimal policy closest in KL-divergence to the initialization $\pi_0$.
\end{proposition}

During the theoretical derivation, they demonstrated that under binary rewards, each RL update step is equivalent to a two-stage alternating projection process (Information Projection and Moment Projection). Crucially, both stages implicitly satisfy the property of minimizing KL-divergence. And these two stages correspond precisely to the two difference shown in Eq \ref{eq:gradient comparison}. I-Projection: Projects the current policy onto the set of all policies with $Reward=1$, corresponding to \textcolor{myred}{reward function}. M-Projection: Projects the intermediate policy derived from the I-Projection back to the policy space representable by the model's parameterization, reflecting the role of \textcolor{myblue}{on-policy}. In summary, unlike SFT which must bridge a potentially large gap to fit an extrinsic expert distribution, RL implicitly minimizes $D_{KL}(\pi^* || \pi_0)$. This restriction on deviation directly leads to the observed \textbf{sparsity} and \textbf{small magnitude} in parameter updates.

\subsection{Qualitative Analysis of Multi-Task Gradients}
\label{sec:orthogonality}

In this section, we will analyze why gradients from different tasks in RL exhibit near-orthogonality. Similarity, we qualitatively examine this phenomenon from two perspectives: \textcolor{myblue}{policy} and \textcolor{myred}{advantage function}.

We consider two distinct tasks $i$ and $j$ with data distributions $\mathcal{D}_i$ and $\mathcal{D}_j$. Since many recent works on reasoning enhancement use RL algorithms based on GRPO, we will use it as a representative for analysis.

\begin{definition}[GRPO Gradient]
\label{def:grpo_gradient}
Given an input $x$ from task $i$, we generate $G$ group-wise rollouts $\{y_k\}_{k=1}^G$. The empirical gradient estimate is given by:
\begin{equation}
    g_i(x) = \frac{1}{G} \sum_{k=1}^G \hat{A}_{i,k}(x) \nabla_\theta \log \pi_\theta(y_k|x),
\end{equation}
where $\hat{A}_{i,k}(x)$ is the standardized advantage. Let $r_{i,k}(x)$ be the reward for the $k$-th sample, and let $\mu_{r_i}(x)=\frac{1}{G}\sum_{k=1}^G r_{i,k}(x)$ and $\sigma_{r_i}(x)=\sqrt{\frac{1}{G}\sum_{k=1}^G (r_{i,k}(x)-\mu_{r_i}(x))^2 }$ be the group mean and standard deviation, respectively. The advantage is defined as $\hat{A}_{i,k}(x) = (r_{i,k} - \mu_{r_i}(x)) / \sigma_{r_i}(x)$. 

\textbf{Definition of Score Function.} For notational convenience, we will denote the log-likelihood gradient of the $k$-th rollout by the \textbf{score function} $S_{i,k}(x) = \nabla_\theta \log \pi_\theta(y_k|x)$ below.

According to the definition, the advantages for GRPO satisfy the \textbf{zero-sum property}. 

\begin{equation}
\label{zero-sum property}
    \sum_{k=1}^G \hat{A}_{i,k}(x) = 0.
\end{equation}
\end{definition}

We contrast the inner products of gradient directions between tasks $i$ and $j$ in RL with those in SFT. We refer to this metric as the \textbf{gradient interference} term, which is formalized as follows:

\textbf{Gradient Interference in SFT.}
Let $\pi^*$ denote the expert policy (i.e., the supervision distribution). For a task $i$, the target response is $y_i^* \sim \pi^*(\cdot|x)$. Like RL, we define the \textbf{expert score function} as $S_i^*(x) \coloneqq \nabla_\theta \log \pi_\theta(y_i^*|x)$. Thus, the expected gradient interference in SFT is as follows:
\begin{equation}
\label{eq:sft_inner_product}
    \mathcal{I}_{\text{SFT}}(i, j) \coloneqq \mathbb{E}_{x, x'} \left[ \langle S_i^*(x), S_j^*(x') \rangle \right].
\end{equation}

\textbf{Gradient Interference in RL.}
For RL (GRPO), the expected inner product is given by:
{\small 
\begin{equation}
\label{eq:rl_raw_inner_product}
\begin{aligned} 
    &\mathcal{I}_{\text{RL}}(i, j) \coloneqq \mathbb{E}_{x, x'} \\ 
    &\quad \left[ \left\langle \frac{1}{G}\sum_{k=1}^G \hat{A}_{i,k} S_{i,k}(x), \frac{1}{G}\sum_{l=1}^G \hat{A}_{j,l} S_{j,l}(x') \right\rangle \right].
\end{aligned}
\end{equation}
}

Unlike SFT, we can decompose this raw form using the \textbf{zero-sum property} of advantages shown in Eq.\eqref{zero-sum property}.

\begin{lemma}[Gradient Inner Product Decomposition]
\label{lemma:inner_product_decomp}
Let $\bar{S}_i(x) = \frac{1}{G} \sum_{k=1}^G S_{i,k}(x)$ be the mean score function for input $x$, and let $\delta S_{i,k}(x) = S_{i,k}(x) - \bar{S}_i(x)$ be the residual score function. The expected inner product between the gradients of task $i$ and task $j$ satisfies:
{\small
\begin{equation}
\label{eq:rl_gradient_inner_product}
\begin{aligned}
    &\mathcal{I}_{\text{RL}}(i, j) = \mathbb{E}_{x, x'}  \\
    &\left[ \frac{1}{G^2} \sum_{k=1}^G \sum_{l=1}^G \textcolor{myred}{\hat{A}_{i,k}(x) \hat{A}_{j,l}(x')} \langle \textcolor{myblue}{\delta S_{i,k}(x), \delta S_{j,l}(x')} \rangle \right],
\end{aligned}
\end{equation}
}
which we prove in Appendix \ref{appendix:proof}.
\end{lemma}

Based on the derivations above, we analyze the gradient interference in SFT and RL from two mechanisms.

\textbf{\textcolor{myred}{Mechanism of Advantages.}}
Based on Lemma \ref{lemma:inner_product_decomp}, the \textbf{zero-sum property} of advantage function \textit{algebraically removes the contribution of the mean gradient direction $\bar{S}_i(x)$}. This transforms the inner product of the score functions (i.e., log-likelihood gradients) $\langle S_{i,k}, S_{j,l} \rangle$ into the inner product of residuals $\langle \delta S_{i,k}(x), \delta S_{j,l}(x') \rangle$, diminishing the gradient interference in multi-task RL. Conversely, the SFT inner product in Eq.\eqref{eq:sft_inner_product} depends directly on $\langle S_i^*(x), S_j^*(x') \rangle$. Consequently, the gradient interference in multi-task SFT is significantly larger than RL.

\textbf{\textcolor{myblue}{Mechanism of Policy Source.}}
The second critical distinction is that SFT is \textbf{off-policy}, relying on an external expert distribution, whereas RL is \textbf{on-policy}.

First, according to Lemma \ref{lemma:inner_product_decomp}, the RL inner product \ref{eq:rl_gradient_inner_product} is governed by residual score function, $\delta S_{i,k}(x)$. \textbf{Intuitively}, the residual $\delta S(x)$ captures the \textbf{intra-group policy divergence} across $G$ rollouts for the same input $x$. Since both the input and the model parameters remain fixed during generation, this divergence is inherently limited. (We quantify the extent of this variation in Section \ref{sec:quantify}.) 



Second, for two distinct tasks $i$ and $j$, the data distributions $\mathcal{D}_i$ and $\mathcal{D}_j$ are independent. Consequently, the small vectors $\delta S_i$ and $\delta S_j$, generated from independent sampling processes, are statistically \textbf{independent zero-mean} vectors.
In the context of high-dimensional parameter space of LLM $\mathbb{R}^d$, the Concentration of Measure phenomenon~\cite{vershynin2018high} dictates that the \textbf{independent, zero-mean and sparse} vectors are approximately orthogonal with high probability. Formally, the probability that their inner product deviates from zero decays exponentially with the dimension $d$: $ \mathbb{P}\left( |\langle \delta S_i, \delta S_j \rangle| \ge t \right) \le 2 \exp(-ct^2 d) $
In contrast, SFT is off-policy, targeting an external distribution. This requires \textbf{dense}, high-magnitude updates that significantly \textbf{overlap} across tasks, directly resulting in a large interference term $\langle S_i^*(x), S_j^*(x') \rangle$.
This intrinsic distinction between on-policy and off-policy directly accounts for the observed differences in gradient inner products. Quantitative analyses supporting are provided in Section \ref{sec:quantify}.

\textbf{Remark.} The advantage function acts as a filter that removes the dense mean gradient $\bar{S}$, leaving only the intra-group residuals. The on-policy nature further ensures these residuals $\delta S$ remain minimal and non-interfering.

\subsection{Upper Bound of Gradient Orthogonality}
\label{sec:upper bound}


We now derive an \textbf{upper bound} on the gradient interference in formula \ref{eq:sft_inner_product} and \ref{eq:rl_gradient_inner_product}. This derivation allows us to identify the key factors contributing to the approximate orthogonality of RL observed in Section \ref{sec:parameter-level-empirical}. To formally compare the interference magnitude, we introduce bounds characterizing the \textit{absolute magnitude} of the score functions for SFT and the \textit{intra-group variance} for RL.




\begin{assumption}[Bounded Expected Norm and Variance]
\label{ass:bounds}
For any task $i$ with data distribution $\mathcal{D}_i$:

\textbf{SFT Score Function Norm:} The expected squared $L_2$ norm of the expert score function is bounded by a constant $M_i^2$. This characterizes the average energy of the updates:
    \begin{equation*}
        \mathbb{E}_{x \sim \mathcal{D}_i} \left[ \| S^*_i(x) \|_2^2 \right] \le M_i^2.
    \end{equation*}

\textbf{RL Score Function Variance:} The expected intra-group variance of the residual score function (averaged over $G$ rollouts) is bounded by a constant $V_i^2$: 
    \begin{equation*}
        \mathbb{E}_{x \sim \mathcal{D}_i} \left[ \frac{1}{G} \sum_{k=1}^G \| \delta S_{i,k}(x) \|_2^2 \right] \le V_i^2.
    \end{equation*}
\end{assumption}
Based on this assumption, we can give upper bounds for the gradient inner product of SFT and RL for different tasks.

\begin{theorem}[Upper Bound on Gradient Interference]
\label{thm:upper_bounds}
Under Assumption \ref{ass:bounds}, the expected gradient inner product for SFT and RL (GRPO) satisfies the following upper bounds:

\textbf{1. SFT Upper Bound (Norm-Limited):}
\begin{equation}
\label{UB-SFT}
    | \mathcal{I}_{\text{SFT}}(i, j) | \le M_i \cdot M_j.
    \tag{UB-SFT}
\end{equation}

\textbf{2. RL Upper Bound (Variance-Limited):}
\begin{equation}
\label{UB-RL}
    | \mathcal{I}_{\text{RL}}(i, j) | \le V_i \cdot V_j.
    \tag{UB-RL}
\end{equation}
\end{theorem}


Theorem \ref{thm:upper_bounds} reveals a fundamental mechanistic difference between SFT and RL gradient interference upper bounds. The SFT bound depends on $M_i$, the absolute norm of the score function. As we have in our qualitative analysis in Section \ref{sec:theoretical_analysis}, the magnitude of SFT score function $M_i$ is large. However, the situation is completely different for RL. The RL bound is controlled exclusively by $V_i$, the magnitude of the \textit{residuals} $\delta S$. This is actually the variance with a group of rollouts generated by the model for the same input $x$. As we analyzed in Section \ref{sec:orthogonality}, due to the on-policy nature of RL, the intra-group variation will be small. Furthermore, $\delta S$ will continue to decay as training converges. The proof of this upper bound is provided in Appendix \ref{appendix:proof_upper_bound}. This theoretical distinction mirrors the fundamental essence of the two paradigms: SFT forces the model to fit fixed expert trajectories, whereas RL optimizes based on intra-group relative differences, reinforcing good trajectory while suppressing bad ones. We note that most properties also hold for other RL algorithms, as discussed in Appendix \ref{appendix: other_rl_proof}.

\vspace{-3pt}
\begin{tcolorbox}[takeaway,title={Takeaway 4.3}]
Gradient interference in SFT is norm-limited, governed by the absolute magnitude of parameter updates. Conversely, RL interference is variance-limited, strictly bounded by the diversity of intra-group rollouts, which filters the common-mode conflicts.
\end{tcolorbox}
\vspace{-5pt}

\subsection{Quantification and Visualization}
\label{sec:quantify}

\begin{figure}
\centering
  \includegraphics[width=\linewidth]{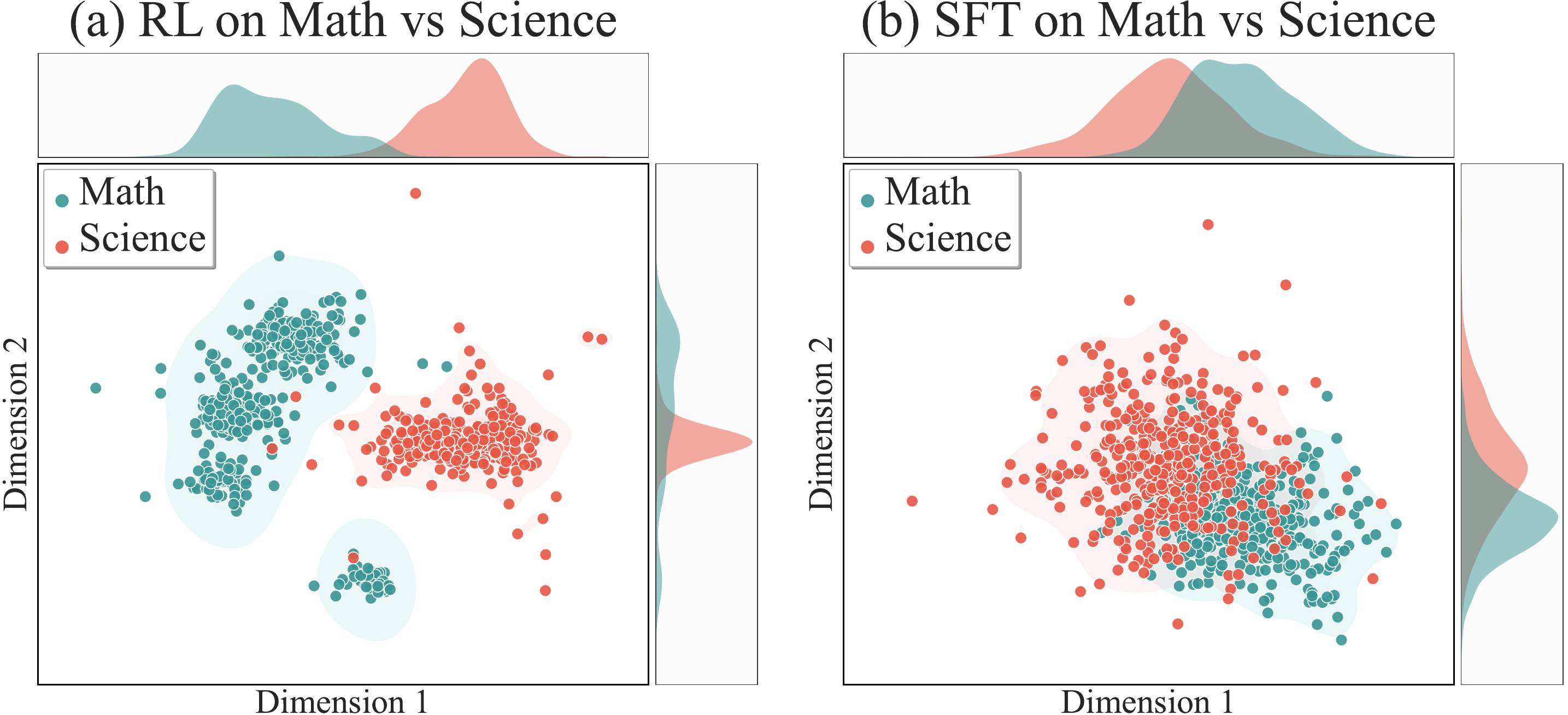}
  \vspace{-8pt}
  \caption {Distribution of the score function $S$ sampled during training by RL and SFT on different tasks (Math and Science). We use t-SNE to reduce the dimensionality for visualization.}
  \label{fig:score_func_distribution}
\end{figure}

To quantitatively validate our theoretical analysis, we examine three key metrics related to the score function $S$: the magnitude $||S||_2$, the residual $\delta S$ (derived in Lemma \ref{lemma:inner_product_decomp}), and the cosine similarity. We follow the setup in Section \ref{sec:setup for preliminaries}, using 200 sampled instances per task. The comparative results are summarized in Table \ref{tab:indicators}.

\begin{table}[h]
    \centering
    \small
    \vspace{-3pt} 
    \caption{Quantitative comparison of score function ($S$) indicators between SFT and RL. RL exhibits significantly smaller magnitude and inter-task similarity, alongside a bounded residual $\delta$.}
    \label{tab:indicators}
    \setlength{\tabcolsep}{6mm} 
    \begin{tabular}{lccc}
        \toprule
        \textbf{Indicator} & \textbf{SFT} & \textbf{RL} \\
        \midrule
        $||S||_2$ & $\sim 7.1$ & $\sim 10^{-1}$ \\
        $||\delta S||_2$ & - & $\sim 10^{-2}$ \\
        $\text{CosSim}(S_i, S_j)$ & $\sim 10^{-1}$ & $\sim 10^{-3}$ \\
        \bottomrule
    \end{tabular}
    \vspace{-3pt} 
\end{table}

\textbf{Verification of Interference Upper Bounds.} 
We first evaluate the factors governing the gradient interference bounds derived in Theorem \ref{thm:upper_bounds}. 
For SFT, the interference is \textit{norm-limited}. Table \ref{tab:indicators} shows that $||S||_2$ for SFT is substantial ($\sim 7.1$), confirming that fitting external experts leading to a high gradient interference. 
In contrast, RL interference is \textit{variance-limited}, governed strictly by the intra-group residual $\delta S$. Empirically, the $||\delta S||_2$ is observed to be small ($\sim 10^{-2}$). This confirms that filtering effect of the advantage function restricts interference to this minimal variance term, explaining the significantly lower interference in RL.

\begin{table*}[!htbp]
\vspace{-6pt}
\caption{The accuracy (\%) on different tasks. Superscripts denote accuracy change vs. \textbf{corresponding single-task baseline}. $\Delta$ Base shows the average performance improvement over the Base Model.  Retention (\%) represents the \textbf{percentage of performance retained} compared to the Single-Task Training Model before merging for Parallel Training method. \textbf{Bold} represents the best performance.} 
\vspace{0.1in}
\centering
\small
\resizebox{1.0 \linewidth}{!}{
\begin{tabular}{clcccccccc}
\toprule[1pt]
\multirow{2}{*}{\normalsize \textbf{Base Model}} & \multicolumn{1}{l}{\multirow{2}{*}{\normalsize \textbf{Method}}} & \multicolumn{2}{c}{\normalsize \textbf{Math}} & \multicolumn{2}{c}{\normalsize \textbf{Science}} & \multicolumn{1}{c}{\normalsize \textbf{Logic}} & \multicolumn{1}{c}{\normalsize \textbf{Code}} & \multicolumn{2}{c}{\normalsize \textbf{Average Statistic}} \\ \cline{3-10} 
    & \multicolumn{1}{c}{}                        & MATH500    & AIME2025   & MMLU     & GPQA     & KK    & LiveCodeBench & $\Delta$ Base & Retention (\%) \\ 
    \hline

\multirow{17}{*}{\shortstack{DeepSeek-R1- \\ Distill-Qwen-1.5B}}  
& \multicolumn{9}{c}{\emph{Baselines (Full Parameter GRPO)}} \\
&Base Model & 82.0 & 26.9 & 34.9 & 32.3 & 31.0 & 15.0 & - & - \\
&Single-Task SFT & 85.8 & 31.5 & 49.1 & 35.4 & 37.0 & 16.3 & +5.5 & - \\
&Single-Task RL & 87.4 & 32.7 & 51.8 & 39.9 & 44.0 & 21.6 & +9.3 & - \\
 \cline{2-10}

& \multicolumn{9}{c}{\emph{Previous Multi Task Training Paradigms}} \\
& Mixed Data SFT  & $85.6^{\textcolor[HTML]{8B0000}{-0.2}}$ & $30.1^{\textcolor[HTML]{8B0000}{-1.4}}$ & $48.5^{\textcolor[HTML]{8B0000}{-0.6}}$ & $34.3^{\textcolor[HTML]{8B0000}{-1.1}}$ & $38.0^{\textcolor[HTML]{00008B}{+1.0}}$ & $17.0^{\textcolor[HTML]{00008B}{+0.7}}$ & +5.2 & - \\
& Multi Stage SFT & $71.0^{\textcolor[HTML]{8B0000}{-14.8}}$ & $21.8^{\textcolor[HTML]{8B0000}{-9.7}}$ & $33.2^{\textcolor[HTML]{8B0000}{-15.9}}$ & $23.2^{\textcolor[HTML]{8B0000}{-12.2}}$ & $11.0^{\textcolor[HTML]{8B0000}{-26.0}}$ & $12.1^{\textcolor[HTML]{8B0000}{-4.2}}$ & -8.3 & - \\ 
& Mixed Data RL & $87.0^{\textcolor[HTML]{8B0000}{-0.4}}$ & $32.9^{\textcolor[HTML]{00008B}{+0.2}}$ & $50.9^{\textcolor[HTML]{8B0000}{-0.9}}$ & $38.9^{\textcolor[HTML]{8B0000}{-1.0}}$ & $45.0^{\textcolor[HTML]{00008B}{+1.0}}$ & $20.7^{\textcolor[HTML]{8B0000}{-0.9}}$ & +8.9 & - \\
& Multi Stage RL & $87.8^{\textcolor[HTML]{00008B}{+0.4}}$ & $33.1^{\textcolor[HTML]{00008B}{+0.4}}$ & $\textbf{52.8}^{\textcolor[HTML]{00008B}{+1.0}}$ & $40.4^{\textcolor[HTML]{00008B}{+0.5}}$ & $48.0^{\textcolor[HTML]{00008B}{+4.0}}$ & $21.0^{\textcolor[HTML]{8B0000}{-0.6}}$ & +10.2 & - \\ 
 \cline{2-10}

& \multicolumn{9}{c}{\emph{Parallel Training}} \\
&Naive Parallel-SFT \textit{(mean)} & $63.2^{\textcolor[HTML]{8B0000}{-22.6}}$ & $19.2^{\textcolor[HTML]{8B0000}{-12.3}}$ & $26.9^{\textcolor[HTML]{8B0000}{-22.2}}$ & $25.3^{\textcolor[HTML]{8B0000}{-10.1}}$ & $26.0^{\textcolor[HTML]{8B0000}{-11.0}}$ & $9.3^{\textcolor[HTML]{8B0000}{-7.0}}$ & -8.7 & 66.6 \\
&Naive Parallel-SFT \textit{(sum)} & $61.2^{\textcolor[HTML]{8B0000}{-24.6}}$ & $20.0^{\textcolor[HTML]{8B0000}{-11.5}}$ & $25.9^{\textcolor[HTML]{8B0000}{-23.2}}$ & $26.8^{\textcolor[HTML]{8B0000}{-8.6}}$ & $23.0^{\textcolor[HTML]{8B0000}{-14.0}}$ & $10.1^{\textcolor[HTML]{8B0000}{-6.2}}$ & -9.2 & 65.4 \\
&Naive Parallel-RL \textit{(mean)} & $85.6^{\textcolor[HTML]{8B0000}{-1.8}}$ & $31.5^{\textcolor[HTML]{8B0000}{-1.2}}$ & $47.3^{\textcolor[HTML]{8B0000}{-4.5}}$ & $36.4^{\textcolor[HTML]{8B0000}{-3.5}}$ & $39.0^{\textcolor[HTML]{8B0000}{-5.0}}$ & $19.0^{\textcolor[HTML]{8B0000}{-2.6}}$ & +6.1 & 93.3 \\
&Naive Parallel-RL \textit{(sum)} & $86.4^{\textcolor[HTML]{8B0000}{-1.0}}$ & $32.3^{\textcolor[HTML]{8B0000}{-0.4}}$ & $48.0^{\textcolor[HTML]{8B0000}{-3.8}}$ & $37.4^{\textcolor[HTML]{8B0000}{-2.5}}$ & $39.0^{\textcolor[HTML]{8B0000}{-5.0}}$ & $18.4^{\textcolor[HTML]{8B0000}{-3.2}}$ & +6.6 & 94.2 \\
&TIES Parallel-RL & $87.6^{\textcolor[HTML]{00008B}{+0.2}}$ & $32.5^{\textcolor[HTML]{8B0000}{-0.2}}$ & $49.2^{\textcolor[HTML]{8B0000}{-2.6}}$ & $38.4^{\textcolor[HTML]{8B0000}{-1.5}}$ & $43.0^{\textcolor[HTML]{8B0000}{-1.0}}$ & $19.7^{\textcolor[HTML]{8B0000}{-1.9}}$ & +8.0 & 97.4 \\
&SVD Parallel-RL & $87.2^{\textcolor[HTML]{8B0000}{-0.2}}$ & $31.9^{\textcolor[HTML]{8B0000}{-0.8}}$ & $48.3^{\textcolor[HTML]{8B0000}{-3.5}}$ & $36.9^{\textcolor[HTML]{8B0000}{-3.0}}$ & $42.0^{\textcolor[HTML]{8B0000}{-2.0}}$ & $15.9^{\textcolor[HTML]{8B0000}{-5.7}}$ & +6.7 & 94.6 \\
&Adapted Parallel-RL & $\textbf{88.6}^{\textcolor[HTML]{00008B}{+1.2}}$ & $\textbf{33.8}^{\textcolor[HTML]{00008B}{+1.1}}$ & $50.9^{\textcolor[HTML]{8B0000}{-0.9}}$ & $\textbf{41.4}^{\textcolor[HTML]{00008B}{+1.5}}$ & $\textbf{49.0}^{\textcolor[HTML]{00008B}{+5.0}}$ & $\textbf{22.5}^{\textcolor[HTML]{00008B}{+0.9}}$ & +10.7 & \textbf{103.2} \\

    \hline

\multirow{17}{*}{\shortstack{DeepSeek-R1- \\ Distill-Qwen-7B}}  
& \multicolumn{9}{c}{\emph{Baselines (Full Parameter GRPO)}} \\
&Base Model    & 91.8 & 49.2 & 51.8 & 48.5 & 62.0 & 36.5 & - & - \\
&Single-Task SFT & 93.8 & 54.4 & 59.1 & 51.0 & 65.0 & 39.1 & +3.8 & - \\
&Single-Task RL & 94.6 & 57.1 & 63.5 & 53.0 & 68.0 & 42.3 & +6.5 & - \\
 \cline{2-10}
& \multicolumn{9}{c}{\emph{Previous Multi Task Training Paradigms}} \\
& Mixed Data SFT  & $94.6^{\textcolor[HTML]{00008B}{+0.8}}$ & $53.1^{\textcolor[HTML]{8B0000}{-1.3}}$ & $61.7^{\textcolor[HTML]{00008B}{+2.6}}$ & $52.5^{\textcolor[HTML]{00008B}{+1.5}}$ & $66.0^{\textcolor[HTML]{00008B}{+1.0}}$ & $40.0^{\textcolor[HTML]{00008B}{+0.9}}$ & +4.7 & - \\
& Multi Stage SFT & $81.4^{\textcolor[HTML]{8B0000}{-12.4}}$ & $38.6^{\textcolor[HTML]{8B0000}{-15.8}}$ & $35.3^{\textcolor[HTML]{8B0000}{-23.8}}$ & $39.1^{\textcolor[HTML]{8B0000}{-11.9}}$ & $46.0^{\textcolor[HTML]{8B0000}{-19.0}}$ & $29.5^{\textcolor[HTML]{8B0000}{-9.6}}$ & -11.7 & - \\ 
& Mixed Data RL & $95.0^{\textcolor[HTML]{00008B}{+0.4}}$ & $56.5^{\textcolor[HTML]{8B0000}{-0.6}}$ & $62.2^{\textcolor[HTML]{8B0000}{-1.3}}$ & $52.8^{\textcolor[HTML]{8B0000}{-0.2}}$ & $67.0^{\textcolor[HTML]{8B0000}{-1.0}}$ & $39.9^{\textcolor[HTML]{8B0000}{-2.4}}$ & +5.6 & - \\
& Multi Stage RL & $95.0^{\textcolor[HTML]{00008B}{+0.4}}$ & $58.2^{\textcolor[HTML]{00008B}{+1.1}}$ & $65.9^{\textcolor[HTML]{00008B}{+2.4}}$ & $\textbf{55.6}^{\textcolor[HTML]{00008B}{+2.6}}$ & $66.0^{\textcolor[HTML]{8B0000}{-2.0}}$ & $\textbf{44.2}^{\textcolor[HTML]{00008B}{+1.9}}$ & +7.5 & - \\ 
 \cline{2-10}

& \multicolumn{9}{c}{\emph{Parallel Training}} \\
&Naive Parallel-SFT \textit{(mean)} & $61.0^{\textcolor[HTML]{8B0000}{-32.8}}$ & $35.8^{\textcolor[HTML]{8B0000}{-18.6}}$ & $33.1^{\textcolor[HTML]{8B0000}{-26.0}}$ & $38.9^{\textcolor[HTML]{8B0000}{-12.1}}$ & $48.0^{\textcolor[HTML]{8B0000}{-17.0}}$ & $27.0^{\textcolor[HTML]{8B0000}{-12.1}}$ & -16.0 & 67.3 \\
&Naive Parallel-SFT \textit{(sum)} & $59.2^{\textcolor[HTML]{8B0000}{-34.6}}$ & $35.6^{\textcolor[HTML]{8B0000}{-18.8}}$ & $32.4^{\textcolor[HTML]{8B0000}{-26.7}}$ & $37.8^{\textcolor[HTML]{8B0000}{-13.2}}$ & $50.0^{\textcolor[HTML]{8B0000}{-15.0}}$ & $26.5^{\textcolor[HTML]{8B0000}{-12.6}}$ & -16.4 & 66.6 \\
&Naive Parallel-RL \textit{(mean)} & $93.2^{\textcolor[HTML]{8B0000}{-1.4}}$ & $52.3^{\textcolor[HTML]{8B0000}{-4.8}}$ & $59.0^{\textcolor[HTML]{8B0000}{-4.5}}$ & $51.0^{\textcolor[HTML]{8B0000}{-2.0}}$ & $61.0^{\textcolor[HTML]{8B0000}{-7.0}}$ & $38.3^{\textcolor[HTML]{8B0000}{-4.0}}$ & +2.5 & 93.7 \\
&Naive Parallel-RL \textit{(sum)} & $93.6^{\textcolor[HTML]{8B0000}{-1.0}}$ & $53.1^{\textcolor[HTML]{8B0000}{-4.0}}$ & $59.6^{\textcolor[HTML]{8B0000}{-3.9}}$ & $52.5^{\textcolor[HTML]{8B0000}{-0.5}}$ & $63.0^{\textcolor[HTML]{8B0000}{-5.0}}$ & $38.0^{\textcolor[HTML]{8B0000}{-4.3}}$ & +3.3 & 95.1 \\
&TIES Parallel-RL & $94.8^{\textcolor[HTML]{00008B}{+0.2}}$ & $56.5^{\textcolor[HTML]{8B0000}{-0.6}}$ & $62.5^{\textcolor[HTML]{8B0000}{-1.0}}$ & $52.0^{\textcolor[HTML]{8B0000}{-1.0}}$ & $64.0^{\textcolor[HTML]{8B0000}{-4.0}}$ & $40.5^{\textcolor[HTML]{8B0000}{-1.8}}$ & +5.1 & 97.8 \\
&SVD Parallel-RL & $93.2^{\textcolor[HTML]{8B0000}{-1.4}}$ & $55.2^{\textcolor[HTML]{8B0000}{-1.9}}$ & $61.4^{\textcolor[HTML]{8B0000}{-2.1}}$ & $51.5^{\textcolor[HTML]{8B0000}{-1.5}}$ & $65.0^{\textcolor[HTML]{8B0000}{-3.0}}$ & $39.5^{\textcolor[HTML]{8B0000}{-2.8}}$ & +4.3 & 96.7 \\
&Adapted Parallel-RL & $\textbf{96.2}^{\textcolor[HTML]{00008B}{+1.6}}$ & $\textbf{59.0}^{\textcolor[HTML]{00008B}{+1.9}}$ & $\textbf{66.8}^{\textcolor[HTML]{00008B}{+3.3}}$ & $54.0^{\textcolor[HTML]{00008B}{+1.0}}$ & $\textbf{70.0}^{\textcolor[HTML]{00008B}{+2.0}}$ & $41.5^{\textcolor[HTML]{8B0000}{-0.8}}$ & +8.0 & \textbf{102.4} \\

  \bottomrule[1pt]
\vspace{-34pt}
\end{tabular}
}
\label{tab:main_result}
\end{table*}

\textbf{Verification of Task Separation.} 
Next, we validate the mechanism discussed in Section \ref{sec:orthogonality}, which posits that on-policy RL naturally leads to low overlap between updates from different tasks, while SFT updates show significant overlap. 
Table \ref{tab:indicators} reveals that the cosine similarity for RL ($\sim 10^{-3}$) is much lower than that of SFT ($\sim 10^{-1}$), indicating that RL updates occupy nearly orthogonal subspaces. We also employ t-SNE to plot the distribution of score functions. As shown in Figure \ref{fig:score_func_distribution}, RL score functions form distinct, separable clusters for different tasks. Conversely, SFT distributions exhibit heavy overlap. Together, these results corroborate that SFT suffers from dense, conflicting updates, whereas RL naturally decouples task optimization. More visualizations are shown in Appendix \ref{appendix:more visualizations}.

\section{Application: Parallel-RL Framework}

\subsection{Methodology}

Our preceding empirical and theoretical analyses establish a fundamental geometric property of multi-task RL: parameter updates for distinct reasoning tasks occupy \textbf{approximately orthogonal} subspaces (Section \ref{sec:theoretical_analysis}). This finding implies that the interference term $\langle \Delta W_i, \Delta W_j \rangle$ is negligible. Considering optimization directions are irrelevant, the sum of parameter updates obtained from parallel training of different tasks should theoretically approximate the result of multi-stage training of tasks.

Motivated by this insight, we propose \textbf{Parallel-RL}, a decoupled post-training paradigm. In Parallel-RL, we can optimize each task independently.
Formally, let $\mathcal{T} = \{T_1, \dots, T_N\}$ be a set of diverse reasoning tasks. We launch $N$ parallel RL training processes. Each process $i$ produces a task-specific update $\Delta W_i$. The final model parameter $W_{final}$ is obtained by merging these independent updates:
$$ W_{final} = W_{base} + \mathcal{M}(\Delta W_1, \dots, \Delta W_N) $$
where $\mathcal{M}$ is a merging function (e.g., linear averaging or SVD, discussed in Section \ref{sec:parallel_exp}). We emphasize that Parallel-RL is not merely a model merging technique, but a complete post-training paradigm. Beyond the merging stage, it also involves determining which tasks can be trained in parallel, as well as adopting appropriate tricks for each single task training, which are preliminarily discussed in Appendix \ref{appendix:parallel_rl}.

\subsection{Main Experiment}
\label{sec:parallel_exp}
\textbf{Setup.} We adopt DeepSeek-R1-Distill-Qwen-1.5B and 7B as the base models for full parameter training. The configuration follows the settings described in Appendix \ref{appendix:experiment details}. Here we choose GRPO as the RL algorithm. Results for other RL algorithms can be found in Appendix \ref{appendix:more_ablation}.

\textbf{Parallel-RL Method.} We explore several basic strategies for implementing and merging on Parallel-RL. \textbf{(1) Naive Parallel-RL.}We employ two strategies: (a) \textit{sum}: directly summing the $\Delta W_i$, (b) \textit{mean}: averaging the $\Delta W_i$. \textbf{(2) Sparse Parallel.}
Two sparsification strategies are considered here. First, we adopt the standard TIES method ~\cite{yadav2023ties}. Second, motivated by evidence that the rank-1 component of $\Delta W_i$ encapsulates most RL training effects ~\cite{cai2025predictability}, we perform SVD on each update and retain only the rank-1 directions for merging. \textbf{(3) Adapted Parallel-RL.}
Since different RL models enhance the sampling probability of high-reward trajectory for their respective tasks, the merged model can benefit from light post-merge adaptation. After Naive Parallel RL \textit{(sum)}, we perform rapid adaptation using a small subset of samples (5\% of the original train-set size) to refine the model.

\subsection{Main Results}

As shown in Table \ref{tab:main_result}, we observe that even Naive Parallel-RL \textit{(sum)} allows different task updates ($\Delta W_i$) to coexist, largely preserving the performance gains (95\%) from single task training and achieving an average improvement of 5.0\% over the base model. In contrast, Parallel-SFT \textit{(sum)} suffers from severe task conflicts, with only 66\% of the performance of the corresponding single-task SFT retained. 
Furthermore, incorporating sparsification strategies further enhances the compatibility of various $\Delta W_i$ for RL. TIES and SVD Parallel-RL respectively retained 98\% and 96\% of the performance of the corresponding single-task RL. Notably, Adapted Parallel-RL achieves the best performance across most tasks. It delivers a 9.4\% gain over the base model and even surpasses individual task-specific models, retaining an average of 102.8\% of the Single-Task RL performance. Despite these gains, this paradigm only needs an additional 5\% samples adaptation time compared to single-task training, demonstrating both the \textbf{efficiency and superior efficacy} of our approach. More experimental results are available in Appendix \ref{appendix:more_main_results}. We also performed Pass@K analysis to further verify the stability of Parallel RL, as shown in Appendix \ref{appendix:pass_at_k}.

\subsection{Ablation Study}

\label{sec:ablation}


To intuitively assess the orthogonality and decoupling of the learned task representations, we conduct an ablation study by selectively excluding each task-specific update ($\Delta W_i$) during the merging process. As reported in Table \ref{tab:ablation}, removing a specific $\Delta W_i$ leads to a significant performance drop in its corresponding task, with an average drop ($\Delta_{Target}$) of 7.1\%. However, In contrast, the performance on the remaining tasks stays remarkably robust, even exhibiting a slight average improvement ($\Delta_{Others}$) of 0.6\%. This result demonstrates that Parallel-RL effectively decouples capabilities across tasks. Unlike multi-stage or mixed-data RL, this paradigm allows for the flexible composition of models, enabling modular adaptation to diverse task scenarios. More ablation results are shown in Appendix \ref{appendix:more_ablation}

\begin{table}[!htbp]
\vspace{-3pt}
\caption{Performance (\%) ablation study of Naive Parallel-RL. $\Delta_{Target}$ and $\Delta_{Others}$ highlight the change in accuracy on the corresponding task and other tasks after removing a certain module.}
\vspace{0.1in}
\centering
\resizebox{0.95\linewidth}{!}{
\begin{tabular}{lcccccc}
\toprule[1pt]
Method & MATH500 & MMLU & KK & LCB & $\Delta_{Target}$ & $\Delta_{Others}$ \\ 
\hline
\multicolumn{7}{c}{\emph{Baselines}} \\
Single-Task RL & 87.4 & 51.8 & 44.0 & 21.6 & - & - \\
Base Model & 82.0 & 34.9 & 31.0 & 15.0 & - & - \\ \hline
\multicolumn{7}{c}{\emph{Ablation of Parallel-RL (1.5B GRPO)}} \\
Parallel-RL & 86.4 & 48.0 & 39.0 & 18.4 & - & - \\ \hline
- w/o Math & 82.8 & 47.8 & 41.0 & 19.2 & -3.6 & +0.9 \\
- w/o Science & 86.0 & 37.5 & 39.0 & 18.4 & -10.5 & -0.1 \\
- w/o Logic & 85.6 & 49.6 & 30.0 & 20.3 & -9.0 & +0.9 \\
- w/o Code & 88.2 & 46.9 & 40.0 & 13.1 & -5.3 & +0.6 \\ \bottomrule[1pt]
\end{tabular}
}
\vspace{-5pt}
\label{tab:ablation}
\end{table}

\section{Related Work}

We summarize research aligned with our core investigation and provide a comprehensive discussion in Appendix \ref{appendix:related works}.
Enhancing the reasoning capabilities of LLMs has become a focal point of recent research~\cite{guo2025deepseek, jaech2024openai}. The dominant training strategies remain SFT~\cite{xu2025redstar, openr1} and RL~\cite{zheng2025group, ouyang2022training}. 
Recently, significant efforts have been made to distinguish the mechanisms of SFT and RL in \textit{single-task} settings. \citet{chu2025sft} argue that SFT memorizes, RL generalizes, suggesting SFT focuses on learning new paths, while RL enhances generalization. Similar works include ~\cite{matsutani2025rl,jin2025rl}, etc. There is also work~\cite{mukherjee2025reinforcement} pointing out that compared to SFT, RL training only fine-tunes the small subnetwork. Other works have explored the combination of these two paradigms ~\cite{zhang2025policy, fu2025srft}.
Multi-task learning is also essential for scaling towards AGI~\cite{zeng2025glm, li2025omni}. However, the divergence between multi-task SFT and RL remains underexplored. 
Some works have explored related aspects. RL's Razor \cite{shenfeld2025rl} theoretically explains that RL minimizes parameter deviation due to implicit KL constraints, thereby reducing forgetting. Similar analysis can be found in The Path Not Taken \cite{zhu2025path}. However, their analysis targets forgetting rather than multi-task learning and the fundamental differences in the \textit{training paradigms}. In contrast, our work is the first to systematically investigate the divergence between SFT and RL in multi-task training with theoretical and empirical analysis.

\section{Conclusion}

This work systematically analyzes the divergence between SFT and RL in multi-task training. We observe that in multi-stage training, SFT exhibits significant task conflicts, whereas RL demonstrates a coexistence property, allowing for stable improvements. Through mechanistic analysis, we find that the optimization directions of RL across different tasks share significantly lower similarity than SFT, appearing approximately orthogonal. We further substantiate these empirical findings with a in-depth theoretical analysis, elucidating how the advantage function and on-policy nature of RL drive this mechanism. Building on the findings, we introduce Parallel-RL to decouple multi-task training, enabling efficient, modular training that achieves superior performance with minimal adaptation.

\section*{Impact Statements}

This paper aims to advance the development of multi-task learning for LLMs. By analyzing and improving the efficiency and effectiveness of model learning across diverse task scenarios, we hope to facilitate progress toward the vision of AGI. However, we emphasize that while striving for efficient multi-task training, it is also crucial to maintain the safety and reliability of each individual task to ensure responsible deployment. Finally, we believe that this methodological study does not introduce distinct ethical risks or societal impacts that warrant detailed discussion here.



\bibliography{example_paper}
\bibliographystyle{icml2026}

\newpage
\appendix
\onecolumn

\section{Experiment Details}
\label{appendix:experiment details}

\subsection{Task and Dataset Description}
\label{appendix:task description}

In this study, we select four representative tasks widely utilized in reasoning enhancement research for multi-task learning: Math, Science, Code, and Logic. In the following sections, we detail the training and evaluation datasets employed for each of these tasks.

\textbf{Training Datasets.} 
We use different datasets that are more suitable for SFT and RL training, respectively. \textbf{For SFT}, we utilize a subset of \textit{OpenR1-Math-220k~\cite{openr1}} as the math training set, a subset of \textit{AM-DeepSeek-Distilled-40M~\cite{tian2025deepdistill}} for code generation, and a subset of \textit{AM-Thinking-v1-Distilled~\cite{tian2025correctanswersequaldistillation}} for science. For the logical reasoning task, we employ the \textit{knights-and-knaves~\cite{xie2024memorization}} dataset. Since the original data lacks CoTs, we leverage the DeepSeek-R1 API to distill LongCoT data and ensure the correctness of the distilled answers using the rule-based verification logic from \textit{Logic-RL~\cite{xie2025logic}}. \textbf{Regarding RL}, our training data consists of a subset of the \textit{DeepScaleR-Preview-Dataset~\cite{meng2023deepscaler}} for math and a subset of the \textit{DeepCoder-Preview-Dataset~\cite{luo2025deepcoder}} for code. For science, we use a subset of the science portion from \textit{AM-Thinking-v1-Distilled}. Finally, following the protocol of \textit{Logic-RL}, a subset of the \textit{knights-and-knaves} dataset is used for the logical task. 

\textbf{Evaluation Benchmarks.} 
We evaluate our models using six mainstream and widely recognized benchmarks. 
For Mathematics, we select \textit{MATH500~\cite{lightman2023lets}} and \textit{AIME2025}. It is worth noting that the number of questions for AIME2025 is relatively small. To avoid randomness, we report the evaluation metric as avg@16, which is the average accuracy of 16 tests on the dataset. For Science, we employ \textit{MMLU~\cite{hendrycks2021measuringmassivemultitasklanguage}} and \textit{GPQA-Diamond~\cite{rein2023gpqagraduatelevelgoogleproofqa}}. Considering the massive scale of MMLU and the fact that some of its tasks are not directly related to scientific reasoning, we specifically focus on the following subjects: \textit{high\_school\_physics, high\_school\_chemistry, college\_chemistry, college\_biology, astronomy,} and \textit{professional\_medicine}. 
For Code Generation, we primarily use \textit{LiveCodeBench~\cite{jain2024livecodebenchholisticcontaminationfree}} for evaluation. 
Lastly, for Logical Reasoning, the \textit{knights-and-knaves} dataset is used as the evaluation benchmark.

\subsection{Training Details}
\label{appendix:training_details}

We employ \textit{DeepSeek-R1-Distill-Qwen-1.5B} and \textit{7B} as base models. All experiments are conducted on a machine of eight NVIDIA A100 (80GB) GPUs. To facilitate efficient parameter analysis, some experiments are performed using LoRA~\cite{hu2021loralowrankadaptationlarge}. For all LoRA experiments, we set the rank to $r = 64$ and the scaling factor to $\alpha = 32$. For multi-stage training, we chose the following task order: Math, Science, Code, Logic. For TIES merging~\cite{yadav2023ties}, we follow the default settings and set the sensitivity parameter to 0.8.

\textbf{SFT.} SFT is implemented using the \textit{LLaMA-Factory~\cite{zheng2024llamafactory}} framework. We fine-tune the models on teacher CoT trajectories. The learning rate is set to $1 \times 10^{-5}$, and the sequence cutoff length is configured at 8K tokens.

\textbf{RL.} We utilize the \textit{VeRL~\cite{sheng2024hybridflow}} framework for reinforcement learning training, specifically employing the GRPO algorithm. The RL learning rate is fixed at $3 \times 10^{-6}$. Following the official DeepSeek configuration, we set the sampling temperature to $0.6$ and $top\text{-}p$ to $0.95$. For each prompt, we generate $G=16$ rollouts to perform group-based relative reward estimation. The maximum output length is set to 8K tokens. Notably, we disable the KL divergence penalty in the loss function for our specific analysis. 

We acknowledge the exceptional contributions of \textit{LLaMA-Factory} and \textit{VeRL} teams. Their extensible and user-friendly frameworks, which support diverse architectures and advanced algorithms, significantly facilitated our training implementation. 

\subsection{Evaluation Details}
\label{appendix:evaluation_details}

For the majority of the tasks, we conduct evaluation using the \textit{lighteval~\cite{lighteval}} toolkit provided by Hugging Face. Consistent with the training configuration described above, we employ a sampling temperature of $0.6$ and a $top\text{-}p$ value of $0.95$ for all model generations. For the logical reasoning tasks, we utilize the specialized evaluation suite from \textit{Logic-RL}. We express our gratitude to the developers of these tools for their significant contributions to the open-source community, which ensured the reliability and comparability of our evaluation results. All evaluation experiments are performed on NVIDIA A100 (80GB) GPUs to maintain consistency in the computational environment.

\section{More Related Works}
\label{appendix:related works}

In this section, we provide a comprehensive review of the literature concerning post-training paradigms, the comparative analysis of SFT~\cite{wei2021finetuned} and RL~\cite{ziegler2019fine}, and multi-task learning.

\textbf{Post-Training for Large Language Models.} The mainstream post-training paradigms mainly include SFT and RL. In the era of reasoning enhancement, SFT has evolved beyond simple instruction following to incorporate Chain-of-Thought (CoT) methodologies~\citep{wei2022chain,jin2026looklightthinkheavy, zhu2026mmrvwhatsleftunsaid,li2026mmrlifepiecingreallifescenes}, which encourage models to generate step-by-step reasoning paths~\citep{li2025thinking, longpre2023flan}. Recent advancements have further refined this approach through specialized reasoning fine-tuning, significantly boosting capabilities in mathematical and logical tasks~\citep{xu2025redstar, openr1, tan2026pyxpyinvestigatingreinforcement, tan2026bottomuppolicyoptimizationlanguage}. At the same time RL has shifted focus from traditional human preference alignment~\citep{ouyang2022training, jin2025omni, li2026fixingbrokencompassdiagnosing, men2025agentrewardbenchunifiedbenchmarkreward} to outcome-driven reasoning optimization. By leveraging verifiable rewards in domains like mathematics and coding~\cite{hao2026multisteptoolusereinforcementlearning}, methods exemplified by DeepSeek-R1~\citep{guo2025deepseek} and other recent frameworks~\citep{yu2025dapo, team2025kimi, zeng2025glm} have greatly improved the model reasoning capabilities~\citep{li2026agenticenvironmentengineeringlarge,men2026physicsmultiturnlonghorizonplanning}. 

\textbf{Comparative Analysis of SFT and RL in Single-Task Learning.} 
A representative comparative study~\cite{chu2025sft} characterize the distinction as ``SFT memorizes, RL generalizes'', observing that while SFT is responsible for fitting supervised training paths, RL focuses on generalization. Mechanistically, this divergence is often attributed to search space dynamics. 
Other work~\cite{matsutani2025rl} propose that ``RL squeezes, SFT expands'', indicating that RL contracts the policy towards high-reward regions, whereas SFT broadens the sampling space. Furthermore, \citet{jin2025rl} find that RL can recover generalization capabilities lost during aggressive SFT. Other research has also compared the differences between the two in terms of out-of-distribution (OOD) task generalization~\citep{huan2025does}. There is also work\cite{mukherjee2025reinforcement} pointing out that compared to SFT, RL training only fine-tunes the small subnetwork. 
Recognizing the distinct properties of RL and SFT, recent hybrid approaches integrate on-policy RL with off-policy SFT to leverage the benefits of both paradigms~\citep{fu2025srft, zhang2025policy}.


\textbf{Multi-Task Learning.} To enable deployment in complex real-world scenarios, enhancing the multi-task reasoning capabilities of LLMs has become a priority~\citep{gemini3, gpt5, yang2025qwen3}. For SFT, the common practice is to train on mixed datasets spanning diverse tasks~\citep{openr1}. Existing studies suggest that conflicts in multi-task SFT are often severe and evolve throughout training, motivating adaptive data selection, curriculum design, and conflict-aware reweighting instead of static data mixing alone~\citep{wu2024mixture, liang2025boosting}. In contrast, RL paradigms exhibit more diverse strategies. Many works explore multi-stage training~\citep{zeng2025glm}, or employ curriculum learning strategies~\cite{cho2024hard, li2025omni} to progressively build capabilities. Other studies investigate mixed-data RL by alleviating gradient imbalance~\cite{wu2025imbalanced} or scaling RL on heterogeneous reasoning data~\cite{liu2025deepseek}. Recent multi-task RL studies further show that cross-domain transfer is highly dependent on tasks: while some domains are mutually beneficial, others induce negative transfer or capacity competition~\citep{li2025can, cheng2025revisiting}. Complementary optimization-based methods further seek to mitigate such conflicts by improving gradient compatibility across domains~\citep{liang2026boosting}. Despite these advances, the paradigmatic divergence between SFT and RL in multi-task settings remains under-analyzed. Recent studies demonstrate that RL exhibits significantly less catastrophic forgetting than SFT~\citep{shenfeld2025rl, lai2025reinforcementfinetuningnaturallymitigates, zhu2025path}. In contrast, our work is the first to systematically investigate this divergence through both theoretical and empirical analysis. As the field evolves towards Agentic RL and the number of reasoning tasks continues to grow~\citep{zeng2025rlve, men2026empoweringguiagentsautonomous}, understanding these fundamental differences is increasingly important for scalable AGI development.

\textbf{Model Merging Techniques.} The paradigm of model merging has emerged as a highly efficient strategy to integrate the capabilities of multiple expert models without the prohibitive costs of retraining. Foundational works in this domain demonstrated that models can be fused via simple weight averaging~\citep{wortsman2022model} or through Task Arithmetic~\citep{ilharco2022editing}, which operates on task vectors to add or subtract specific model behaviors. However, directly combining multiple task vectors frequently leads to severe performance degradation due to task conflicts. To mitigate this, subsequent research introduced subspace and sparsification strategies. Most notably, TIES-Merging~\citep{yadav2023ties} resolves interference by pruning task vectors based on parameter magnitudes and electing a dominant sign direction to eliminate sign conflicts during aggregation. Following a similar sparsification philosophy, DARE~\citep{yu2024language} randomly drops parameters and rescales the remaining weights to seamlessly assimilate homologous models. Beyond sparsification, other advanced methodologies calculate precise merging coefficients by evaluating parameter importance, utilizing tools such as the Fisher Information Matrix~\citep{matena2022merging} or closed-form statistical solutions for linear layers~\citep{jin2022dataless}.

\clearpage
\section{Theory and Proof}

\label{appendix:proof}
\subsection{RL Inner Product Decomposition}
We present the proof of Lemma \ref{lemma:inner_product_decomp} in the main text. This lemma states that for GRPO, the inner product of training gradients for different tasks actually depends on the difference between rollouts within a group rather than the rollout itself.
\begin{proof}
We expand the inner product of the gradients defined in Definition \ref{def:grpo_gradient}:
\begin{equation*}
    \langle g_i(x), g_j(x') \rangle = \left\langle \frac{1}{G} \sum_{k=1}^G \hat{A}_{i,k} S_{i,k}, \frac{1}{G} \sum_{l=1}^G \hat{A}_{j,l} S_{j,l} \right\rangle.
\end{equation*}
Substituting the decomposition $S_{i,k} = \bar{S}_i + \delta S_{i,k}$, the term for task $i$ becomes:
\begin{equation*}
\begin{aligned}
  \sum_{k=1}^G \hat{A}_{i,k} (\bar{S}_i + \delta S_{i,k}) 
  &= \bar{S}_i \underbrace{\sum_{k=1}^G \hat{A}_{i,k}}_{=0} + \sum_{k=1}^G \hat{A}_{i,k} \delta S_{i,k} \\
  &= \sum_{k=1}^G \hat{A}_{i,k} \delta S_{i,k}.
\end{aligned}
\end{equation*}
The mean score direction $\bar{S}_i$ is eliminated due to the zero-sum property of the standardized advantages. Substituting this back into the inner product yields:
\begin{equation*}
    \langle g_i(x), g_j(x') \rangle = \frac{1}{G^2} \sum_{k,l} \hat{A}_{i,k} \hat{A}_{j,l} \langle \delta S_{i,k}, \delta S_{j,l} \rangle.
\end{equation*}
Taking the expectation over $x \sim \mathcal{D}_i$ and $x' \sim \mathcal{D}_j$ completes the proof.
\end{proof}





\subsection{Proof of Theorem \ref{thm:upper_bounds} (Upper Bound on Gradient Interference)}
\label{appendix:proof_upper_bound}
In this section, we derive the upper bounds for the gradient inner product for SFT and RL, as presented in Theorem \ref{thm:upper_bounds}.

\subsubsection{Proof for SFT (Norm-Limited)}
\begin{proof}
As defined in Eq \ref{eq:sft_inner_product}, the gradient interference term is defined as the expected inner product of the expert score functions over independent data distributions $\mathcal{D}_i$ and $\mathcal{D}_j$ of different task $i$ and $j$ respectively:
\begin{equation*}
    \mathcal{I}_{\text{SFT}}(i, j) = \mathbb{E}_{x \sim \mathcal{D}_i, x' \sim \mathcal{D}_j} \left[ \langle S^*_i(x), S^*_j(x') \rangle \right].
\end{equation*}

First, we apply the property $| \mathbb{E}[Z] | \le \mathbb{E}[|Z|]$ (which follows from \textbf{Jensen's inequality}) together with the \textbf{Cauchy-Schwarz inequality} for the inner product ($|\langle \mathbf{u}, \mathbf{v} \rangle| \le \|\mathbf{u}\|_2 \|\mathbf{v}\|_2$):

\begin{equation*}
\begin{aligned}
    | \mathcal{I}_{\text{SFT}}(i, j) | & = \left| \mathbb{E}_{x, x'} \left[ \langle S^*_i(x), S^*_j(x') \rangle \right] \right| \\
    &\le \mathbb{E}_{x, x'} \left[ | \langle S^*_i(x), S^*_j(x') \rangle | \right] \\
    &\le \mathbb{E}_{x, x'} \left[ \| S^*_i(x) \|_2 \cdot \| S^*_j(x') \|_2 \right].
\end{aligned}
\end{equation*}

Since the samples $x$ and $x'$ are drawn independently from $\mathcal{D}_i$ and $\mathcal{D}_j$, the expectation of the product can be factored into the product of expectations:
\begin{equation*}
    \mathbb{E}_{x, x'} \left[ \| S^*_i(x) \|_2 \cdot \| S^*_j(x') \|_2 \right] = \mathbb{E}_{x} \left[ \| S^*_i(x) \|_2 \right] \cdot \mathbb{E}_{x'} \left[ \| S^*_j(x') \|_2 \right].
\end{equation*}

Next, we apply the property: $(\mathbb{E}[Z])^2 \le \mathbb{E}[Z^2]$, implying $\mathbb{E}[Z] \le \sqrt{\mathbb{E}[Z^2]}$. This can also be derived from \textbf{Jensen's inequality}.
\begin{equation*}
\begin{aligned}
    \mathbb{E}_{x} \left[ \| S^*_i(x) \|_2 \right] &\le \sqrt{\mathbb{E}_{x} \left[ \| S^*_i(x) \|_2^2 \right]} \\
    \mathbb{E}_{x'} \left[ \| S^*_j(x') \|_2 \right] &\le \sqrt{\mathbb{E}_{x'} \left[ \| S^*_j(x') \|_2^2 \right]}.
\end{aligned}
\end{equation*}

Finally, substituting this back into the expression and applying the bounds $M_i$ and $M_j$ defined in Assumption \ref{ass:bounds} (where $\mathbb{E}_{x \sim \mathcal{D}_i} \left[ \| S^*_i(x) \|_2^2 \right] \le M_i^2$), we obtain:
\begin{equation*}
    | \mathcal{I}_{\text{SFT}}(i, j) | \le M_i \cdot M_j.
\end{equation*}
This confirms the bound in Eq.\eqref{UB-SFT}.
\end{proof}

\subsubsection{Proof for RL (Variance-Limited)}

\begin{proof}
For RL (GRPO series, i.e., RL algorithms that normalize the advantage function within group), we start with the decomposition of the expected gradient inner product derived in Lemma \ref{lemma:inner_product_decomp}:
\begin{equation*}
    \mathcal{I}_{\text{RL}}(i, j) = \mathbb{E}_{x, x'} \left[ \frac{1}{G^2} \sum_{k=1}^G \sum_{l=1}^G \hat{A}_{i,k} \hat{A}_{j,l} \langle \delta S_{i,k}, \delta S_{j,l} \rangle \right].
\end{equation*}

To derive the upper bound, we first apply the property $| \mathbb{E}[Z] | \le \mathbb{E}[|Z|]$ (\textbf{Jensen's inequality}) and the Triangle Inequality to the sum inside the expectation:
\begin{equation*}
\begin{aligned}
    | \mathcal{I}_{\text{RL}}(i, j) | 
    &\le \mathbb{E}_{x, x'} \left[ \left| \frac{1}{G^2} \sum_{k=1}^G \sum_{l=1}^G \hat{A}_{i,k} \hat{A}_{j,l} \langle \delta S_{i,k}, \delta S_{j,l} \rangle \right| \right] \\
    &\le \mathbb{E}_{x, x'} \left[ \frac{1}{G^2} \sum_{k=1}^G \sum_{l=1}^G | \hat{A}_{i,k} | | \hat{A}_{j,l} | | \langle \delta S_{i,k}, \delta S_{j,l} \rangle | \right].
\end{aligned}
\end{equation*}

Next, we apply the \textbf{Cauchy-Schwarz inequality} ($|\langle \mathbf{u}, \mathbf{v} \rangle| \le \|\mathbf{u}\|_2 \|\mathbf{v}\|_2$) to the inner product: $|\langle \delta S_{i,k}, \delta S_{j,l} \rangle| \le \| \delta S_{i,k} \|_2 \| \delta S_{j,l} \|_2$. Thus the expression can be factored into the product of two independent sums:
\begin{equation*}
    | \mathcal{I}_{\text{RL}}(i, j) | \le \mathbb{E}_{x, x'} \left[ \frac{1}{G^2} \left( \sum_{k=1}^G | \hat{A}_{i,k} | \| \delta S_{i,k} \|_2 \right) \left( \sum_{l=1}^G | \hat{A}_{j,l} | \| \delta S_{j,l} \|_2 \right) \right].
\end{equation*}

Since we consider two distinct tasks, the data sampling and rollout generation processes for task $i$ and task $j$ are independent. Consequently, the expectation of the product factors into the product of expectations:

\begin{equation*}
\label{eq:appendix_proof_it_rl}
    | \mathcal{I}_{\text{RL}}(i, j) | \le \frac{1}{G^2} \underbrace{\mathbb{E}_{x} \left[ \sum_{k=1}^G | \hat{A}_{i,k} | \| \delta S_{i,k} \|_2 \right]}_{Term_i} \cdot \underbrace{\mathbb{E}_{x'} \left[ \sum_{l=1}^G | \hat{A}_{j,l} | \| \delta S_{j,l} \|_2 \right]}_{Term_j}.
\end{equation*}

We now bound the term for task $i$ ($Term_i$). First, we apply the \textbf{Cauchy-Schwarz inequality} (for scalars) to the sum inside the expectation:
\begin{equation}
\label{ieq:appendix_for_it_rl_mid}
    \sum_{k=1}^G | \hat{A}_{i,k} | \| \delta S_{i,k} \|_2 \le \sqrt{\sum_{k=1}^G \hat{A}_{i,k}^2} \cdot \sqrt{\sum_{k=1}^G \| \delta S_{i,k} \|_2^2}.
\end{equation}

For RL algorithms of GRPO series, the advantage estimates are standardized within each group to stabilize training. Specifically, for a group of rewards $\{r_{i,1}, \dots, r_{i,G}\}$, the standardized advantage is computed as:
\begin{equation*}
    \hat{A}_{i,k} = \frac{r_{i,k} - \text{mean}(\{r_i\})}{\text{std}(\{r_i\})}.
\end{equation*}
By the definition of standardization, the mean of their squares is exactly 1:
\begin{equation*}
    \frac{1}{G} \sum_{k=1}^G \hat{A}_{i,k}^2 = 1 \quad \implies \quad \sum_{k=1}^G \hat{A}_{i,k}^2 = G.
\end{equation*}

Substituting this property into the inequality \ref{ieq:appendix_for_it_rl_mid} simplifies the bound to:
\begin{equation*}
    \sum_{k=1}^G | \hat{A}_{i,k} | \| \delta S_{i,k} \|_2 \le \sqrt{G} \cdot \sqrt{\sum_{k=1}^G \| \delta S_{i,k} \|_2^2}.
\end{equation*}

Taking the expectation on both sides, we use \textbf{Jensen's inequality} for the concave function $f(z) = \sqrt{z}$ (implying $\mathbb{E}[\sqrt{Z}] \le \sqrt{\mathbb{E}[Z]}$):
\begin{equation*}
\begin{aligned}
    Term_i = \mathbb{E}_{x} \left[ \sum_{k=1}^G | \hat{A}_{i,k} | \| \delta S_{i,k} \|_2 \right] 
    &\le \mathbb{E}_{x} \left[ \sqrt{G} \cdot \sqrt{\sum_{k=1}^G \| \delta S_{i,k} \|_2^2} \right] \\
    &\le \sqrt{G} \cdot \sqrt{ \mathbb{E}_{x} \left[ \sum_{k=1}^G \| \delta S_{i,k} \|_2^2 \right] }.
\end{aligned}
\end{equation*}

According to Assumption \ref{ass:bounds}, the expected intra-group variance is bounded: $\mathbb{E}_{x} \left[ \frac{1}{G} \sum_{k=1}^G \| \delta S_{i,k} \|_2^2 \right] \le V_i^2$, which implies $\mathbb{E}_{x} \left[ \sum_{k=1}^G \| \delta S_{i,k} \|_2^2 \right] \le G V_i^2$. Substituting this into the inequality:
\begin{equation*}
    Term_i \le \sqrt{G} \cdot \sqrt{ G \cdot V_i^2 } = G V_i.
\end{equation*}

Applying the same bound for task $j$ ($Term_j \le G V_j$) and substituting back into the expression \ref{eq:appendix_proof_it_rl}:
\begin{equation*}
    | \mathcal{I}_{\text{RL}}(i, j) | \le \frac{1}{G^2} \cdot (G V_i) \cdot (G V_j) = V_i \cdot V_j.
\end{equation*}

This confirms the bound in Eq.\eqref{UB-RL}.
\end{proof}

\subsection{Generalization to Other RL Algorithm}
\label{appendix: other_rl_proof}
The variance-limited interference mechanism is not exclusive to GRPO. First, in the context of the reasoning enhancement or RLVR \cite{guo2025deepseek} focused on in this paper, most recent reasoning-oriented RL algorithms (e.g., GSPO~\cite{zheng2025group}, DAPO~\cite{yu2025dapo}) generally follow the GRPO paradigm by generating a group of rollouts for a single prompt and performing group-wise advantage normalization. Consequently, these algorithms also satisfy the derivation presented above. Second, regarding classical RL algorithms, this property also holds approximately. In fact, as long as there is normalization for the advantage and the on-policy property, our derivation can hold. In this section, we provide a derivation using PPO~\cite{schulman2017proximal} as an example.


Unlike GRPO which normalizes advantages within a group of outputs for the same prompt, PPO (as implemented in VeRL~\cite{sheng2024hybridflow}) typically employs \textbf{token-level GAE} and performs \textbf{Batch-Level Advantage Normalization}. We show that our theoretical framework seamlessly generalizes to this setting.

\textbf{Definition (PPO Gradient with Batch Normalization).} 
Consider a training batch $\mathcal{B}$ consisting of multiple trajectories. Let $\mathcal{K}$ be the set of indices for all valid tokens in this batch (flattening batch and time dimensions), with total count $N = |\mathcal{K}|$. For each token $k \in \mathcal{K}$, let $S_k = \nabla_\theta \log \pi_\theta(a_k | s_k)$ be the score function and $\hat{A}_k$ be the normalized advantage.
The empirical gradient estimate for task $i$ is given by:
\begin{equation*}
    g_i^{\text{PPO}} = \frac{1}{N} \sum_{k=1}^N \hat{A}_{k} S_{k}.
\end{equation*}
Crucially, the \textbf{Batch-Level Advantage Normalization} ensures that the advantages also satisfy the property that the sum of advantage values is 0 and the variance is 1 over the batch:
\begin{equation*}
    \sum_{k=1}^N \hat{A}_k = 0, \quad \text{and} \quad \sum_{k=1}^N \hat{A}_k^2 = N.
\end{equation*}

\textbf{Lemma (PPO Gradient Decomposition).}
Let $\bar{S}_{\mathcal{B}} = \frac{1}{N} \sum_{k=1}^N S_k$ be the \textbf{batch mean score function}, representing the common optimization direction of the entire batch. Let $\delta S_k = S_k - \bar{S}_{\mathcal{B}}$ be the \textbf{residual score} for the $k$-th token. The PPO gradient can be decomposed as:
\begin{equation*}
    g_i^{\text{PPO}} = \frac{1}{N} \sum_{k=1}^N \hat{A}_k \delta S_k.
\end{equation*}

\textit{Proof.}
Substituting $S_k = \bar{S}_{\mathcal{B}} + \delta S_k$ into the gradient definition:
\begin{equation*}
    g_i^{\text{PPO}} = \frac{1}{N} \sum_{k=1}^N \hat{A}_k (\bar{S}_{\mathcal{B}} + \delta S_k) = \frac{1}{N} \underbrace{\left( \sum_{k=1}^N \hat{A}_k \right)}_{=0} \bar{S}_{\mathcal{B}} + \frac{1}{N} \sum_{k=1}^N \hat{A}_k \delta S_k = \frac{1}{N} \sum_{k=1}^N \hat{A}_k \delta S_k.
\end{equation*}
This confirms that PPO also filters out the common mean component $\bar{S}_{\mathcal{B}}$, relying solely on the residuals over batch.

\textbf{Theorem (PPO Interference Upper Bound).}
Let Assumption \ref{ass:bounds} be adapted to PPO such that the expected squared norm of the residual scores over a batch of size $N$ is bounded:
\begin{equation*}
    \mathbb{E}_{\mathcal{B} \sim \mathcal{D}_i} \left[ \frac{1}{N} \sum_{k=1}^N \| \delta S_{i,k} \|_2^2 \right] \le V_i^2.
\end{equation*}
Then, the expected gradient inner product between task $i$ and task $j$ (trained with PPO) satisfies:
\begin{equation*}
    | \mathcal{I}_{\text{PPO}}(i, j) | \le V_i \cdot V_j.
\end{equation*}


The derivation steps are the same as in Section \ref{appendix:proof_upper_bound}.

\textbf{Discussion on the Differences with GRPO Series.}
While the upper bound for PPO is algebraically identical to GRPO in Eq \ref{UB-RL}, the interpretation of $\delta S$ differs slightly. 
In RL algorithms utilizing group-wise advantage normalization (e.g. GRPO), $\delta S$ captures the intra-group policy divergence corresponding to a single prompt. In contrast, within PPO, this term represents the deviation across the prompts of entire batch. Nevertheless, the core mechanism remains the same: the advantage term $\hat{A}$ acts as a high-pass filter that removes the high-magnitude common direction ($\bar{S}$). 
In GRPO, the filtered $\bar{S}$ represents the common characteristics of the model's responses to a specific input. This common direction exhibits a large magnitude because, given a fixed model and input, the variance in model activations is inherently limited. Conversely, in PPO, the filtering applies to the common direction across all samples in a batch. Intuitively, this common direction is weaker than that derived from a single input. This results in PPO being more perturbative than GRPO during parallel training. Supporting evidence can be found in Appendix \ref{appendix:more_main_results}. However, since samples within a batch typically belong to the same task domain (e.g., mathematics or coding), this process still effectively filters out significant commonalities associated with task-specific features and linguistic priors. Consequently, the upper bound of interference in PPO is determined by the policy fluctuations within the specific task. Drawing from this analysis, we propose a potential strategy to optimize PPO in future work: by generating a group of rollouts for a single input and constraining the scope of advantage normalization to these intra-group trajectories, PPO may potentially achieve orthogonality effects comparable to those of GRPO.

\subsection{Supplementary Analysis: Cosine Similarity of Multi-Task Gradients}
\label{app:cosine_similarity_multitask}

This subsection provides a complementary analysis to the derivation in the main text. In Section \ref{sec:upper bound}, we upper bounded the \emph{gradient inner product} between tasks, which characterizes practical interference. Here, we further study the problem from a directional perspective by factoring out scale and directly examining the interference induced by different tasks under SFT and RL. Specifically, we analyze the \emph{cosine similarity} between task gradients for the two training paradigms.

To make the directional comparison explicit, we introduce an additional high-dimensional decomposition assumption on the score function. Under this assumption, we show that SFT admits a positive lower bound on cross-task cosine similarity due to the persistence of a shared dominant direction, whereas RL admits an upper bound controlled only by residual coupling, because the shared direction is exactly canceled by the zero-sum advantage weights.

\paragraph{Setup.}
Consider two independent tasks $i$ and $j$. For task $i$, given prompt $x_i$, the model generates $G$ rollouts $\{y_{i,k}\}_{k=1}^G$; similarly, task $j$ generates $\{y_{j,m}\}_{m=1}^G$.

We assume that the score function of each rollout admits the decomposition
\[
S_{i,k}(x_i) = \mu + \epsilon_{i,k},
\qquad
S_{j,m}(x_j) = \mu + \epsilon_{j,m},
\]
where $\mu$ denotes a task-shared dominant direction, and $\epsilon$ is a task- and sample-specific residual component.

\paragraph{High-dimensional geometric condition.}
We assume that, on a high-probability event $\mathcal{E}$,
\[
\|\mu\| = M_\mu,
\qquad
\|\epsilon_{i,k}\| = \|\epsilon_{j,m}\| = \sigma = \eta M_\mu,
\qquad
\eta \ll 1.
\]
That is, the residual norm is concentrated and much smaller than the shared dominant direction.

We further define the \emph{task residual coupling factor} $\gamma_{i,j}\in[0,1]$ as
\[
\max_{k,m}
\frac{|\langle \epsilon_{i,k}, \epsilon_{j,m}\rangle|}{\sigma^2}
\le \gamma_{i,j}.
\]
Similarly, for distinct rollouts within the same task, we assume an intra-task coupling factor $\gamma_{\mathrm{in}}\ll 1$ such that
\[
\max_{k\neq m}
\frac{|\langle \epsilon_{i,k}, \epsilon_{i,m}\rangle|}{\sigma^2}
\le \gamma_{\mathrm{in}}.
\]

\paragraph{SFT gradients preserve the shared dominant direction.}
Under SFT, the gradient is the uniform average of score functions:
\[
g_i^{\mathrm{SFT}}
=
\frac{1}{G}\sum_{k=1}^G S_{i,k}(x_i)
=
\mu + \bar{\epsilon}_i,
\qquad
\bar{\epsilon}_i := \frac{1}{G}\sum_{k=1}^G \epsilon_{i,k}.
\]
By the triangle inequality,
\[
\|\bar{\epsilon}_i\| \le \sigma = \eta M_\mu.
\]

We first lower bound the numerator:
\begin{align*}
\langle g_i^{\mathrm{SFT}}, g_j^{\mathrm{SFT}}\rangle
&=
\langle \mu+\bar{\epsilon}_i,\mu+\bar{\epsilon}_j\rangle \\
&=
\|\mu\|^2
+
\langle \mu,\bar{\epsilon}_i\rangle
+
\langle \mu,\bar{\epsilon}_j\rangle
+
\langle \bar{\epsilon}_i,\bar{\epsilon}_j\rangle \\
&\ge
M_\mu^2
-
M_\mu\|\bar{\epsilon}_i\|
-
M_\mu\|\bar{\epsilon}_j\|
-
\|\bar{\epsilon}_i\|\,\|\bar{\epsilon}_j\| \\
&\ge
M_\mu^2(1-2\eta-\eta^2).
\end{align*}
Next, we upper bound the denominator:
\begin{align*}
\|g_i^{\mathrm{SFT}}\|\,\|g_j^{\mathrm{SFT}}\|
&\le
(\|\mu\|+\|\bar{\epsilon}_i\|)(\|\mu\|+\|\bar{\epsilon}_j\|) \\
&\le
M_\mu^2(1+\eta)^2.
\end{align*}
Therefore,
\begin{equation}
|\cos(g_i^{\mathrm{SFT}},g_j^{\mathrm{SFT}})|
\ge
\frac{1-2\eta-\eta^2}{(1+\eta)^2}.
\label{eq:sft_cosine_lower_bound}
\end{equation}
When $\eta\ll 1$, the right-hand side remains strictly positive and close to $1$. Hence, even for distinct tasks, SFT gradients remain directionally aligned because the shared component $\mu$ cannot be removed by uniform averaging.

\paragraph{RL gradients cancel the shared dominant direction.}
For RL algorithms such as GRPO, the gradient is weighted by normalized advantages $\hat{A}_{i,k}$ satisfying
\[
\sum_{k=1}^G \hat{A}_{i,k}=0,
\qquad
\sum_{k=1}^G \hat{A}_{i,k}^2 = G.
\]
Thus,
\begin{align*}
g_i^{\mathrm{RL}}
&=
\frac{1}{G}\sum_{k=1}^G \hat{A}_{i,k}(\mu+\epsilon_{i,k}) \\
&=
\mu\left(\frac{1}{G}\sum_{k=1}^G \hat{A}_{i,k}\right)
+
\frac{1}{G}\sum_{k=1}^G \hat{A}_{i,k}\epsilon_{i,k} \\
&=
\frac{1}{G}\sum_{k=1}^G \hat{A}_{i,k}\epsilon_{i,k}.
\end{align*}
Similarly,
\[
g_j^{\mathrm{RL}}
=
\frac{1}{G}\sum_{m=1}^G \hat{A}_{j,m}\epsilon_{j,m}.
\]

We first upper bound the numerator:
\begin{align*}
|\langle g_i^{\mathrm{RL}}, g_j^{\mathrm{RL}}\rangle|
&=
\left|
\frac{1}{G^2}
\sum_{k=1}^G\sum_{m=1}^G
\hat{A}_{i,k}\hat{A}_{j,m}
\langle \epsilon_{i,k},\epsilon_{j,m}\rangle
\right| \\
&\le
\frac{1}{G^2}
\sum_{k=1}^G\sum_{m=1}^G
|\hat{A}_{i,k}|\,|\hat{A}_{j,m}|\,
|\langle \epsilon_{i,k},\epsilon_{j,m}\rangle| \\
&\le
\frac{\gamma_{i,j}\sigma^2}{G^2}
\left(\sum_{k=1}^G |\hat{A}_{i,k}|\right)
\left(\sum_{m=1}^G |\hat{A}_{j,m}|\right).
\end{align*}
By Cauchy--Schwarz,
\[
\sum_{k=1}^G |\hat{A}_{i,k}|
\le
\sqrt{G}\left(\sum_{k=1}^G \hat{A}_{i,k}^2\right)^{1/2}
=
G,
\]
and similarly for task $j$, yielding
\[
|\langle g_i^{\mathrm{RL}}, g_j^{\mathrm{RL}}\rangle|
\le
\gamma_{i,j}\sigma^2.
\]

We now lower bound the denominator. For one task,
\begin{align*}
\|g_i^{\mathrm{RL}}\|^2
&=
\frac{1}{G^2}
\sum_{k=1}^G\sum_{m=1}^G
\hat{A}_{i,k}\hat{A}_{i,m}
\langle \epsilon_{i,k},\epsilon_{i,m}\rangle \\
&=
\frac{1}{G^2}
\left(
\sum_{k=1}^G \hat{A}_{i,k}^2 \|\epsilon_{i,k}\|^2
+
\sum_{k\neq m}\hat{A}_{i,k}\hat{A}_{i,m}
\langle \epsilon_{i,k},\epsilon_{i,m}\rangle
\right).
\end{align*}
Using $\|\epsilon_{i,k}\|^2=\sigma^2$ and the intra-task coupling bound,
\begin{align*}
\|g_i^{\mathrm{RL}}\|^2
&\ge
\frac{1}{G^2}
\left(
G\sigma^2
-
\gamma_{\mathrm{in}}\sigma^2
\sum_{k\neq m} |\hat{A}_{i,k}|\,|\hat{A}_{i,m}|
\right).
\end{align*}
Since
\[
\sum_{k\neq m} |\hat{A}_{i,k}|\,|\hat{A}_{i,m}|
\le
\left(\sum_{k=1}^G |\hat{A}_{i,k}|\right)^2
\le
G^2,
\]
we obtain
\[
\|g_i^{\mathrm{RL}}\|^2
\ge
\frac{\sigma^2}{G}(1-\gamma_{\mathrm{in}}G).
\]
The same bound holds for task $j$. Combining the bounds above,
\begin{equation}
|\cos(g_i^{\mathrm{RL}},g_j^{\mathrm{RL}})|
\le
\frac{\gamma_{i,j}G}{1-\gamma_{\mathrm{in}}G}.
\label{eq:rl_cosine_upper_bound}
\end{equation}
In the high-dimensional regime where $\gamma_{\mathrm{in}}\to 0$,
\[
|\cos(g_i^{\mathrm{RL}},g_j^{\mathrm{RL}})|
\lesssim
\mathcal{O}(\gamma_{i,j}G).
\]

\paragraph{Interpretation.}
Eq.~\eqref{eq:sft_cosine_lower_bound} and Eq.~\eqref{eq:rl_cosine_upper_bound} reveal a directional contrast between SFT and RL. In SFT, averaging preserves the shared dominant direction $\mu$, which induces a strictly positive lower bound on cross-task cosine similarity. In RL, the zero-sum advantage weighting removes the shared direction algebraically, so the cosine similarity is controlled only by the residual coupling factor $\gamma_{i,j}$. Therefore, under weak residual coupling, RL gradients from different tasks become nearly orthogonal, which provides an additional explanation for the observed task coexistence. We emphasize, however, that this conclusion does not imply that all task pairs in RL are non-interfering. In particular, the residual coupling factor $\gamma_{i,j}$ need not always be very small. We discuss how to assess such cases in detail in Appendix~\ref{appendix:judge_for_diff}.

\section{Parallel RL}
\label{appendix:parallel_rl}

\subsection{Method for Judging Different Tasks}
\label{appendix:judge_for_diff}
As discussed in Section \ref{sec:ablation} of the main text, the efficacy of Parallel-RL is sensitive to the selection of parallel tasks. Not all task pairs qualify as "coexist"; certain tasks may exhibit great interference and are thus unsuitable for parallel training. In this section, building upon the analysis in Section \ref{sec:theoretical_analysis}, we conduct a preliminary exploration into defining whether two tasks are independent and can be trained in parallel.

Based on the qualitative analysis in Section \ref{sec:orthogonality}, we found that in two different tasks, high-reward trajectories activate disparate neural circuits and patterns. Thus an intuitive method to verify this is to analyze the distributions of the score functions $S$ for the two tasks. Specifically, we visualize these distributions using t-SNE. Given that t-SNE is a dimensionality reduction technique, the separation of score function distributions in the t-SNE latent space serves as a practical heuristic for determining task coexistence. As illustrated in Figure \ref{fig:score_func_distribution} of the main text and Figure \ref{fig:appendix_score_func_distribution} in the Appendix \ref{appendix:more visualizations}, the training data for the selected tasks in main experiment exhibit relative separation for RL, suggesting their suitability for parallel training. Next, we present a counter-example involving a task where the training data likely interferes with others. We selected SynLogic~\cite{liu2025synlogic}, a recent influential work that synthesizes data for logical reasoning games (e.g., Sudoku, Mazes), which we refer to as the \textbf{Game} task. We trained DeepSeek-R1-Distill-Qwen-1.5B on this game task following the same protocol as the main experiments and subsequently merged it with the \textbf{Math} and \textbf{Code} models trained in Section \ref{sec:parallel_exp}, respectively.
The results reveal significant interference between the Game task and both the Math and Code tasks. The results are shown in Figure \ref{fig:appendix_game_analysis_combined} (c) and (d). Consistent with Table \ref{tab:main_result} in the main text, we primarily compare the performance of the merged Parallel-RL model against the corresponding task-specific model. For instance, on the math task, we evaluate the performance change of Parallel-RL relative to the single-task math model. We found that Parallel-RL of Game and Math resulted in a performance decline of 4.0\% on math and 6.0\% on game. Similarly, a decline of 3.3\% and 4.9\% was observed in Parallel-RL of Game and Code. These results demonstrate significant interference. Consistent with our methodology, we generated score functions for the SynLogic data and visualized them using t-SNE. As depicted in Figure \ref{fig:appendix_game_analysis_combined} (a) and (b), there is a marked overlap between the distributions of the Game task and those of the Math and Code tasks. This observation preliminarily validates the effectiveness of our proposed criterion for task compatibility.

\begin{figure*}[t]
\centering
  \includegraphics[width=0.85\linewidth]{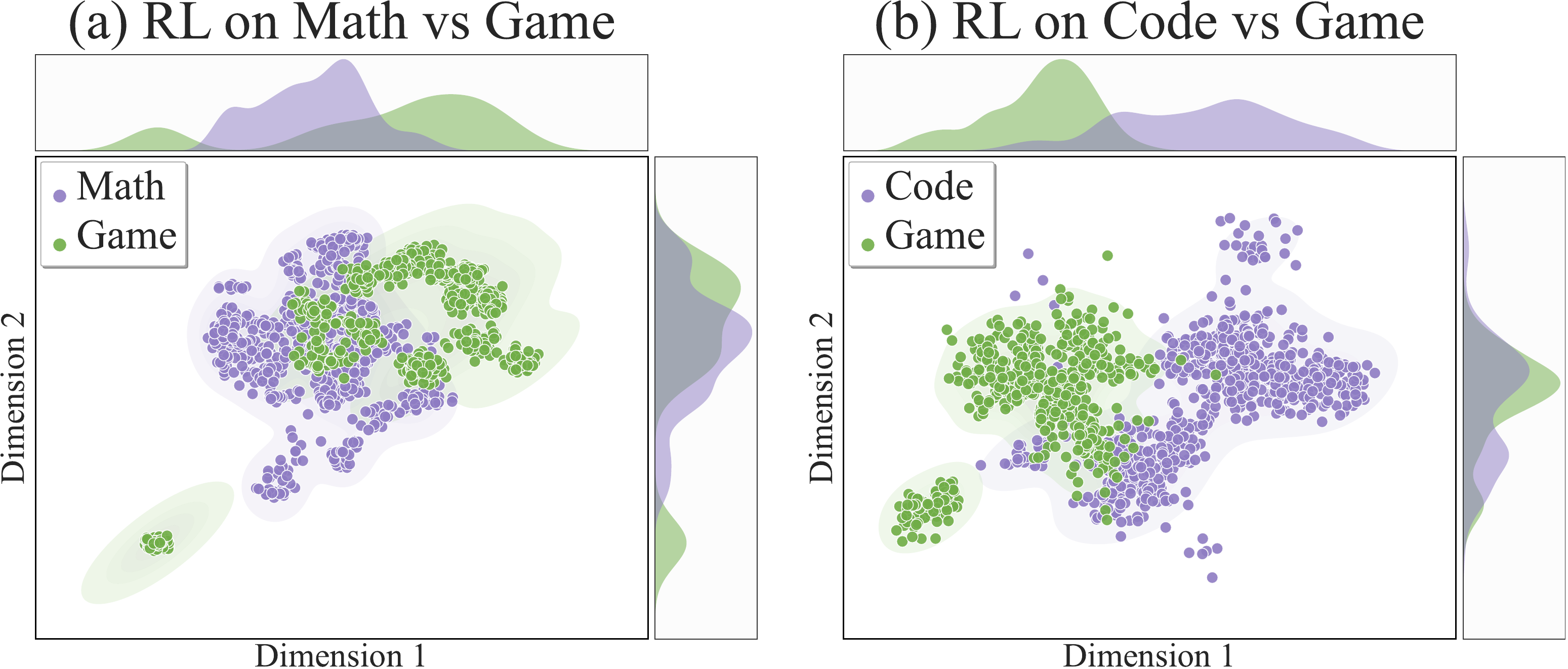}
  
  \vspace{1em} 
  
  \includegraphics[width=0.85\linewidth]{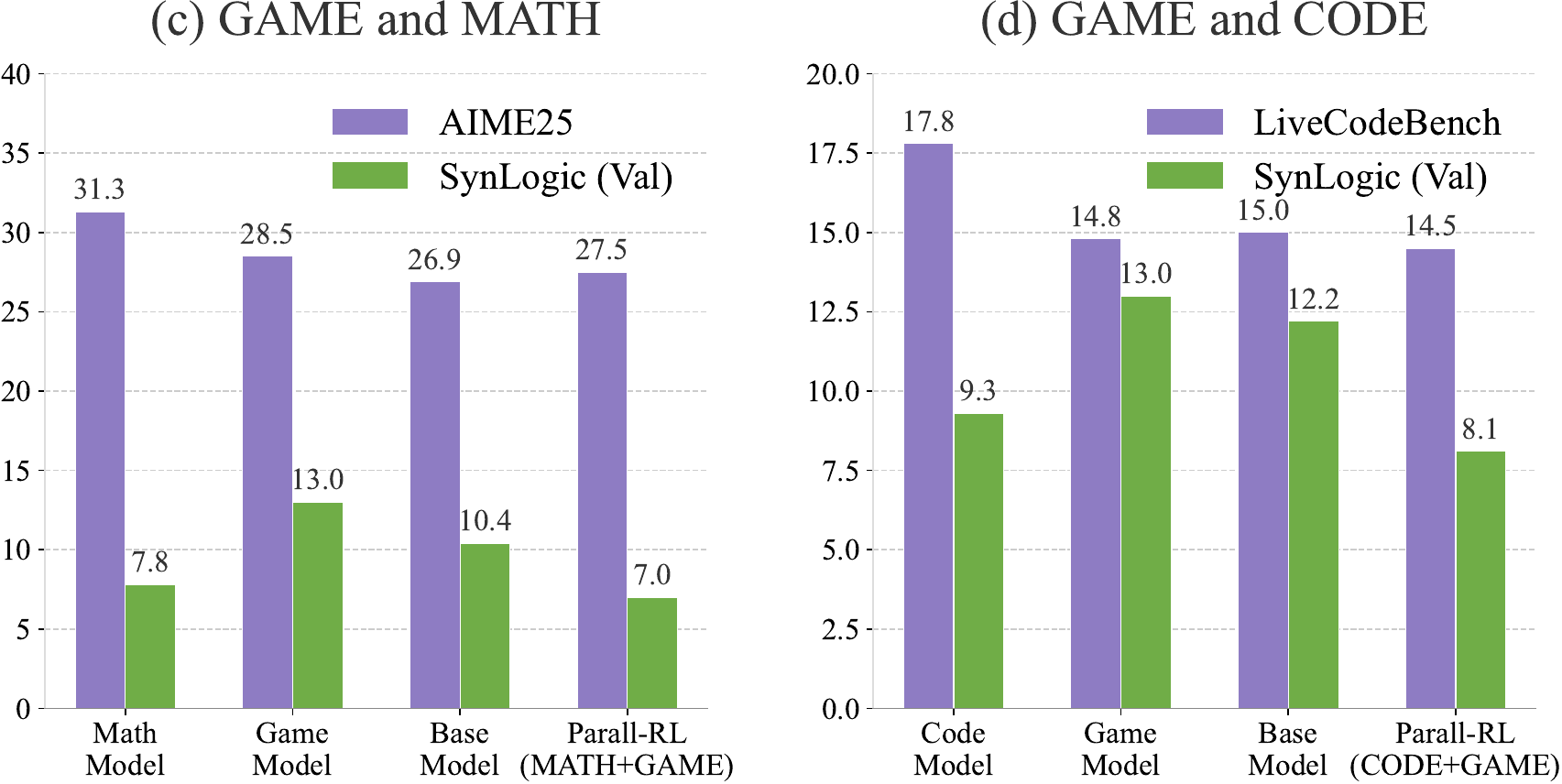}
  
  \caption{\textbf{Analysis of RL training on interfering tasks.} 
  \textbf{Top:} Distribution of the score function $S$ sampled during RL training on math, code and game. 
  \textbf{Bottom:} Corresponding performance comparison on Math (AIME25) and Code (LiveCodeBench) tasks. Training game models in parallel with math or code models can lead to significant performance losses.}
  \label{fig:appendix_game_analysis_combined}
\end{figure*}

\clearpage

\subsection{Single-Task Training Techniques: Trade-off in Parallel RL}

\begin{wrapfigure}{r}{0.6\textwidth}
\vspace{-10pt}
  \centering
  \includegraphics[width=0.6\textwidth]{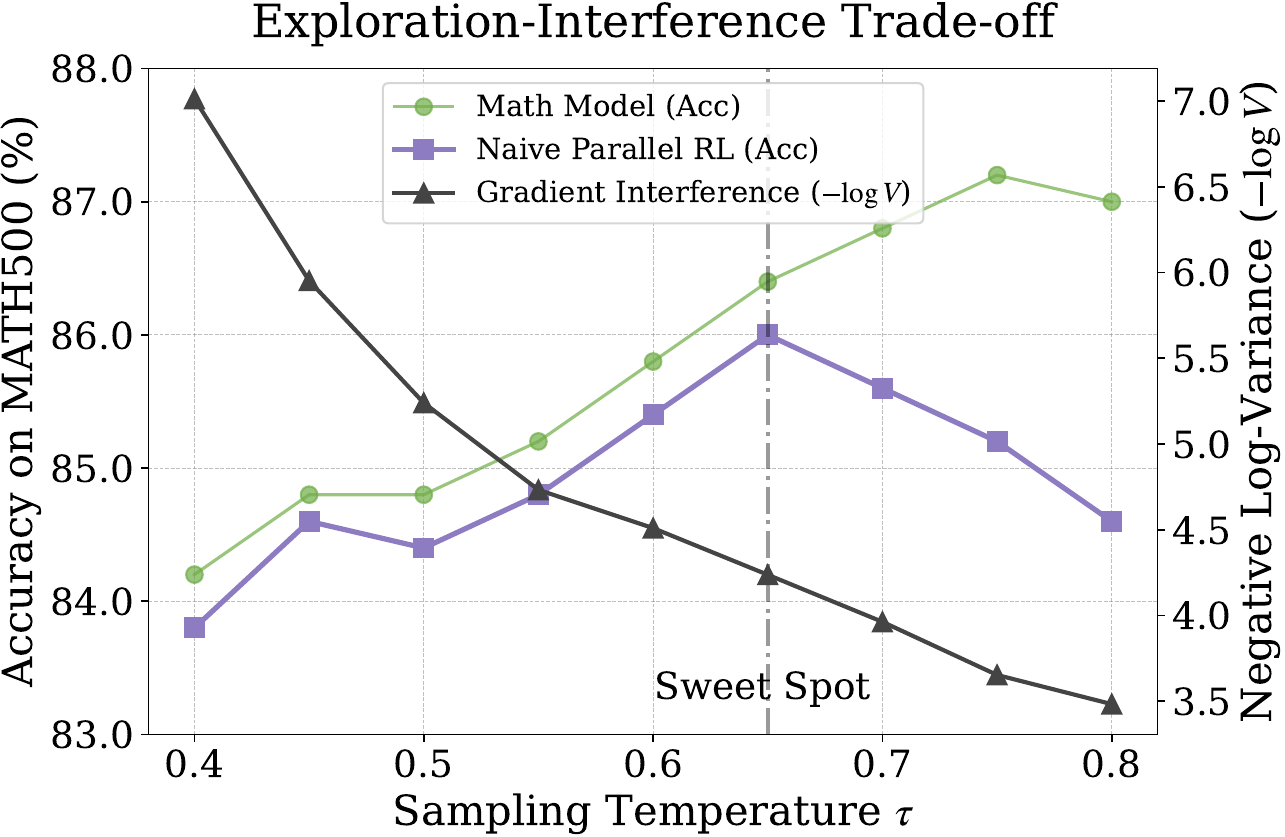}
  \caption{The Trade-off between Single-Task Exploration and Multi-Task Coexistence. When rollout temperature $\tau$ grows, we calculates the performance of Math Model, intra-group gradient variance (measured by $-\log V$), and Naive Parallel-RL accuracy on MATH500. Increasing $\tau$ improves single-task performance but also increases gradient variance, leading to higher gradient interference. As a result, the non-monotonic trend of Parallel-RL performance reveals a sweet spot that balances exploration benefits and interference costs.}
  \label{fig:trade_off}
  \vspace{-15pt}
\end{wrapfigure}

Based on the upper bound derived in Section \ref{sec:upper bound}, the gradient interference of RL between different tasks is limited by the intra-group variance of rollouts, specifically $|\mathcal{I}_{\text{RL}}(i, j)| \le V_i \cdot V_j$. A deeper analysis of this theory reveals an interesting trade-off in Parallel-RL training: \textit{the trade-off between single-task exploration and multi-task interference.}.

\begin{itemize}
    \item \textbf{Exploration Benefit:} In single-task RL, increasing exploration and output diversity helps the model discover low-probability but high-reward reasoning trajectory~\cite{guo2025deepseek}. As a result, higher output diversity often leads to improved convergence and performance in single-task RL training.
    
    \item \textbf{Interference Cost:} However, increased exploration also amplifies the diversity among the $G$ rollout trajectories generated for the same prompt, thereby enlarging the intra-group variance $V$. According to Theorem \ref{thm:upper_bounds}, the interference between tasks rises accordingly. While high exploration benefits single-task training, it increases the gradient interference. This may cause greater conflicts and performance loss during the merging phase.
\end{itemize}

\textbf{Empirical Analysis.} 

To verify this, we conducted an ablation study using sampling temperature $\tau$ to control exploration. Since $\tau$ directly affects the entropy of the policy $\pi_\theta$, it serves as an effective control variable. Based on preliminary tests, we set the temperature range to $\tau \in [0.4, 0.8]$. We varied the rollout temperature for the Math Model training while keeping other tasks constant. We evaluated the accuracy of the Naive Parallel RL model on MATH500 and measured the negative log-variance ($-\log V$) of the intra-group variance $V$. Due to resource constraints, we performed this analysis using LoRA on a data subset.

As shown in Figure \ref{fig:trade_off}, the performance of Math Model improves as $\tau$ increases from 0.4 to 0.8. At the same time $-\log V$ consistently decreases. This indicates a rise in intra-group variance and a corresponding increase in the interference upper bound. Crucially, the Naive Parallel-RL performance exhibits a non-monotonic trend: accuracy first increases and then degrades as $\tau$ grows. At low temperatures, interference is small but single-task learning is insufficient; at high temperatures, single-task performance improves but excessive inter-task interference degrades the merged model. This curve clearly illustrates the exploration–interference trade-off.

Notably, around $\tau \approx 0.65$, Parallel-RL achieves the best performance, achieving an optimal balance between single-task capability and multi-task interference. This result highlights that, under Parallel-RL, optimizing single-task exploration alone is insufficient; controlling multi-task interference is equally crucial for effective task merging with low loss.

\clearpage

\section{More Experimental Results}


In Section \ref{sec:parallel_exp} of the main text, building upon our comprehensive empirical and theoretical analyses, we proposed the \textbf{Parallel-RL} training paradigm. In this appendix, we present an extensive set of additional experiments to further substantiate the efficacy and robustness of this approach. Specifically, Section \ref{appendix:more_main_results} provides the complete results corresponding to the main experiments discussed in Table \ref{tab:main_result}. Section \ref{appendix:more_ablation} presents expanded ablation studies (complementing Table \ref{tab:ablation}) to demonstrate how Parallel-RL effectively decouples capabilities across diverse tasks. Section \ref{appendix:more visualizations} visualizes the distributions of score functions across tasks as referenced in Section \ref{sec:orthogonality}. Finally, in Section \ref{appendix:pass_at_k}, we conduct a Pass@$k$ evaluation of the Parallel-RL model to provide a more rigorous assessment of its reasoning boundaries.

\subsection{More Main Experiment Results}
\label{appendix:more_main_results}

In this section, we present the results of Parallel-RL models derived from two additional training settings on the 1.5B base model: (1) GRPO training using LoRA, and (2) full-parameter training using PPO. 

We found that when using LoRA for training, the performance loss caused by parameter interference is minimal. Even direct summation of $\Delta W_i$ (Naive Parallel-RL (sum)) maintains 98\% of the accuracy of the corresponding single-task models. Adapted Parallel-RL even outperforms Single Task RL, achieving an average performance of 103\% of Single Task RL. This indicates that the reasoning capability of task-specific LoRA modules is largely preserved during the merging process.

For PPO, Parallel-RL is equally effective. Even Naive Parallel-RL model achieves an average performance improvement of 5.3\% compared to the base model. Adapted Parallel-RL showed an average improvement of 9.1\%. This demonstrates the effectiveness of the Parallel RL paradigm. However, the performance of the single-task model that Naive Parallel (sum) of PPO retained (92.9\%) is lower than that of GRPO shown in Table \ref{tab:main_result} (94.2\%). This aligns with our theoretical analysis in Appendix \ref{appendix: other_rl_proof}. GRPO-based parallel training exhibits lower interference compared to PPO.

\begin{table*}[!htbp]
\caption{The accuracy (\%) of several multi task training methods. Superscripts denote the performance gap relative to the \textbf{corresponding Single-Task Baseline}. $\Delta$ Base shows the average performance improvement over the Base Model. Retention (\%) represents the \textbf{percentage of performance retained} compared to the Single-Task Baseline.}
\vspace{0.1in}
\centering
\small
\resizebox{1.0 \linewidth}{!}{
\begin{tabular}{clcccccccc}
\toprule[1pt]
\multirow{2}{*}{\normalsize \textbf{Base Model}} & \multicolumn{1}{c}{\multirow{2}{*}{\normalsize \textbf{Method}}} & \multicolumn{2}{c}{\normalsize \textbf{Math}} & \multicolumn{2}{c}{\normalsize \textbf{Science}} & \multicolumn{1}{c}{\normalsize \textbf{Logic}} & \multicolumn{1}{c}{\normalsize \textbf{Code}} & \multicolumn{2}{c}{\normalsize \textbf{Average Statistic}} \\ \cline{3-10} 
    & \multicolumn{1}{c}{}                        & MATH500    & AIME2025   & MMLU     & GPQA     & KK    & LiveCodeBench & $\Delta$ Base & Retention (\%) \\ 
    \hline

\multirow{18}{*}{\shortstack{DeepSeek-R1- \\ Distill-Qwen- \\ 1.5B (LoRA)}} 
& \multicolumn{9}{c}{\emph{Baselines (LoRA GRPO)}} \\
& Base Model & 82.0 & 26.9 & 34.9 & 32.3 & 31.0 & 15.0 & - & - \\
&Single-Task SFT & 84.4 & 28.3 & 43.5 & 34.8 & 35.0 & 15.8 & +3.3 & - \\
&Single-Task RL & 85.6 & 31.3 & 48.5 & 35.9 & 38.0 & 17.8 & +5.8 & - \\
 \cline{2-10}

& \multicolumn{9}{c}{\emph{Previous Multi Task Training Paradigms}} \\
& Mixed Data SFT  & $84.6^{\textcolor[HTML]{00008B}{+0.2}}$ & $27.1^{\textcolor[HTML]{8B0000}{-1.2}}$ & $38.9^{\textcolor[HTML]{8B0000}{-4.6}}$ & $32.3^{\textcolor[HTML]{8B0000}{-2.5}}$ & $34.0^{\textcolor[HTML]{8B0000}{-1.0}}$ & $16.0^{\textcolor[HTML]{00008B}{+0.2}}$ & +1.8 & - \\
& Multi Stage SFT & $78.2^{\textcolor[HTML]{8B0000}{-6.2}}$ & $19.2^{\textcolor[HTML]{8B0000}{-9.1}}$ & $31.1^{\textcolor[HTML]{8B0000}{-12.4}}$ & $25.3^{\textcolor[HTML]{8B0000}{-9.5}}$ & $9.0^{\textcolor[HTML]{8B0000}{-26.0}}$ & $14.3^{\textcolor[HTML]{8B0000}{-1.5}}$ & -7.5 & - \\ 
& Mixed Data RL & $85.2^{\textcolor[HTML]{8B0000}{-0.4}}$ & $30.5^{\textcolor[HTML]{8B0000}{-0.8}}$ & $43.2^{\textcolor[HTML]{8B0000}{-5.3}}$ & $34.8^{\textcolor[HTML]{8B0000}{-1.1}}$ & $37.0^{\textcolor[HTML]{8B0000}{-1.0}}$ & $15.7^{\textcolor[HTML]{8B0000}{-2.1}}$ & +4.1 & - \\
& Multi Stage RL & $86.6^{\textcolor[HTML]{00008B}{+1.0}}$ & $31.5^{\textcolor[HTML]{00008B}{+0.2}}$ & $49.3^{\textcolor[HTML]{00008B}{+0.8}}$ & $\textbf{36.4}^{\textcolor[HTML]{00008B}{+0.5}}$ & $43.0^{\textcolor[HTML]{00008B}{+5.0}}$ & $\textbf{17.3}^{\textcolor[HTML]{8B0000}{-0.5}}$ & +7.0 & - \\ 
 \cline{2-10}

& \multicolumn{9}{c}{\emph{Parallel Training}} \\
&Naive Parallel-SFT \textit{(mean)} & $51.0^{\textcolor[HTML]{8B0000}{-33.4}}$ & $17.3^{\textcolor[HTML]{8B0000}{-11.0}}$ & $26.7^{\textcolor[HTML]{8B0000}{-16.8}}$ & $25.8^{\textcolor[HTML]{8B0000}{-9.0}}$ & $20.0^{\textcolor[HTML]{8B0000}{-15.0}}$ & $9.7^{\textcolor[HTML]{8B0000}{-6.1}}$ & -11.9 & 62.2 \\
&Naive Parallel-SFT \textit{(sum)} & $52.6^{\textcolor[HTML]{8B0000}{-31.8}}$ & $16.9^{\textcolor[HTML]{8B0000}{-11.4}}$ & $27.1^{\textcolor[HTML]{8B0000}{-16.4}}$ & $24.7^{\textcolor[HTML]{8B0000}{-10.1}}$ & $25.0^{\textcolor[HTML]{8B0000}{-10.0}}$ & $8.8^{\textcolor[HTML]{8B0000}{-7.0}}$ & -11.2 & 64.1 \\
&Naive Parallel-RL \textit{(mean)} & $85.4^{\textcolor[HTML]{8B0000}{-0.2}}$ & $28.5^{\textcolor[HTML]{8B0000}{-2.8}}$ & $45.9^{\textcolor[HTML]{8B0000}{-2.6}}$ & $31.8^{\textcolor[HTML]{8B0000}{-4.1}}$ & $39.0^{\textcolor[HTML]{00008B}{+1.0}}$ & $16.9^{\textcolor[HTML]{8B0000}{-0.9}}$ & +4.2 & 96.3 \\
&Naive Parallel-RL \textit{(sum)} & $85.2^{\textcolor[HTML]{8B0000}{-0.4}}$ & $29.2^{\textcolor[HTML]{8B0000}{-2.1}}$ & $46.5^{\textcolor[HTML]{8B0000}{-2.0}}$ & $32.8^{\textcolor[HTML]{8B0000}{-3.1}}$ & $40.0^{\textcolor[HTML]{00008B}{+2.0}}$ & $16.9^{\textcolor[HTML]{8B0000}{-0.9}}$ & +4.8 & 97.5 \\
&TIES Parallel-RL & $86.0^{\textcolor[HTML]{00008B}{+0.4}}$ & $31.0^{\textcolor[HTML]{8B0000}{-0.3}}$ & $47.0^{\textcolor[HTML]{8B0000}{-1.5}}$ & $34.8^{\textcolor[HTML]{8B0000}{-1.1}}$ & $42.0^{\textcolor[HTML]{00008B}{+4.0}}$ & $16.3^{\textcolor[HTML]{8B0000}{-1.5}}$ & +5.8 & 100.0 \\
&SVD Parallel-RL & $85.4^{\textcolor[HTML]{8B0000}{-0.2}}$ & $31.5^{\textcolor[HTML]{00008B}{+0.2}}$ & $45.8^{\textcolor[HTML]{8B0000}{-2.7}}$ & $32.8^{\textcolor[HTML]{8B0000}{-3.1}}$ & $37.0^{\textcolor[HTML]{8B0000}{-1.0}}$ & $15.6^{\textcolor[HTML]{8B0000}{-2.2}}$ & +4.3 & 96.5 \\
&Adapted Parallel-RL & $\textbf{86.6}^{\textcolor[HTML]{00008B}{+1.0}}$ & $\textbf{32.1}^{\textcolor[HTML]{00008B}{+0.8}}$ & $\textbf{48.3}^{\textcolor[HTML]{8B0000}{-0.2}}$ & $35.4^{\textcolor[HTML]{8B0000}{-0.5}}$ & $\textbf{45.0}^{\textcolor[HTML]{00008B}{+7.0}}$ & $\textbf{17.3}^{\textcolor[HTML]{8B0000}{-0.5}}$ & +7.1 & \textbf{102.9} \\

    \hline

\multirow{18}{*}{\shortstack{DeepSeek-R1- \\ Distill-Qwen- \\ 1.5B (PPO)}}  
& \multicolumn{9}{c}{\emph{Baselines (Full Parameter PPO)}} \\
& Base Model (1.5B) & 82.0 & 26.9 & 34.9 & 32.3 & 31.0 & 15.0 & - & - \\
&Single-Task SFT & 85.8 & 31.5 & 49.1 & 35.4 & 37.0 & 16.3 & +5.5 & - \\
&Single-Task RL & 86.6 & 31.5 & 49.8 & 37.9 & 46.0 & 20.3 & +8.3 & - \\
 \cline{2-10}

& \multicolumn{9}{c}{\emph{Previous Multi Task Training Paradigms}} \\
& Mixed Data SFT  & $85.6^{\textcolor[HTML]{8B0000}{-0.2}}$ & $30.1^{\textcolor[HTML]{8B0000}{-1.4}}$ & $48.5^{\textcolor[HTML]{8B0000}{-0.6}}$ & $34.3^{\textcolor[HTML]{8B0000}{-1.1}}$ & $38.0^{\textcolor[HTML]{00008B}{+1.0}}$ & $17.0^{\textcolor[HTML]{00008B}{+0.7}}$ & +5.2 & - \\
& Multi Stage SFT & $71.0^{\textcolor[HTML]{8B0000}{-14.8}}$ & $21.8^{\textcolor[HTML]{8B0000}{-9.7}}$ & $33.2^{\textcolor[HTML]{8B0000}{-15.9}}$ & $23.4^{\textcolor[HTML]{8B0000}{-12.0}}$ & $11.0^{\textcolor[HTML]{8B0000}{-26.0}}$ & $12.1^{\textcolor[HTML]{8B0000}{-4.2}}$ & -8.3 & - \\ 
& Mixed Data RL & $85.0^{\textcolor[HTML]{8B0000}{-1.6}}$ & $30.0^{\textcolor[HTML]{8B0000}{-1.5}}$ & $48.2^{\textcolor[HTML]{8B0000}{-1.6}}$ & $37.4^{\textcolor[HTML]{8B0000}{-0.5}}$ & $41.0^{\textcolor[HTML]{8B0000}{-5.0}}$ & $16.0^{\textcolor[HTML]{8B0000}{-4.3}}$ & +5.9 & - \\
& Multi Stage RL & $87.8^{\textcolor[HTML]{00008B}{+1.2}}$ & $33.5^{\textcolor[HTML]{00008B}{+2.0}}$ & $53.5^{\textcolor[HTML]{00008B}{+3.7}}$ & $39.4^{\textcolor[HTML]{00008B}{+1.5}}$ & $43.0^{\textcolor[HTML]{8B0000}{-3.0}}$ & $\textbf{21.8}^{\textcolor[HTML]{00008B}{+1.5}}$ & +9.5 & - \\ 
 \cline{2-10}

& \multicolumn{9}{c}{\emph{Parallel Training}} \\
&Naive Parallel-SFT \textit{(mean)} & $63.2^{\textcolor[HTML]{8B0000}{-22.6}}$ & $19.2^{\textcolor[HTML]{8B0000}{-12.3}}$ & $26.9^{\textcolor[HTML]{8B0000}{-22.2}}$ & $25.3^{\textcolor[HTML]{8B0000}{-10.1}}$ & $26.0^{\textcolor[HTML]{8B0000}{-11.0}}$ & $9.3^{\textcolor[HTML]{8B0000}{-7.0}}$ & -8.7 & 66.6 \\
&Naive Parallel-SFT \textit{(sum)} & $61.2^{\textcolor[HTML]{8B0000}{-24.6}}$ & $20.0^{\textcolor[HTML]{8B0000}{-11.5}}$ & $25.9^{\textcolor[HTML]{8B0000}{-23.2}}$ & $26.8^{\textcolor[HTML]{8B0000}{-8.6}}$ & $23.0^{\textcolor[HTML]{8B0000}{-14.0}}$ & $10.1^{\textcolor[HTML]{8B0000}{-6.2}}$ & -9.2 & 65.5 \\
&Naive Parallel-RL \textit{(mean)} & $85.4^{\textcolor[HTML]{8B0000}{-1.2}}$ & $27.3^{\textcolor[HTML]{8B0000}{-4.2}}$ & $49.1^{\textcolor[HTML]{8B0000}{-0.7}}$ & $33.8^{\textcolor[HTML]{8B0000}{-4.1}}$ & $38.0^{\textcolor[HTML]{8B0000}{-8.0}}$ & $16.1^{\textcolor[HTML]{8B0000}{-4.2}}$ & +4.6 & 91.7 \\
&Naive Parallel-RL \textit{(sum)} & $84.8^{\textcolor[HTML]{8B0000}{-1.8}}$ & $28.3^{\textcolor[HTML]{8B0000}{-3.2}}$ & $48.9^{\textcolor[HTML]{8B0000}{-0.9}}$ & $34.3^{\textcolor[HTML]{8B0000}{-3.6}}$ & $40.0^{\textcolor[HTML]{8B0000}{-6.0}}$ & $16.5^{\textcolor[HTML]{8B0000}{-3.8}}$ & +5.1 & 92.9 \\
&TIES Parallel-RL & $86.0^{\textcolor[HTML]{8B0000}{-0.6}}$ & $30.4^{\textcolor[HTML]{8B0000}{-1.1}}$ & $51.4^{\textcolor[HTML]{00008B}{+1.6}}$ & $36.9^{\textcolor[HTML]{8B0000}{-1.0}}$ & $44.0^{\textcolor[HTML]{8B0000}{-2.0}}$ & $18.5^{\textcolor[HTML]{8B0000}{-1.8}}$ & +7.5 & 98.2 \\
&SVD Parallel-RL & $85.4^{\textcolor[HTML]{8B0000}{-1.2}}$ & $29.6^{\textcolor[HTML]{8B0000}{-1.9}}$ & $49.9^{\textcolor[HTML]{00008B}{+0.1}}$ & $35.4^{\textcolor[HTML]{8B0000}{-2.5}}$ & $39.0^{\textcolor[HTML]{8B0000}{-7.0}}$ & $17.5^{\textcolor[HTML]{8B0000}{-2.8}}$ & +5.8 & 94.4 \\
&Adapted Parallel-RL & $\textbf{87.0}^{\textcolor[HTML]{00008B}{+0.4}}$ & $\textbf{32.1}^{\textcolor[HTML]{00008B}{+0.6}}$ & $\textbf{54.4}^{\textcolor[HTML]{00008B}{+4.6}}$ & $\textbf{38.4}^{\textcolor[HTML]{00008B}{+0.5}}$ & $\textbf{45.0}^{\textcolor[HTML]{8B0000}{-1.0}}$ & $19.5^{\textcolor[HTML]{8B0000}{-0.8}}$ & +9.1 & \textbf{101.6} \\

  \bottomrule[1pt]

\end{tabular}
}
\label{tab:all_parallel_main_results}
\end{table*}
\subsection{More Ablation Study Results}
\label{appendix:more_ablation}

In this section, we present more ablation studies on the trained Naive Parallel-RL model to evaluate the impact of excluding specific task updates on the overall multi-task performance. As shown in Table \ref{tab:comprehensive_ablation}, removing the update $\Delta W_i$ of specific task from the Naive Parallel RL model results in a sharp performance drop on that specific task, with an average change of -6.0\% ($\Delta_{Target}$). However, the performance on the remaining tasks remains robust or even slightly improves, with an average change of +0.9\% ($\Delta_{Others}$). This further shows that parameter updates $\Delta W_i$ derived from RL on different tasks exhibit low coherence and high decoupling. The update vector for a given task primarily affects the performance of that task with minimal impact on others. This suggests that beyond improving multi-task training efficiency, the Parallel-RL paradigm enhances the \textbf{modularity and decoupling of task capabilities in the final model}. We envision that future advancements in Parallel-RL algorithms could enable the on-demand composition of model capabilities, allowing for flexible adaptation to specialized real-world environments with specific requirements.

\begin{table*}[!htbp]
\caption{Comprehensive ablation study across different model settings. Superscripts denote the performance gap ($\Delta$) relative to the full \textbf{Parallel-RL} model. \textcolor[HTML]{8B0000}{Red} indicates drop, \textcolor[HTML]{00008B}{Blue} indicates improvement.}
\vspace{0.1in}
\centering
\resizebox{1.0\linewidth}{!}{
\begin{tabular}{clcccccc}
\toprule[1pt]
\multirow{2}{*}{Model Setting} & \multicolumn{1}{c}{\multirow{2}{*}{Method}} & MATH500 & MMLU & KK & LCB & $\Delta_{Target}$ & $\Delta_{Others}$ \\ 
 & & (Math) & (Science) & (Logic) & (Code) & (Impact) & (Side-effect) \\ \hline

\multirow{8}{*}{1.5B LoRA} 
& \multicolumn{7}{c}{\emph{Naive Parallel RL}} \\
& Parallel-RL & 85.2 & 46.5 & 40.0 & 16.9 & - & - \\ \cmidrule{2-8}
& - w/o Math & $82.4^{\textcolor[HTML]{8B0000}{-2.8}}$ & $47.3^{\textcolor[HTML]{00008B}{+0.8}}$ & $38.0^{\textcolor[HTML]{8B0000}{-2.0}}$ & $17.1^{\textcolor[HTML]{00008B}{+0.2}}$ & -2.8 & -0.3 \\
& - w/o Science & $85.0^{\textcolor[HTML]{8B0000}{-0.2}}$ & $38.7^{\textcolor[HTML]{8B0000}{-7.8}}$ & $37.0^{\textcolor[HTML]{8B0000}{-3.0}}$ & $17.5^{\textcolor[HTML]{00008B}{+0.6}}$ & -7.8 & -0.9 \\
& - w/o Logic & $84.6^{\textcolor[HTML]{8B0000}{-0.6}}$ & $47.1^{\textcolor[HTML]{00008B}{+0.6}}$ & $32.0^{\textcolor[HTML]{8B0000}{-8.0}}$ & $16.3^{\textcolor[HTML]{8B0000}{-0.6}}$ & -8.0 & -0.2 \\
& - w/o Code & $85.8^{\textcolor[HTML]{00008B}{+0.6}}$ & $46.3^{\textcolor[HTML]{8B0000}{-0.2}}$ & $39.0^{\textcolor[HTML]{8B0000}{-1.0}}$ & $14.7^{\textcolor[HTML]{8B0000}{-2.2}}$ & -2.2 & -0.2 \\ \cmidrule{2-8}
& \multicolumn{7}{c}{\emph{Baselines}} \\ 
& Single-Task RL & 85.6 & 48.5 & 38.0 & 17.8 & - & - \\
& Base Model & 82.0 & 34.9 & 31.0 & 15.0 & - & - \\ \hline

\multirow{8}{*}{1.5B PPO} 
& \multicolumn{7}{c}{\emph{Naive Parallel RL}} \\
& Parallel-RL & 84.8 & 48.9 & 36.0 & 16.5 & - & - \\ \cmidrule{2-8}
& - w/o Math & $82.8^{\textcolor[HTML]{8B0000}{-2.0}}$ & $48.5^{\textcolor[HTML]{8B0000}{-0.4}}$ & $38.0^{\textcolor[HTML]{00008B}{+2.0}}$ & $18.6^{\textcolor[HTML]{00008B}{+2.1}}$ & -2.0 & +1.2 \\
& - w/o Science & $85.8^{\textcolor[HTML]{00008B}{+1.0}}$ & $35.6^{\textcolor[HTML]{8B0000}{-13.3}}$ & $44.0^{\textcolor[HTML]{00008B}{+8.0}}$ & $18.1^{\textcolor[HTML]{00008B}{+1.6}}$ & -13.3 & +3.5 \\
& - w/o Logic & $84.6^{\textcolor[HTML]{8B0000}{-0.2}}$ & $48.3^{\textcolor[HTML]{8B0000}{-0.6}}$ & $30.0^{\textcolor[HTML]{8B0000}{-6.0}}$ & $16.8^{\textcolor[HTML]{00008B}{+0.3}}$ & -6.0 & -0.2 \\
& - w/o Code & $87.4^{\textcolor[HTML]{00008B}{+2.6}}$ & $50.3^{\textcolor[HTML]{00008B}{+1.4}}$ & $45.0^{\textcolor[HTML]{00008B}{+9.0}}$ & $12.5^{\textcolor[HTML]{8B0000}{-4.0}}$ & -4.0 & +4.3 \\ \cmidrule{2-8}
& \multicolumn{7}{c}{\emph{Baselines}} \\
& Single-Task RL & 86.6 & 49.8 & 46.0 & 20.3 & - & - \\
& Base Model & 82.0 & 34.9 & 31.0 & 15.0 & - & - \\ \hline

\multirow{8}{*}{7B GRPO} 
& \multicolumn{7}{c}{\emph{Naive Parallel RL}} \\
& Parallel-RL & 93.6 & 59.6 & 68.0 & 38.0 & - & - \\ \cmidrule{2-8}
& - w/o Math & $90.2^{\textcolor[HTML]{8B0000}{-3.4}}$ & $59.1^{\textcolor[HTML]{8B0000}{-0.5}}$ & $67.0^{\textcolor[HTML]{8B0000}{-1.0}}$ & $41.0^{\textcolor[HTML]{00008B}{+3.0}}$ & -3.4 & +0.5 \\
& - w/o Science & $94.2^{\textcolor[HTML]{00008B}{+0.6}}$ & $49.7^{\textcolor[HTML]{8B0000}{-9.9}}$ & $68.0^{\textcolor[HTML]{00008B}{+0.0}}$ & $39.5^{\textcolor[HTML]{00008B}{+1.5}}$ & -9.9 & +0.7 \\
& - w/o Logic & $93.8^{\textcolor[HTML]{00008B}{+0.2}}$ & $60.2^{\textcolor[HTML]{00008B}{+0.6}}$ & $60.0^{\textcolor[HTML]{8B0000}{-8.0}}$ & $38.3^{\textcolor[HTML]{00008B}{+0.3}}$ & -8.0 & +0.4 \\
& - w/o Code & $95.2^{\textcolor[HTML]{00008B}{+1.6}}$ & $61.8^{\textcolor[HTML]{00008B}{+2.2}}$ & $69.0^{\textcolor[HTML]{00008B}{+1.0}}$ & $33.9^{\textcolor[HTML]{8B0000}{-4.1}}$ & -4.1 & +1.6 \\ \cmidrule{2-8}
& \multicolumn{7}{c}{\emph{Baselines}} \\
& Single-Task RL & 94.6 & 63.5 & 68.0 & 42.3 & - & - \\
& Base Model & 91.8 & 51.8 & 62.0 & 36.5 & - & - \\

\bottomrule[1pt]
\end{tabular}
}
\label{tab:comprehensive_ablation}
\end{table*}

\subsection{More Visualizations}

In Section \ref{sec:quantify}, we presented a comparative visualization of the score function $S$ distributions for Math and Science tasks trained via RL and SFT, respectively. In this section, we provide additional t-SNE visualizations of score function distributions covering a broader range of tasks. As illustrated in Figure \ref{fig:appendix_score_func_distribution}, the top row displays the distributions for different tasks trained using RL. We observe a clear separation between task clusters with minimal overlap. Conversely, the bottom row shows the score functions $S$ for SFT, where we observe a significant degree of overlap across different tasks. This contrast serves as empirical evidence reflecting the high interference and conflicts of parameter updates inherent to SFT training, as opposed to the incoherence and coexistence nature of RL updates.

\label{appendix:more visualizations}
\begin{figure*}
\centering
  \includegraphics[width=\linewidth]{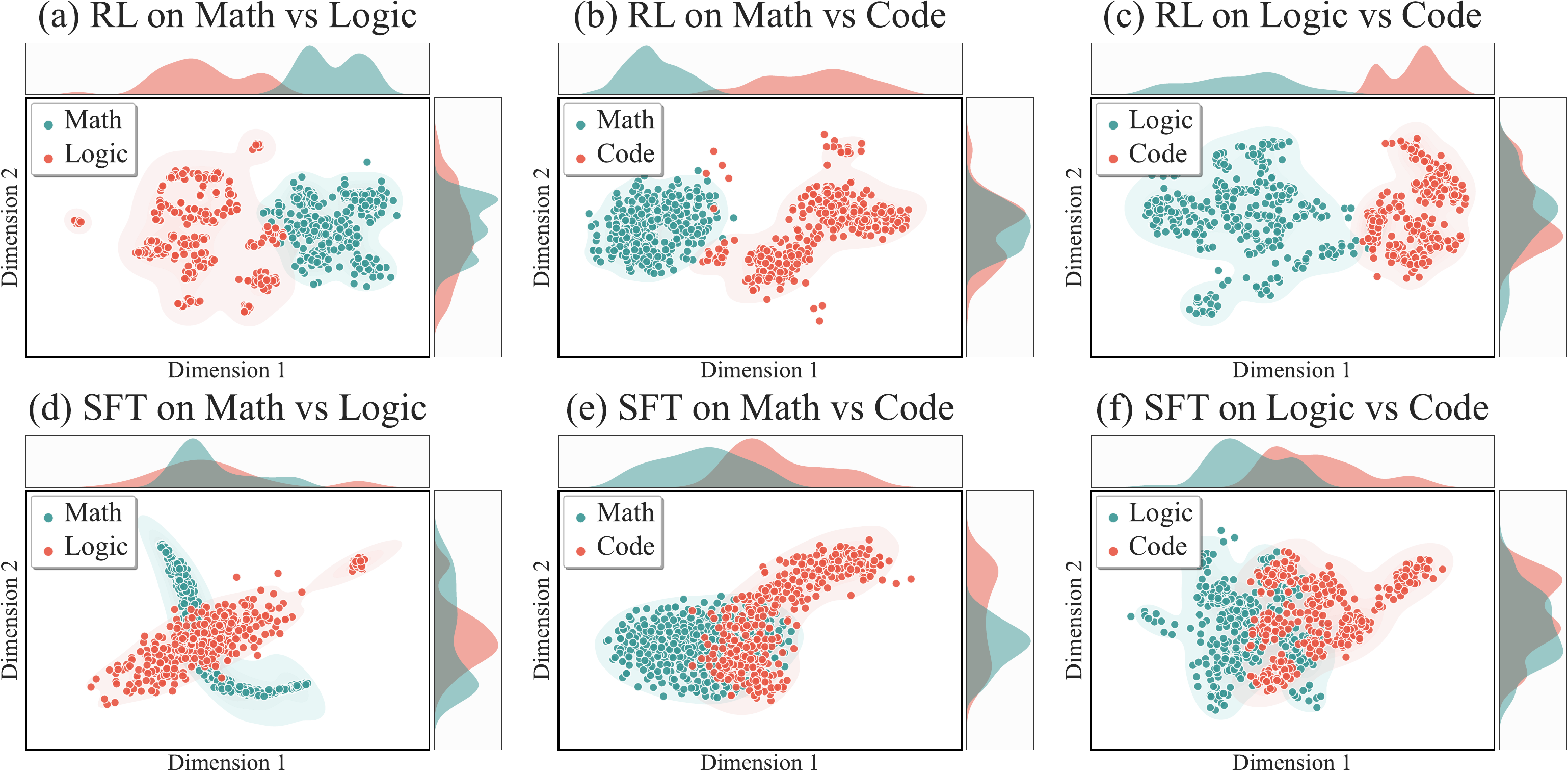}
  \vspace{-5pt}
  \caption {Distribution of the score function $S$ sampled during training by RL and SFT on different tasks (Math, Code and Logic). We use t-SNE to reduce the dimensionality of the high-dimensional score function vectors for visualization.}
  \label{fig:appendix_score_func_distribution}
\end{figure*}

\subsection{Pass@K Analysis}
\label{appendix:pass_at_k}

Accuracy evaluation can be sensitive to the decoding strategy. To provide a more robust evaluation of the stability and reasoning capability boundaries of the Parallel-RL paradigm, we adopt the \textbf{pass@$k$} metric. This metric measures the probability that at least one out of $k$ generated samples for a given problem is correct, serving as a proxy for the model's potential reasoning upper bound. In this section, we conduct a pass@$k$ analysis on the \textbf{Naive Parallel-RL} described in Section \ref{sec:parallel_exp}. Specifically, we select four representative benchmarks for our four tasks: AIME2025 (Math), GPQA Diamond (Science), Knights \& Knaves (Logic), and LiveCodeBench (Code). For each task, we compare three models: the \textbf{Base Model}, the task-specific \textbf{Single-Task RL} model trained on corresponding task, and the \textbf{Naive Parallel-RL} model.

To reduce variance associated with small sample sizes, we employ an unbiased estimation method~\cite{chen2021evaluatinglargelanguagemodels, yue2025doesreinforcementlearningreally}. We do not strictly sample $k$ times. Instead, we generate a larger pool of $n$ samples ($n \ge k$) and compute the unbiased estimator:

\vspace{-10pt}
\begin{equation*}
\text{pass}@k = 1 - \frac{\binom{n-c}{k}}{\binom{n}{k}}
\end{equation*}
\vspace{-5pt}

\noindent where $c$ denotes the number of correct samples among the $n$ generations. In our experiments, we set the total number of samples $n$ to 64, which covers the largest $k$ value evaluated in our benchmarks.

We primarily analyze the pass@$k$ performance for the \textit{DeepSeek-R1-Distill-Qwen-1.5B} and \textit{7B} models trained via GRPO (full-parameter), as well as the 1.5B model trained via LoRA. The results are presented in Figure \ref{fig:appendix_pass_at_k}.

As observed in the figures, although the Parallel-RL model may exhibit a performance drop at Pass@1 compared to the task-specific RL models, the Pass@$k$ curves of the two models tend to \textbf{converge} as the number of samples increases. This convergence indicates that their reasoning boundaries are nearly identical. This phenomenon provides strong evidence for the \textbf{low interference} between RL tasks and the inherent \textbf{stability} of the Parallel-RL paradigm; even simple parameter averaging does not disrupt the underlying reasoning mechanisms learned by individual RL modules.

Furthermore, it is notable that for most tasks, the performance gap between Naive Parallel-RL and Single-Task RL becomes negligible at Pass@16. This observation indirectly validates the feasibility of \textbf{Adapted Parallel-RL}. Since our adaptation phase utilizes $G=16$ rollouts per prompt, the minimal gap at Pass@16 suggests that the model's capability is preserved in the merged parameters. Consequently, a small amount of adaptation data (e.g., 5\%) is sufficient to realign the policy and fully restore performance, as confirmed in our main experiments (Section \ref{sec:parallel_exp}).

\begin{figure*}[t]
\centering
  \includegraphics[width=\linewidth]{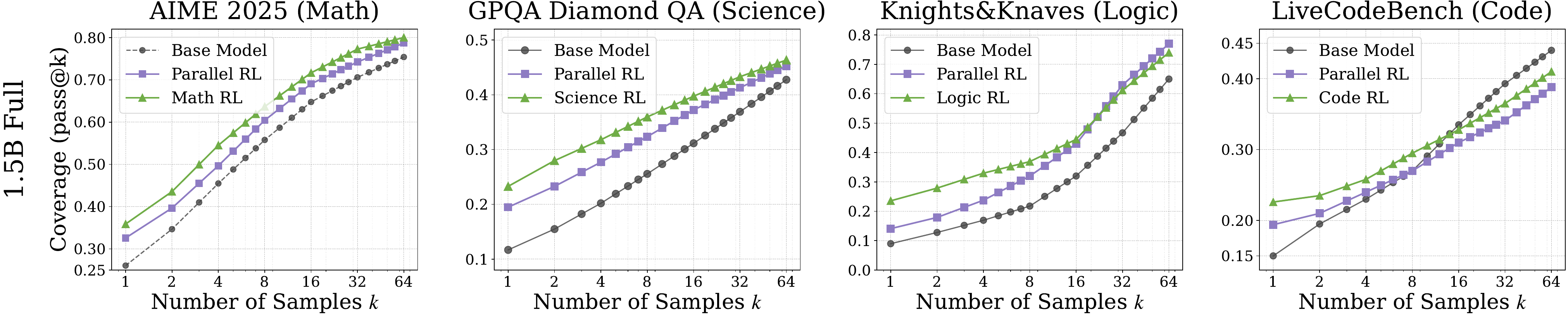}
  
  \vspace{1em} 
  
  \includegraphics[width=\linewidth]{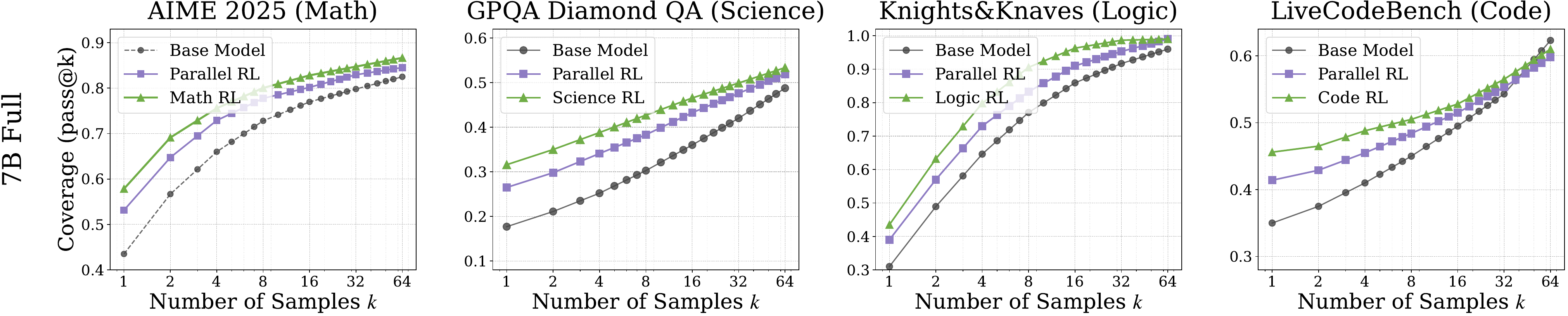}
  

  \vspace{-5pt}
  \caption{Pass@K Analysis of Naive Parallel RL with full parameter GRPO}
  \label{fig:appendix_pass_at_k}
\end{figure*}

\section{Case Study}

In this section, we present several response examples from models trained on different tasks. We also showcase the performance of a multi-task model integrated using a naive parallel RL approach, which involves directly averaging the parameters of individual task models. Specifically, we focus on responses within the math and logic tasks. Our case study includes the base model (DeepSeek-R1-Distill-Qwen-1.5B), the task-specific model (math model and logic model), and the multi-task model (Naive Parallel-RL as described in Section \ref{sec:parallel_exp}).

Figure ~\ref{case:base_model_on_math},\ref{case:logic_model_on_math}, \ref{case:math_model_on_math} and \ref{case:parallel_model_on_math} display the responses of these four models to a test sample from MATH500. Figure \ref{case:base_logic_on_logickk}, \ref{case:math_on_logickk} and \ref{case:parallel_on_logickk} illustrate their performance on a "Knights and Knaves" logic question. As observed, the original base model fails to solve either task. While the math and logic models succeed in their respective domains, they produce errors when applied to untrained tasks. In contrast, the Parallel-RL model demonstrates the capability to provide correct answers for both tasks.

\begin{figure*}
\centering
\begin{tcolorbox}[size=title,opacityfill=0.1]
\noindent
\faUser\ \textbf{Question of MATH500 (Math)}: Solve the following math problem efficiently and clearly.  The last line of your response should be of the following format: 'Therefore, the final answer is: $\boxed{ANSWER}$. I hope it is correct' (without quotes) where ANSWER is just the final number or expression that solves the problem. Think step by step before answering. What is the smallest positive integer $n$ such that all the roots of $z^4 + z^2 + 1 = 0$ are $n^{\text{th}}$ roots of unity? \

\

\faCogs\ \textbf{DeepSeek-R1-Distill-Qwen-1.5B (Base Model)}: \hfill \textcolor[HTML]{D12121}{\faTimesCircle\ \textbf{Incorrect}}

\texttt{<think>}Okay, so I need to find the smallest positive integer $n$ such that all the roots of the equation $z^4 + z^2 + 1 = 0$ are $n$-th roots of unity. Hmm, let me think about how to approach this. First, I remember that the roots of unity are complex numbers that satisfy $z^n = 1$ for some integer $n$. So, each root lies on the unit circle in the complex plane and are equally spaced around it. The equation given is a quartic, which means there are four roots. So, all these roots should lie on the unit circle as well. Let me rewrite the equation: $z^4 + z^2 + 1 = 0$. Hmm, maybe I can factor this equation somehow. Let me see if I can factor it as a quadratic in terms of $z^2$. Let's set $w = z^2$. Then the equation becomes $w^2 + w + 1 = 0$. That's a quadratic equation in $w$. Let me solve for $w$. Using the quadratic formula: $w = [-1 \pm \sqrt{1 - 4}] / 2 = [-1 \pm \sqrt{-3}] / 2$. So, $w$ is $(-1 \pm i\sqrt{3})/2$. These are complex numbers, and they are the primitive 3rd roots of unity because they satisfy $w^3 = 1$. Wait, let me check that. Let me compute $w^3$: Take $w = e^{i2\pi/3}$, which is $(-1 + i\sqrt{3})/2$. Then $w^3 = e^{i2\pi} = 1$. Similarly, the other root is $e^{i4\pi/3}$, so it's also a primitive 3rd root of unity. So, $w = z^2 = e^{i2\pi/3}$ or $e^{i4\pi/3}$. Therefore, to find $z$, I need to take square roots of these. So, $z^2 = e^{i2\pi/3}$ implies $z = e^{i\pi/3}$ or $z = e^{i5\pi/3}$, because when you square $e^{i\pi/3}$, you get $e^{i2\pi/3}$, and squaring $e^{i5\pi/3}$ gives $e^{i10\pi/3}$ which is the same as $e^{i4\pi/3}$ because $10\pi/3$ is $3\pi + \pi/3$, which is equivalent to $\pi/3$ in terms of the unit circle (since $3\pi$ is equivalent to $\pi \pmod{2\pi}$). Wait, no, that's not right. Let me compute $5\pi/3$ squared: $(5\pi/3)^2 = 25\pi^2/9$. That's more than $2\pi$, so I should subtract $2\pi$ to get it within $0$ to $2\pi$. $25\pi/9$ is about $2.777\pi$, so subtract $2\pi$ (which is $18\pi/9$) to get $7\pi/9$. Wait, that can't be right because $5\pi/3$ is 300 degrees, and squaring it would be 600 degrees, which is $600 - 360 = 240$ degrees, which is $4\pi/3$. So, $z = e^{i4\pi/3}$. Similarly, $z^2 = e^{i4\pi/3}$ implies $z = e^{i2\pi/3}$ or $e^{i8\pi/3}$, but $8\pi/3$ is equivalent to $8\pi/3 - 2\pi = 2\pi/3$. So, the roots are $z = e^{i\pi/3}, e^{i5\pi/3}, e^{i4\pi/3}, e^{i2\pi/3}$. Wait, \textbf{[...](omit 3006 tokens)} For example, the 3rd roots of unity have order 3, but perhaps in this case, the equation might have roots that are of higher order, but in this case, they are 3rd roots. So, the minimal $n$ is 3. So, the roots of the original equation are the square roots of $w$, so four roots in total. \textbf{Final Answer}
The smallest positive integer \( n \) is \boxed{3}. \texttt{</think>}
To solve the problem of finding the smallest positive integer \( n \) such that all the roots of the equation \( z^4 + z^2 + 1 = 0 \) are \( n \)-th roots of unity, we start by factoring the given polynomial.
First, we recognize that the equation can be rewritten by substituting \( w = z^2 \), transforming it into a quadratic equation:
\[ w^2 + w + 1 = 0 \]
Solving this quadratic equation using the quadratic formula, we find the roots:
\[ w = \frac{-1 \pm \sqrt{-3}}{2} = \frac{-1 \pm i\sqrt{3}}{2} \]
These roots are the primitive 3rd roots of unity, specifically \( e^{i2\pi/3} \) and \( e^{i4\pi/3} \).
Next, we find the roots of the original equation by solving \( z^2 = e^{i2\pi/3} \) and \( z^2 = e^{i4\pi/3} \). The solutions are:
\[ z = e^{i\pi/3}, e^{i5\pi/3}, e^{i2\pi/3}, e^{i4\pi/3} \]
These roots are all 3rd roots of unity, meaning they satisfy \( z^3 = 1 \). Since none of these roots are 1st or 2nd roots of unity, the smallest positive integer \( n \) such that all roots are \( n \)-th roots of unity is 3.
Thus, the smallest positive integer \( n \) is \textcolor{WrongRed}{\(\boxed{3}\)}. \textbf{(3707 tokens)}

\end{tcolorbox}
\caption{Response from the base model (DeepSeek-R1-Distill-Qwen-1.5B) on a sample problem from MATH500. As shown, the base model fails to solve the problem.}
\label{case:base_model_on_math}
\end{figure*}

\begin{figure*}
\centering
\begin{tcolorbox}[size=title,opacityfill=0.1]
\noindent
\faUser\ \textbf{Question of MATH500 (Math)}: Solve the following math problem efficiently and clearly.  The last line of your response should be of the following format: 'Therefore, the final answer is: $\boxed{ANSWER}$. I hope it is correct' (without quotes) where ANSWER is just the final number or expression that solves the problem. Think step by step before answering. What is the smallest positive integer $n$ such that all the roots of $z^4 + z^2 + 1 = 0$ are $n^{\text{th}}$ roots of unity? \

\

\faSearch\ \textbf{Logic Model}: \hfill \textcolor[HTML]{D12121}{\faTimesCircle\ \textbf{Incorrect}}

\texttt{<think>}Okay, so I've got this problem here: I need to find the smallest positive integer $n$ such that all the roots of the equation $z^4 + z^2 + 1 = 0$ are $n$-th roots of unity. Hmm, I'm a bit rusty on roots of unity, but let me try to recall. First, I remember that the $n$-th roots of unity are the solutions to the equation $z^n = 1$. They are evenly spaced around the unit circle in the complex plane. The general form is $z = e^{2\pi ik/n}$ where $k$ is an integer from $0$ to $n-1$. The equation given is $z^4 + z^2 + 1 = 0$. That's a quartic equation, and I think it can be factored somehow. Maybe it's a quadratic in terms of $z^2$. Let me try substituting $w = z^2$. Then the equation becomes $w^2 + w + 1 = 0$. Oh, right! That's a quadratic equation. Let me solve for $w$. Using the quadratic formula: $w = [-1 \pm \sqrt{1 - 4 \cdot 1 \cdot 1}] / 2$. That gives me $w = [-1 \pm \sqrt{-3}] / 2$. So, the solutions are $w = (-1 + i\sqrt{3})/2$ and $w = (-1 - i\sqrt{3})/2$. These are complex numbers, and I think they can be expressed in polar form. Let me calculate their modulus and argument. The modulus for both is $\sqrt{(-1/2)^2 + (\sqrt{3}/2)^2} = \sqrt{1/4 + 3/4} = \sqrt{1} = 1$. So, both are on the unit circle. What about their arguments? For $w = (-1 + i\sqrt{3})/2$, I think that's a complex number in the second quadrant. The angle is $120^\circ$ because $\cos(120^\circ) = -1/2$ and $\sin(120^\circ) = \sqrt{3}/2$. Similarly, the other solution is $240^\circ$ because $\cos(240^\circ) = -1/2$ and $\sin(240^\circ) = -\sqrt{3}/2$. So, in exponential form, these are $e^{2\pi i \cdot 2/3}$ and $e^{2\pi i \cdot 4/3}$. So $w = e^{2\pi i \cdot 2/3}$ and $w = e^{2\pi i \cdot 4/3}$. But since $w = z^2$, we can write $z^2 = e^{2\pi i \cdot 2/3}$ and $z^2 = e^{2\pi i \cdot 4/3}$. To solve for $z$, we take the square roots of each. So for each $w$, $z$ will be the square roots. Let's compute them. Starting with $w_1 = e^{2\pi i \cdot 2/3}$. The square roots would be $e^{\pi i \cdot 2/3 + \pi ik}$ for $k = 0,1$. So, for $k=0$: $e^{2\pi i \cdot 2/3 \cdot 1/2} = e^{\pi i \cdot 2/3} = e^{i \cdot 120^\circ}$. Similarly, for $k=1$: $e^{\pi i \cdot 2/3 + \pi i} = e^{i \cdot 2\pi/3 + i\pi} = e^{i(2\pi/3 + \pi)} = e^{i \cdot 5\pi/3}$. Similarly, for $w_2 = e^{2\pi i \cdot 4/3}$. The square roots would be $e^{\pi i \cdot 4/3 + \pi ik}$. For $k=0$: $e^{4\pi i/3}$ which is $e^{i \cdot 240^\circ}$. For $k=1$: $e^{4\pi i/3 + \pi i} = e^{7\pi i/3}$ which is equivalent to $e^{i\pi/3}$ because $7\pi/3$ is more than $2\pi$, so subtract $2\pi$ gives $\pi/3$. So, the four roots of the equation $z^4 + z^2 + 1 = 0$ are $z = e^{i\pi/3}$, $e^{i \cdot 2\pi/3}$, $e^{i \cdot 4\pi/3}$, and $e^{i \cdot 5\pi/3}$. Wait, let me make sure. So each of these roots is a 12th root of unity? Because $e^{i\pi/3}$ is the same as $e^{i \cdot 4\pi/3 \cdot 2}$, right? Hmm, no, wait. Let me think. Each of these roots is of the form $e^{i \cdot 2\pi k/12}$ because $2\pi/12 = \pi/6$. So, $e^{i \cdot 2\pi k/12} = e^{i \cdot \pi k/6}$. Let me check: for $k=1$, it's $\pi/6$; $k=2$, $2\pi/6=\pi/3$; $k=3$, $\pi/2$; $k=4$, $2\pi/3$; $k=5$, $5\pi/6$; $k=6$, $\pi$; $k=7$, $7\pi/6$; $k=8$, $4\pi/3$; etc. Looking at our roots, they are $e^{i\pi/3}$ ($k=2$), $e^{i2\pi/3}$ ($k=4$), $e^{i4\pi/3}$ ($k=8$), and $e^{i5\pi/3}$ ($k=10$). \textbf{[...](omit 2458 tokens)} Therefore, the smallest positive integer $n$ is 12. \texttt{</think>} 
To solve the problem, we start by recognizing that the roots of the equation \( z^4 + z^2 + 1 = 0 \) can be found by substituting \( w = z^2 \), transforming the quartic into a quadratic equation \( w^2 + w + 1 = 0 \). Solving this quadratic using the quadratic formula gives the roots \( w = e^{2\pi i \cdot 2/3} \) and \( w = e^{2\pi i \cdot 4/3} \). 
Since \( w = z^2 \), we solve for \( z \) by taking the square roots of each \( w \). This yields four roots: \( z = e^{\pm \pi i / 3} \) and \( z = e^{\pm 5\pi i / 3} \). These roots are \( e^{\pm \pi i / 3} \), \( e^{\pm 5\pi i / 3} \), which correspond to the 12th roots of unity.
To confirm, we factor the quartic as \( (z^2 + z + 1)(z^2 - z + 1) \). The roots of \( z^2 + z + 1 \) are the primitive 3rd roots of unity, and the roots of \( z^2 - z + 1 \) are the primitive 12th roots of unity. Thus, all roots of the original equation are indeed 12th roots of unity.
Since the roots are not all 6th roots of unity (as \( e^{\pi i / 6} \) is not a root of the quartic), the smallest positive integer \( n \) for which all roots are \( n \)-th roots of unity is \textcolor{WrongRed}{\(\boxed{12}\)}. \textbf{(3178 tokens)}

\end{tcolorbox}
\caption{Response from the logic model on a sample problem from MATH500. The logic model is derived from the base model via GRPO training on a logic dataset.}
\label{case:logic_model_on_math}
\end{figure*}

\begin{figure*}
\centering
\begin{tcolorbox}[size=title,opacityfill=0.1]
\noindent
\faUser\ \textbf{Question of MATH500 (Math)}: Solve the following math problem efficiently and clearly.  The last line of your response should be of the following format: 'Therefore, the final answer is: $\boxed{ANSWER}$. I hope it is correct' (without quotes) where ANSWER is just the final number or expression that solves the problem. Think step by step before answering. What is the smallest positive integer $n$ such that all the roots of $z^4 + z^2 + 1 = 0$ are $n^{\text{th}}$ roots of unity? \

\

\faCalculator\ \textbf{Math Model}: \hfill \textcolor[HTML]{2EA121}{\faCheckCircle\ \textbf{Correct}}

\texttt{<think>}Okay, so I have this problem: I need to find the smallest positive integer $n$ such that all the roots of the equation $z^4 + z^2 + 1 = 0$ are $n$-th roots of unity. Hmm, let me try to figure this out step by step. First, I remember that roots of unity are complex numbers that satisfy $z^n = 1$ for some integer $n$. So, each root is $e^{2\pi ik/n}$ where $k$ is an integer from $0$ to $n-1$. These roots are equally spaced around the unit circle in the complex plane. Now, the equation given is $z^4 + z^2 + 1 = 0$. I wonder if I can factor this or simplify it somehow. Let me see. Maybe substituting $w = z^2$ would make it easier. Let's try that. If I let $w = z^2$, then the equation becomes $w^2 + w + 1 = 0$. That's a quadratic equation in terms of $w$. I can solve this using the quadratic formula. The solutions for $w$ would be: $w = [-1 \pm \sqrt{1 - 4 \cdot 1 \cdot 1}] / 2 = [-1 \pm \sqrt{-3}] / 2 = (-1 \pm i\sqrt{3})/2$. So, the solutions for $w$ are $(-1 + i\sqrt{3})/2$ and $(-1 - i\sqrt{3})/2$. These are complex numbers. I recognize them as the complex cube roots of unity, except that the cube roots of unity are $(1, \omega, \omega^2)$ where $\omega = e^{2\pi i/3}$. Wait, but here the solutions are actually the primitive cube roots of unity, because $(-1 \pm i\sqrt{3})/2$ are the non-real cube roots. So, they satisfy $w^3 = 1$, but aren't equal to $1$ themselves. Wait, actually, let me check. Let me compute $(-1 + i\sqrt{3})/2$. Let me denote this as $\omega$. Let's compute $\omega^3$: $\omega = e^{2\pi i/3} = \cos(2\pi/3) + i \sin(2\pi/3) = -1/2 + i\sqrt{3}/2$. So, yes, $\omega$ is a primitive cube root of unity, meaning $\omega^3 = 1$. So, $w = \omega$ and $w = \omega^2$, where $\omega^2$ is the other cube root. So, $w^2 + w + 1 = 0$, so the solutions are the primitive cube roots of unity. Therefore, $w = e^{2\pi i/3}$ and $w = e^{4\pi i/3}$. But since $w = z^2$, that means $z^2 = e^{2\pi i/3}$ or $z^2 = e^{4\pi i/3}$. So, to find $z$, we take the square roots of these. Let me solve $z^2 = e^{2\pi i/3}$. The solutions are $z = e^{\pi i/3}$ and $z = e^{\pi i/3 + \pi i} = e^{4\pi i/3}$. Similarly, for $z^2 = e^{4\pi i/3}$, the solutions are $z = e^{2\pi i/3}$ and $z = e^{2\pi i/3 + \pi i} = e^{5\pi i/3}$. So, the roots of the equation $z^4 + z^2 + 1 = 0$ are $z = e^{\pi i/3}, e^{4\pi i/3}, e^{2\pi i/3}$, and $e^{5\pi i/3}$. Let me list them: 1. $e^{\pi i/3} = \cos(\pi/3) + i \sin(\pi/3) = 1/2 + i\sqrt{3}/2$. 2. $e^{4\pi i/3} = \cos(4\pi/3) + i \sin(4\pi/3) = -1/2 - i\sqrt{3}/2$. 3. $e^{2\pi i/3} = \cos(2\pi/3) + i \sin(2\pi/3) = -1/2 + i\sqrt{3}/2$. 4. $e^{5\pi i/3} = \cos(5\pi/3) + i \sin(5\pi/3) = 1/2 - i\sqrt{3}/2$. So, these are four distinct 12th roots of unity? Wait, no, wait. Let's see. \textbf{[...](omit 9841 tokens)} They are 6th roots, so the minimal $n$ is 6. So, the final answer is 6. \texttt{</think>} 
\textbf{Final Answer}
The smallest positive integer \( n \) is \boxed{6}.
To solve the problem of finding the smallest positive integer \( n \) such that all the roots of \( z^4 + z^2 + 1 = 0 \) are \( n \)-th roots of unity, we start by solving the given equation.
First, we substitute \( w = z^2 \) to transform the equation into a quadratic equation:
\[ w^2 + w + 1 = 0 \]
Solving this quadratic equation using the quadratic formula, we get:
\[ w = \frac{-1 \pm \sqrt{-3}}{2} = \frac{-1 \pm i\sqrt{3}}{2} \]
These solutions are the primitive cube roots of unity, denoted as \( \omega \) and \( \omega^2 \), where \( \omega = e^{2\pi i / 3} \) and \( \omega^2 = e^{4\pi i / 3} \).
Next, we solve for \( z \) by taking the square roots of \( \omega \) and \( \omega^2 \). The square roots of \( \omega \) are \( e^{\pi i / 3} \) and \( e^{4\pi i / 3} \), while the square roots of \( \omega^2 \) are \( e^{2\pi i / 3} \) and \( e^{5\pi i / 3} \).
Thus, the roots of the equation \( z^4 + z^2 + 1 = 0 \) are:
\[ z = e^{\pi i / 3}, e^{2\pi i / 3}, e^{4\pi i / 3}, e^{5\pi i / 3} \]
These roots are all 6th roots of unity because each root can be expressed as \( e^{k\pi i / 3} \) where \( k \) is an integer, and the smallest \( n \) such that all these roots are \( n \)-th roots of unity is 6.
Therefore, the smallest positive integer \( n \) such that all the roots of \( z^4 + z^2 + 1 = 0 \) are \( n \)-th roots of unity is:
\textcolor{CorrectGreen}{$ \boxed{6} $} \textbf{(10667 tokens)}

\end{tcolorbox}
\caption{Response from the math model on a sample problem from MATH500. The math model is derived from the base model via GRPO training on a math dataset.}
\label{case:math_model_on_math}
\end{figure*}

\begin{figure*}
\centering
\begin{tcolorbox}[size=title,opacityfill=0.1]
\noindent
\faUser\ \textbf{Question of MATH500 (Math)}: Solve the following math problem efficiently and clearly.  The last line of your response should be of the following format: 'Therefore, the final answer is: $\boxed{ANSWER}$. I hope it is correct' (without quotes) where ANSWER is just the final number or expression that solves the problem. Think step by step before answering. What is the smallest positive integer $n$ such that all the roots of $z^4 + z^2 + 1 = 0$ are $n^{\text{th}}$ roots of unity? \

\

\faRobot\ \textbf{Parallel-RL Model}: \hfill \textcolor[HTML]{2EA121}{\faCheckCircle\ \textbf{Correct}}

\texttt{<think>}Okay, so I've got this problem here: I need to find the smallest positive integer $n$ such that all the roots of the equation $z^4 + z^2 + 1 = 0$ are $n$-th roots of unity. Hmm, okay. Let's see. First, I remember that roots of unity are complex numbers that satisfy the equation $z^n = 1$. So, they lie on the unit circle in the complex plane and are equally spaced around the circle. The roots of unity are also the solutions to $z^n = 1$. The given equation is $z^4 + z^2 + 1 = 0$. I need to figure out the roots of this equation and then determine the smallest $n$ for which all these roots are $n$-th roots of unity. Let me try to factor the equation or find its roots. Maybe I can factor it as a quadratic in terms of $z^2$. Let me set $w = z^2$, so the equation becomes $w^2 + w + 1 = 0$. That's a quadratic equation, which I can solve using the quadratic formula. The quadratic formula says that for $ax^2 + bx + c = 0$, the roots are $(-b \pm \sqrt{b^2 - 4ac})/(2a)$. Applying that here, $a = 1$, $b = 1$, $c = 1$. So the roots are $(-1 \pm \sqrt{1 - 4})/2 = (-1 \pm \sqrt{-3})/2$. So, $w = (-1 \pm i\sqrt{3})/2$. Therefore, $z^2 = (-1 \pm i\sqrt{3})/2$. So, to find $z$, I need to take the square roots of these two complex numbers. Wait, but I remember that $(-1 + i\sqrt{3})/2$ is actually a complex number on the unit circle. Let me check its modulus. The modulus of a complex number $a + ib$ is $\sqrt{a^2 + b^2}$. So here, $a = -1/2$ and $b = (\sqrt{3})/2$. So modulus squared is $(1/4) + (3/4) = 1$, so modulus is 1. So, it's indeed on the unit circle. Similarly, $(-1 - i\sqrt{3})/2$ also has modulus 1. So, these are both complex numbers with modulus 1, which suggests they are points on the unit circle at angles of 120 degrees and 240 degrees, or $2\pi/3$ and $4\pi/3$ radians, respectively. So, if I can express these roots in exponential form, that might help. Since they're on the unit circle, they can be written as $e^{i\theta}$, where $\theta$ is the angle. So, for $(-1 + i\sqrt{3})/2$, $\theta$ is $2\pi/3$, and for $(-1 - i\sqrt{3})/2$, $\theta$ is $4\pi/3$. Therefore, the solutions for $z^2$ are $e^{i2\pi/3}$ and $e^{i4\pi/3}$. So, taking square roots, the solutions for $z$ would be the square roots of these. Wait, but when you take the square root of $e^{i\theta}$, you get $e^{i\theta/2}$ and $e^{i(\theta/2 + \pi)}$. So, let's compute that. For the first root, $z^2 = e^{i2\pi/3}$, so $z = e^{i\pi/3}$ and $e^{i(\pi/3 + \pi)} = e^{i4\pi/3}$. Similarly, for the second root, $z^2 = e^{i4\pi/3}$, so $z = e^{i2\pi/3}$ and $e^{i(2\pi/3 + \pi)} = e^{i5\pi/3}$. Wait, \textbf{[...](omit 5291 tokens)} Therefore, the smallest positive integer $n$ is 6. \textbf{Final Answer} The smallest positive integer \( n \) is \boxed{6}.\texttt{</think>}
To find the smallest positive integer \( n \) such that all the roots of \( z^4 + z^2 + 1 = 0 \) are \( n \)-th roots of unity, we start by factoring the given polynomial.
First, we substitute \( w = z^2 \) to transform the equation into a quadratic:
\[ w^2 + w + 1 = 0 \]
Solving this quadratic equation using the quadratic formula, we get:
\[ w = \frac{-1 \pm \sqrt{-3}}{2} = \frac{-1 \pm i\sqrt{3}}{2} \]
These roots are complex numbers on the unit circle, specifically:
$ w = e^{i2\pi/3} \quad \text{and} \quad w = e^{i4\pi/3} $
Next, we find the roots of the original equation by solving for \( z \):

1. For \( z^2 = e^{i2\pi/3} \), the roots are:
   $ z = e^{i\pi/3} \quad \text{and} \quad z = e^{i4\pi/3} $
   
2. For \( z^2 = e^{i4\pi/3} \), the roots are:
   $ z = e^{i2\pi/3} \quad \text{and} \quad z = e^{i5\pi/3} $
   
These roots are the 6th roots of unity, as they are at angles \( \pi/3, 2\pi/3, 4\pi/3, \) and \( 5\pi/3 \). These angles are all multiples of \( \pi/3 \), which are angles of the 6th roots of unity.
The 6th roots of unity are \( e^{i2\pi k/6} \) for \( k = 0, 1, 2, 3, 4, 5 \). The roots of the given polynomial are a subset of these, specifically excluding 1 and -1. Therefore, all roots are 6th roots of unity.
Thus, the smallest positive integer \( n \) such that all the roots are \( n \)-th roots of unity is:
\textcolor{CorrectGreen}{$ \boxed{6} $} \textbf{(6080 tokens)}

\end{tcolorbox}
\caption{Response from the Parallel-RL model on a sample problem from MATH500. This model uses \textbf{Naive Parallel-RL} (shown in Section \ref{sec:parallel_exp}), which is obtained by directly averaging the models trained on \textbf{four tasks}.}
\label{case:parallel_model_on_math}
\end{figure*}

\begin{figure*}
\centering
\begin{tcolorbox}[size=title,opacityfill=0.1]
\noindent
\faUser\ \textbf{Question of KK (Logic)}: A very special island is inhabited only by knights and knaves. Knights always tell the truth, and knaves always lie. You meet 2 inhabitants: Owen, and Penelope. Owen asserted: ``Penelope is a knight if and only if Penelope is a knave''. Penelope noted, ``Owen is not a knave''. So who is a knight and who is a knave? \\

\faCogs\ \textbf{DeepSeek-R1-Distill-Qwen-1.5B (Base Model)}: \hfill \textcolor[HTML]{D12121}{\faTimesCircle\ \textbf{Incorrect}}

\texttt{<think>}Okay, so I have this logic puzzle about knights and knaves. Knights always tell the truth, and knaves always lie. I need to figure out who is who between Owen and Penelope based on their statements.First, let's break down Owen's statement: ``Penelope is a knight if and only if Penelope is a knave." The ``if and only if" part is a bit tricky because it means both parts of the statement need to be true for the whole statement to be true. So, either Penelope is a knight and also a knave, or she's neither. But wait, knights are always truthful, and knaves always lie. So Penelope can't be both a knight and a knave at the same time. That makes Penelope's statement a bit confusing.Now, Penelope's statement is: ``Owen is not a knave." Since knights tell the truth and knaves lie, if Penelope is a knight, her statement must be true, meaning Owen is not a knave, so Owen must be a knight too. But if Penelope is a knave, her statement is false, which would mean that Owen is a knave. However, if Owen is a knave, then Penelope's statement is a lie, which aligns with Penelope being a knave. Wait, but if Owen is a knight, then his statement must be true. So if Penelope is also a knight, her statement ``Owen is not a knave" would be true because Owen is indeed a knight. That seems consistent. Alternatively, if Penelope is a knave, her statement would be false, meaning Owen is a knave. But if Owen is a knave, then his statement ``Penelope is a knight if and only if Penelope is a knave" would be false because Penelope is a knave and a knight at the same time, which isn't possible. So that would lead to a contradiction.Therefore, the only consistent scenario is that both Owen and Penelope are knights. That way, Penelope's statement is true, and Owen's statement holds because he's a knight.\texttt{</think>} Penelope and Owen are both knights. This is because Penelope's statement that Owen is not a knave is true, and Owen's statement that Penelope is a knight if and only if Penelope is a knave is also true since both are knights. {\textcolor{WrongRed}{\texttt{<answer>} (1) Owen is a knight.  (2) Penelope is a knight. \texttt{</answer>}}} \textbf{(523 tokens)}

\

\faSearch\ \textbf{Logic Model}: \hfill \textcolor[HTML]{2EA121}{\faCheckCircle\ \textbf{Correct}}

\texttt{<think>} Okay, so I'm trying to solve this logic puzzle where there are two inhabitants on an island: Owen and Penelope. They're either knights, who always tell the truth, or knaves, who always lie. I need to figure out who is who based on what they said. First, let's break down what each statement means. Owen said, ``Penelope is a knight if and only if Penelope is a knave." That's a bit of a mouthful, but let's parse it. The ``if and only if" means that both parts of the statement are connected by an equivalence—so if one is true, the other must be true, and if one is false, the other must be false. So, if Penelope is a knight, she tells the truth, so her statement must be true. That would mean Penelope is both a knight and a knave, which doesn't make sense because knights and knaves are distinct categories. Wait, that's a problem. So if Penelope is a knight, her statement that ``Penelope is a knight if and only if Penelope is a knave" would have to be true. But that would require Penelope to be both a knight and a knave, which is impossible. Therefore, Penelope cannot be a knight. That means Penelope must be a knave, and since knaves lie, her statement is false. Now, let's look at Penelope's statement: ``Owen is not a knave." Since Penelope is a knave, her statement is false. So the opposite of ``Owen is not a knave" must be true. The opposite of ``not a knave" is ``knave," so Penelope's statement being false means ``Owen is a knave." Therefore, Owen is a knave. To recap: - If Penelope is a knave, her statement ``Owen is not a knave" is false, meaning Owen is a knave. - Since Owen is a knave, his statement ``Penelope is a knight if and only if Penelope is a knave" is false, which fits because knaves always lie. So, both Penelope and Owen are knaves. \texttt{</think>} Penelope is a knave, and Owen is also a knave. {\textcolor{CorrectGreen}{\texttt{<answer>} (1)Penelope is a knave.  (2) Owen is a knave. \texttt{</answer>}}} \textbf{(505 tokens)}

\end{tcolorbox}
\caption{Comparison of responses from the base model (DeepSeek-R1-Distill-Qwen-1.5B) and the logic model on a sample problem from Knights and Knaves. The logic model is derived from the base model via GRPO training on a logic dataset. As shown, the base model fails to solve the problem, while the logic model provides the correct answer.}
\label{case:base_logic_on_logickk}
\end{figure*}

\begin{figure*}
\centering
\begin{tcolorbox}[size=title,opacityfill=0.1]
\noindent
\faUser\ \textbf{Question of KK (Logic)}: A very special island is inhabited only by knights and knaves. Knights always tell the truth, and knaves always lie. You meet 2 inhabitants: Owen, and Penelope. Owen asserted: ``Penelope is a knight if and only if Penelope is a knave''. Penelope noted, ``Owen is not a knave''. So who is a knight and who is a knave? \\

\faCalculator\ \textbf{Math Model}: \hfill \textcolor[HTML]{D12121}{\faTimesCircle\ \textbf{Incorrect}}

\texttt{<think>}Okay, so I've got this logic puzzle about knights and knaves. Knights always tell the truth, and knaves always lie. There are two people involved: Owen and Penelope. Owen says, ``Penelope is a knight if and only if Penelope is a knave.'' Penelope responds, ``Owen is not a knave.'' I need to figure out who is a knight and who is a knave. First, I'll try to understand what each statement means. The first statement by Owen uses a biconditional (``if and only if''), which means both parts must be true or both must be false. So, if Penelope is a knight, then Penelope is a knight, and Penelope is a knave. Wait, that can't be right because Penelope can't be both a knight and a knave at the same time. That seems contradictory. On the other hand, if Penelope is a knave, then Penelope is not a knight, and Penelope is a knight. Again, that's a contradiction because a knave can't be both not a knight and a knight simultaneously. Hmm, this is confusing. Maybe I made a mistake in interpreting the statements. Let me break it down step by step. Let's consider the two possibilities: Penelope is a knight or Penelope is a knave. 1. **Case 1: Penelope is a knight.**   - Since Penelope is a knight, she tells the truth. Her statement, ``Owen is not a knave,'' means Owen is a knight. But wait, if Owen is a knight, then his statement must be true. His statement is a biconditional: Penelope is a knight if and only if Penelope is a knave. If Penelope is a knight, then the statement would be ``Knight if and only if Knave,'' which is a contradiction because a knight cannot be a knave. Therefore, this case leads to a contradiction. 2. \textbf{[...](omit 257 tokens)} So, there's a problem here. I think I messed up the logical analysis. Let me try again. If Penelope is a knave, her statement is false. Her statement is ``Owen is not a knave.'' The negation of ``Owen is not a knave'' is ``Owen is a knave.'' So, Owen is a knave. Now, Owen's statement is ``Penelope is a knight if and only if Penelope is a knave.'' Both are knaves, so the biconditional is true. But since Owen is a knave, his statement must be false. That can't happen because a true statement can't be a lie. Therefore, the only consistent scenario is that Penelope is a knight and Owen is a knave. Wait, but earlier that led to a contradiction. Let me think again. If Penelope is a knight, her statement is true. Her statement is ``Owen is a knight if and only if Penelope is a knave.'' Since Penelope is a knight, the statement is ``Owen is a knight if and only if Penelope is a knave,'' which is ``Owen is a knight if and only if not a knight.'' That's a contradiction because a knight can't be a knave. So Penelope can't be a knight. Therefore, Penelope must be a knave, and Owen must be a knight. But then Owen, being a knight, tells the truth about Penelope's status. His statement is ``Penelope is a knight if and only if Penelope is a knave,'' which is a biconditional. Since Penelope is a knave, the statement is false because Penelope is a knave but not a knight. But as a knight, Owen must tell the truth, so his statement must be true. This is a contradiction. This is confusing. Maybe there's no solution, but logically, Penelope must be a knave because her statement leads to a contradiction if she were a knight. So the only consistent solution is that Penelope is a knave, and Owen is also a knave, but that leads to a contradiction in Owen's statement. Therefore, the only solution is Penelope is a knave, and Owen is a knight, but that contradicts Owen's truthful statement.  I think I need to conclude that Penelope is a knave, and Owen is a knight, but that leads to a contradiction. Alternatively, Penelope is a knight, but that leads to a contradiction. Therefore, there must be a mistake in my reasoning. Wait, perhaps if Penelope is a knight, her statement is true, which would mean both Penelope and Owen are knaves, but that contradicts. If Penelope is a knave, her statement is false, meaning Penelope is a knight, which is also a contradiction. Therefore, there must be an error in the problem setup, but logically, Penelope must be a knave, and Owen is a knight, but that contradicts. Alternatively, Penelope is a knight, and Owen is a knave, but that also contradicts. This is perplexing.\texttt{</think>} {\textcolor{WrongRed}{\texttt{<answer>} Penelope is a knave, and Owen is a knight. \texttt{</answer>}}} \textbf{(1399 tokens)}

\end{tcolorbox}
\caption{Response from the math model on a sample problem from Knights and Knaves. The math model is derived from the base model via GRPO training on a math dataset.}
\label{case:math_on_logickk}
\end{figure*}

\begin{figure*}
\centering
\begin{tcolorbox}[size=title,opacityfill=0.1]
 
\faUser\ \textbf{Question of KK (Logic)}: A very special island is inhabited only by knights and knaves. Knights always tell the truth, and knaves always lie. You meet 2 inhabitants: Owen, and Penelope. Owen asserted: ``Penelope is a knight if and only if Penelope is a knave''. Penelope noted, ``Owen is not a knave''. So who is a knight and who is a knave? \\

\faRobot\ \textbf{Parallel-RL Model}: \hfill \textcolor[HTML]{2EA121}{\faCheckCircle\ \textbf{Correct}}

\texttt{<think>}Okay, so I've got this logic puzzle involving knights and knaves. Knights always tell the truth, and knaves always lie. I need to figure out who is who between Owen and Penelope. Let's break it down step by step. First, Owen made a statement: ``Penelope is a knight if and only if Penelope is a knave.'' That's a bit of a mouthful. The ``if and only if'' part means that both parts of the statement are connected by a biconditional, so they must be both true or both false for the statement to hold. Penelope then said, ``Owen is not a knave.'' So Owen is claiming that Owen isn't a knave. Since knaves always lie, if Penelope is a knave, her statement would be false, meaning Owen is a knave. But wait, if Penelope is a knave, that would mean Owen is a knave, which would make Penelope's statement false, which aligns with knaves always lying. Hmm, maybe I should approach this more systematically. Let's consider the two possibilities: Penelope is a knight or a knave. 1. \textbf{Case 1: Penelope is a knight.}    - Since knights tell the truth, Penelope's statement ``Owen is not a knave'' must be true. So, Owen is not a knave. That means Owen is a knight.    - Now, let's look at Owen's statement. He said ``Penelope is a knight if and only if Penelope is a knave.'' If both parts are true, that would mean Penelope is both a knight and a knave, which isn't possible because each inhabitant is either a knight or a knave, not both.    - Alternatively, if both parts are false, then Penelope is neither a knight nor a knave, which also contradicts the fact that each inhabitant is either one or the other.    - Therefore, Penelope being a knight leads to a contradiction because Owen's statement can't be satisfied without a contradiction. So this case can't be true. 2. \textbf{Case 2: Penelope is a knave.}   - As knaves always lie, Penelope's statement ``Owen is not a knave'' is false. So, the opposite is true: Owen is a knave.   - Now, let's check Owen's statement again. He said ``Penelope is a knight if and only if Penelope is a knave.'' If both parts are true, Penelope would have to be both a knight and a knave, which isn't allowed. But since Penelope is a knave, her statement is false, so the biconditional is false. That means either Penelope is a knight and a knave (which is impossible), or Penelope is neither a knight nor a knave, which also isn't possible. Wait, that doesn't seem right. Let me clarify: For a biconditional to be false, at least one part must be false. So, either Penelope is not a knight or Penelope is not a knave. Since we've established Penelope is a knave, the second part is false. Therefore, the biconditional is false, which is consistent with Owen being a knave, who always lies. So this case holds without contradiction. So, putting it all together: - Penelope is a knave.- Owen is also a knave because Penelope's lie makes him one. \texttt{</think>} {\textcolor{CorrectGreen}{\texttt{<answer>} (1) Penelope is a knave.  (2) Owen is a knave. \texttt{</answer>}}} \textbf{(756 tokens)}

\end{tcolorbox}
\caption{Response from the Parallel-RL model on a sample problem from Knights and Knaves. This model uses \textbf{Naive Parallel-RL} (shown in Section \ref{sec:parallel_exp}), which is obtained by directly averaging the models trained on \textbf{four tasks}.}
\label{case:parallel_on_logickk}
\end{figure*}

\clearpage

\section{Prompt Templates}
\label{appendix:prompt}

In this section, we provide the specific prompt templates used for both training and evaluation across the four representative tasks discussed in the main text.

\textbf{Training Prompts.}
The prompt templates used for training the Math and Science tasks are presented in Figure~\ref{fig:math_sci_train_prompt}. For the Logic task, we follow the configuration established in Logic-RL, with the specific prompt shown in Figure~\ref{fig:logic_prompt}. Finally, for the Code task, we adopt the prompt template displayed in Figure~\ref{fig:code_prompt}, which aligns with the settings used in DeepCoder.

\textbf{Evaluation Prompts.}
For evaluation, we utilize the Lighteval framework and largely adhere to the prompt templates provided in its official repository. For Math benchmarks, Figure~\ref{fig:aime_prompt} and Figure~\ref{fig:math500_prompt} illustrate the templates used for evaluating AIME and MATH500, respectively, following the Lighteval standard. For Science benchmarks, the prompt templates for the MMLU and GPQA are shown in Figure~\ref{fig:mmlu_prompt} and Figure~\ref{fig:gpqa_prompt}, also consistent with the Lighteval repository. For Logic and Code tasks, the evaluation prompts are the same as those used during training.

\begin{figure}[h]
  \centering
  \begin{subfigure}[b]{0.8\textwidth}
    \includegraphics[width=\textwidth]{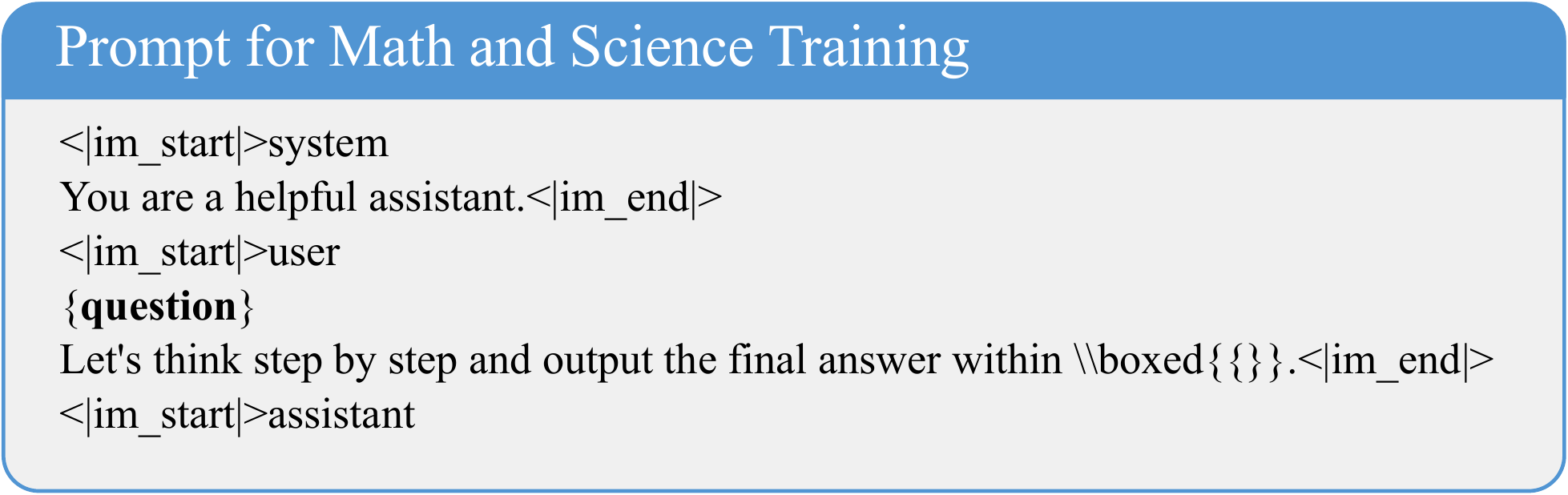}
  \end{subfigure}
\caption{
Prompt for \textbf{Math and Science Task Training}. We follow the prompt template settings in deepscaler.
}
\label{fig:math_sci_train_prompt}
\end{figure}

\begin{figure}[h]
  \vspace{-5pt}
  \centering
  \begin{subfigure}[b]{0.8\textwidth}
    \includegraphics[width=\textwidth]{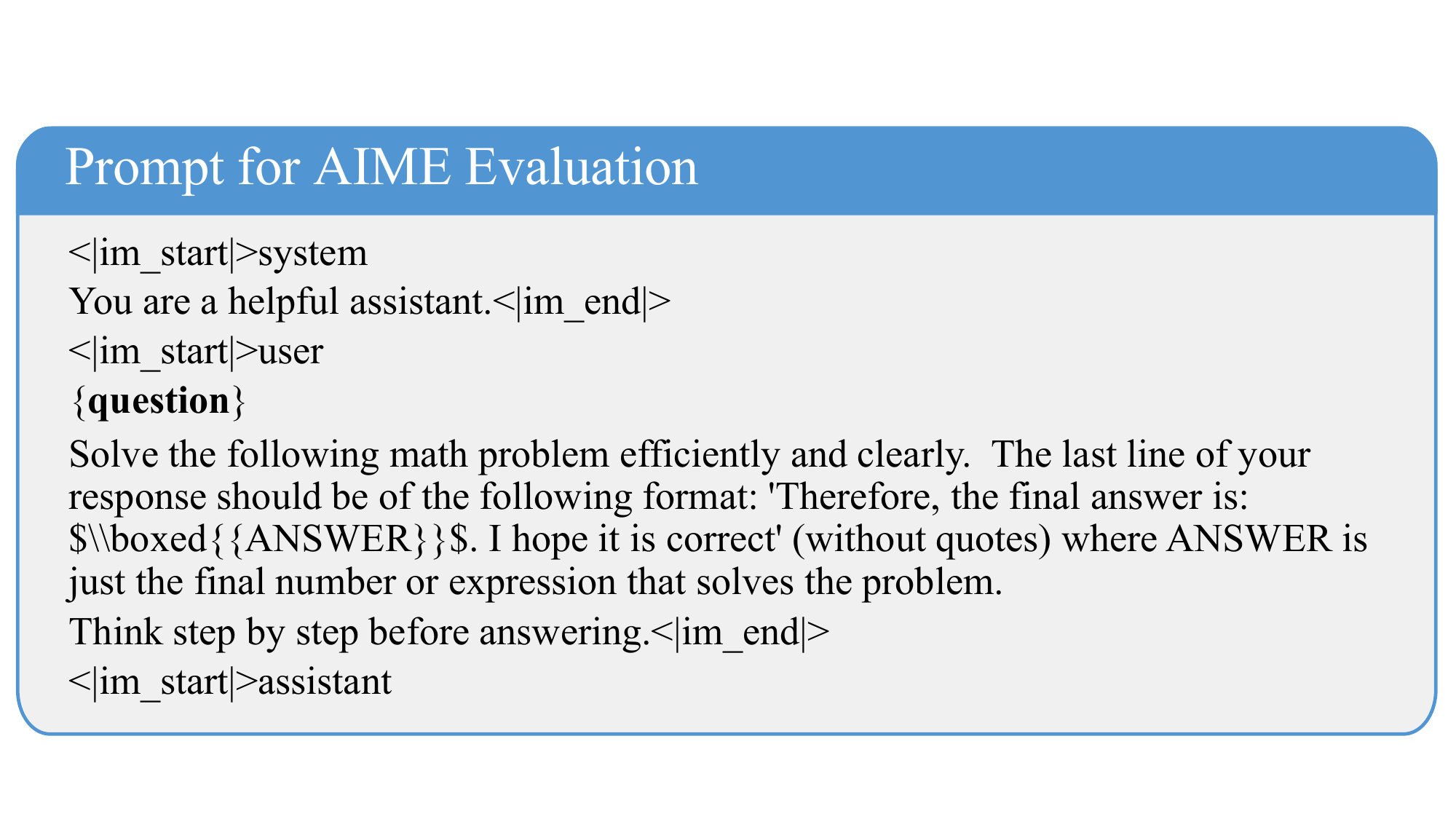}
  \end{subfigure}
\caption{
Prompt for \textbf{AIME Evaluation}. We follow the prompt template settings in Lighteval.
}
\label{fig:aime_prompt}
\end{figure}

\begin{figure}[h]
  \vspace{-5pt}
  \centering
  \begin{subfigure}[b]{0.8\textwidth}
    \includegraphics[width=\textwidth]{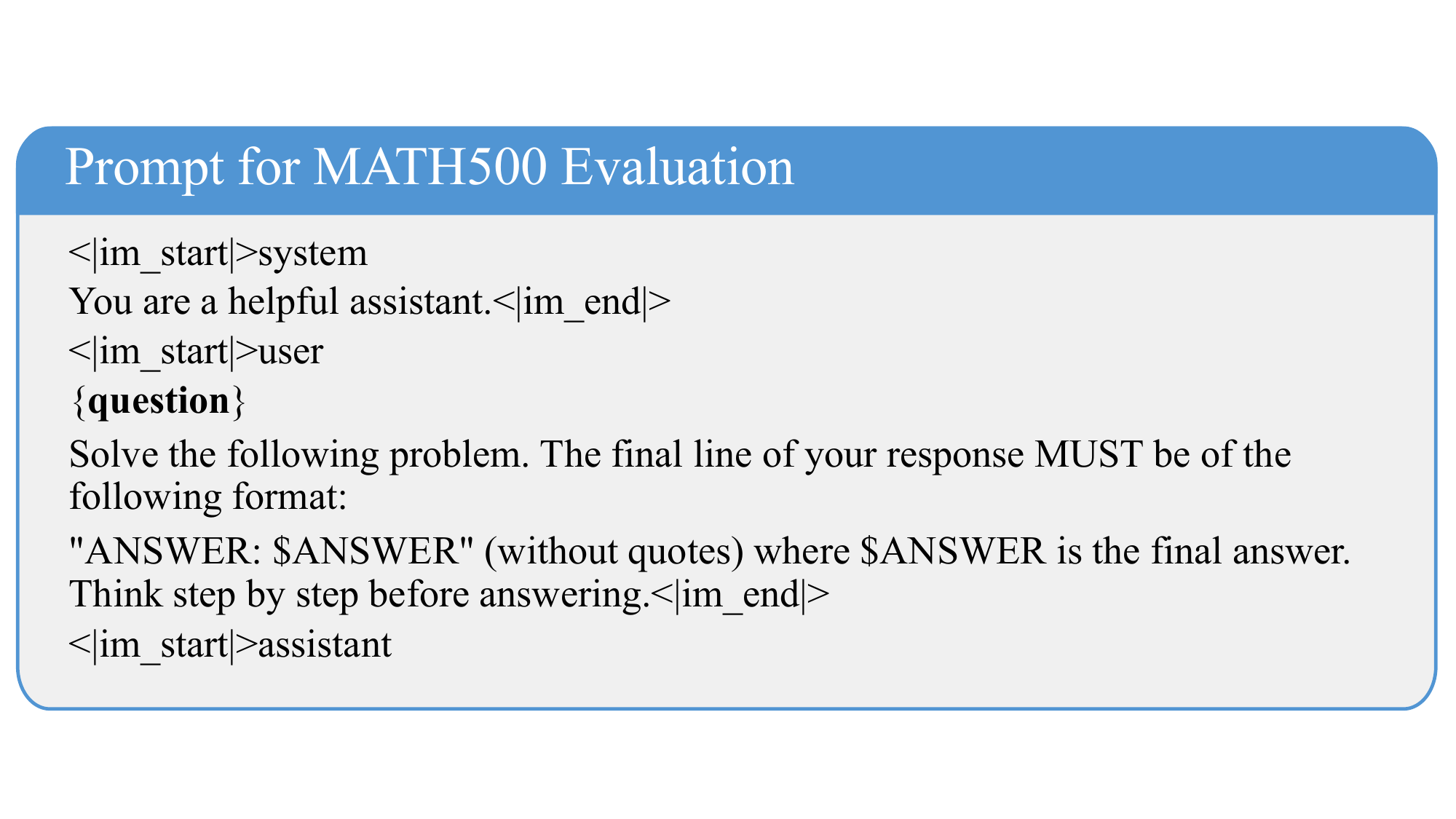}
  \end{subfigure}
\caption{
Prompt for \textbf{MATH500 Evaluation}. We follow the prompt template settings in Lighteval.
}
\label{fig:math500_prompt}
\end{figure}

\begin{figure}[h]
  \vspace{-5pt}
  \centering
  \begin{subfigure}[b]{0.8\textwidth}
    \includegraphics[width=\textwidth]{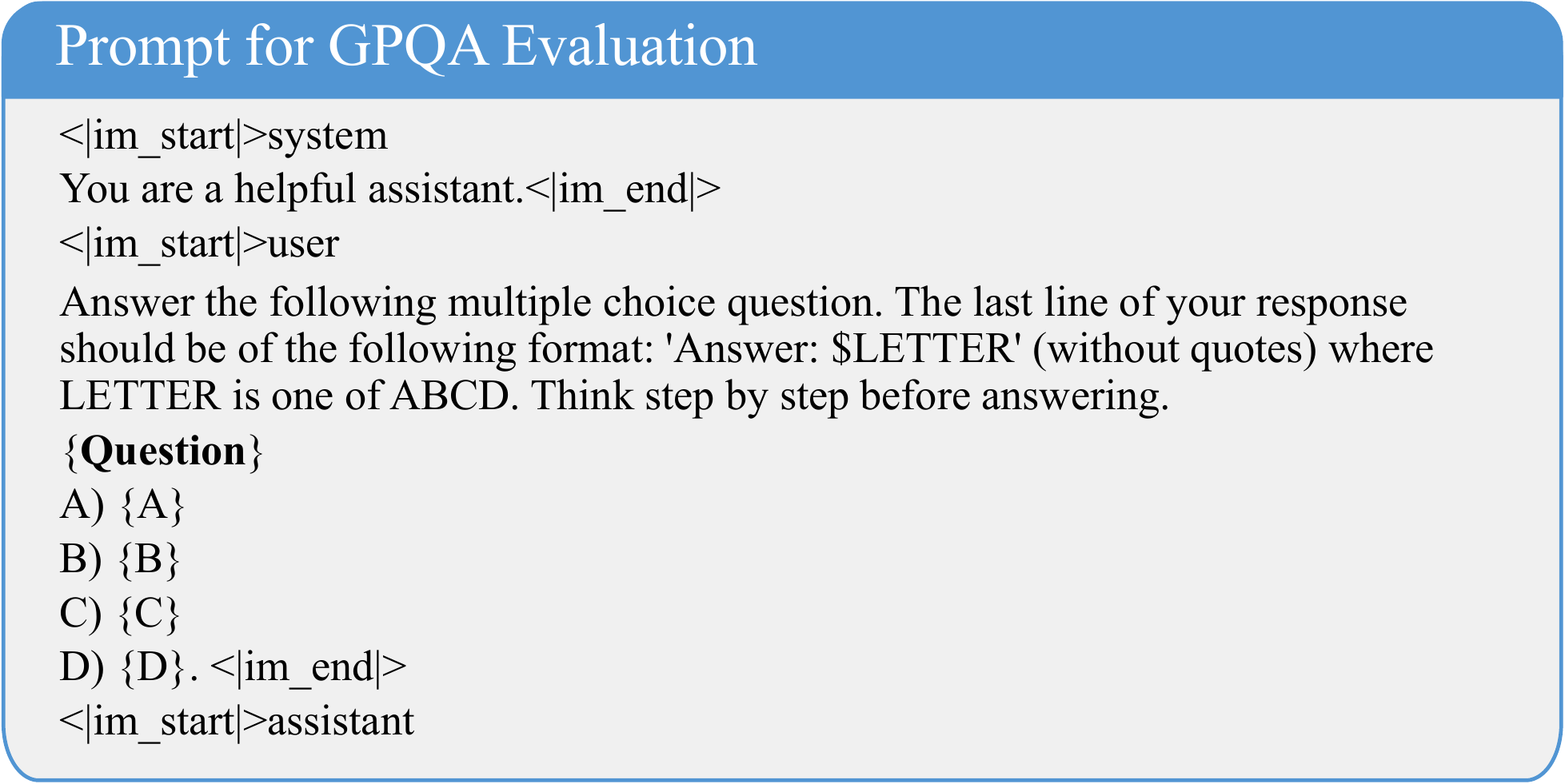}
  \end{subfigure}
\caption{
Prompt for \textbf{GPQA Evaluation}. We follow the prompt template settings in Lighteval.
}
\label{fig:gpqa_prompt}
\end{figure}

\begin{figure}[h]
  \vspace{-5pt}
  \centering
  \begin{subfigure}[b]{0.8\textwidth}
    \includegraphics[width=\textwidth]{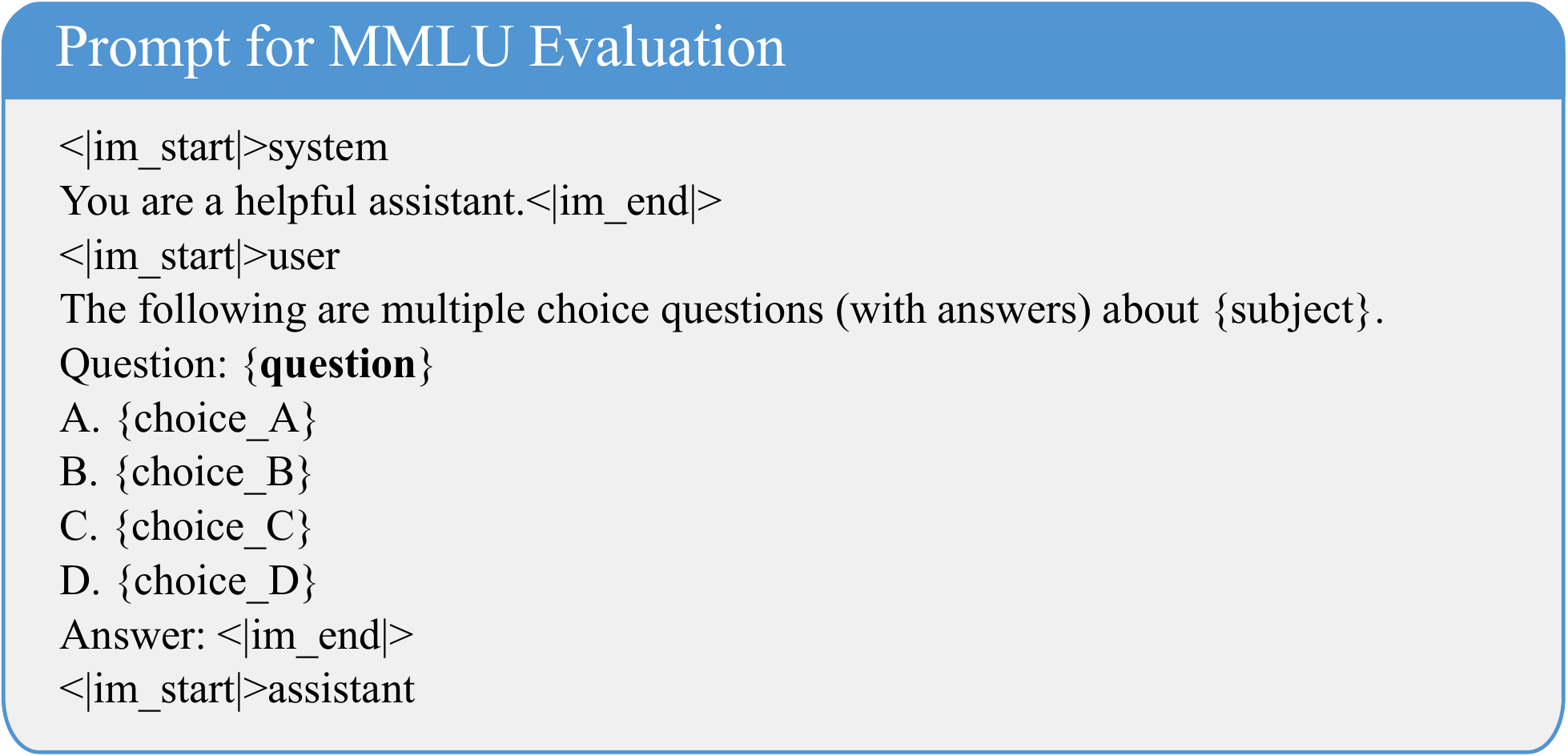}
  \end{subfigure}
\caption{
Prompt for \textbf{MMLU Evaluation}. We follow the prompt template settings in Lighteval.
}
\label{fig:mmlu_prompt}
\end{figure}

\begin{figure}[h]
\vspace{-5pt}
  \centering
  \begin{subfigure}[b]{0.8\textwidth}
    \includegraphics[width=\textwidth]{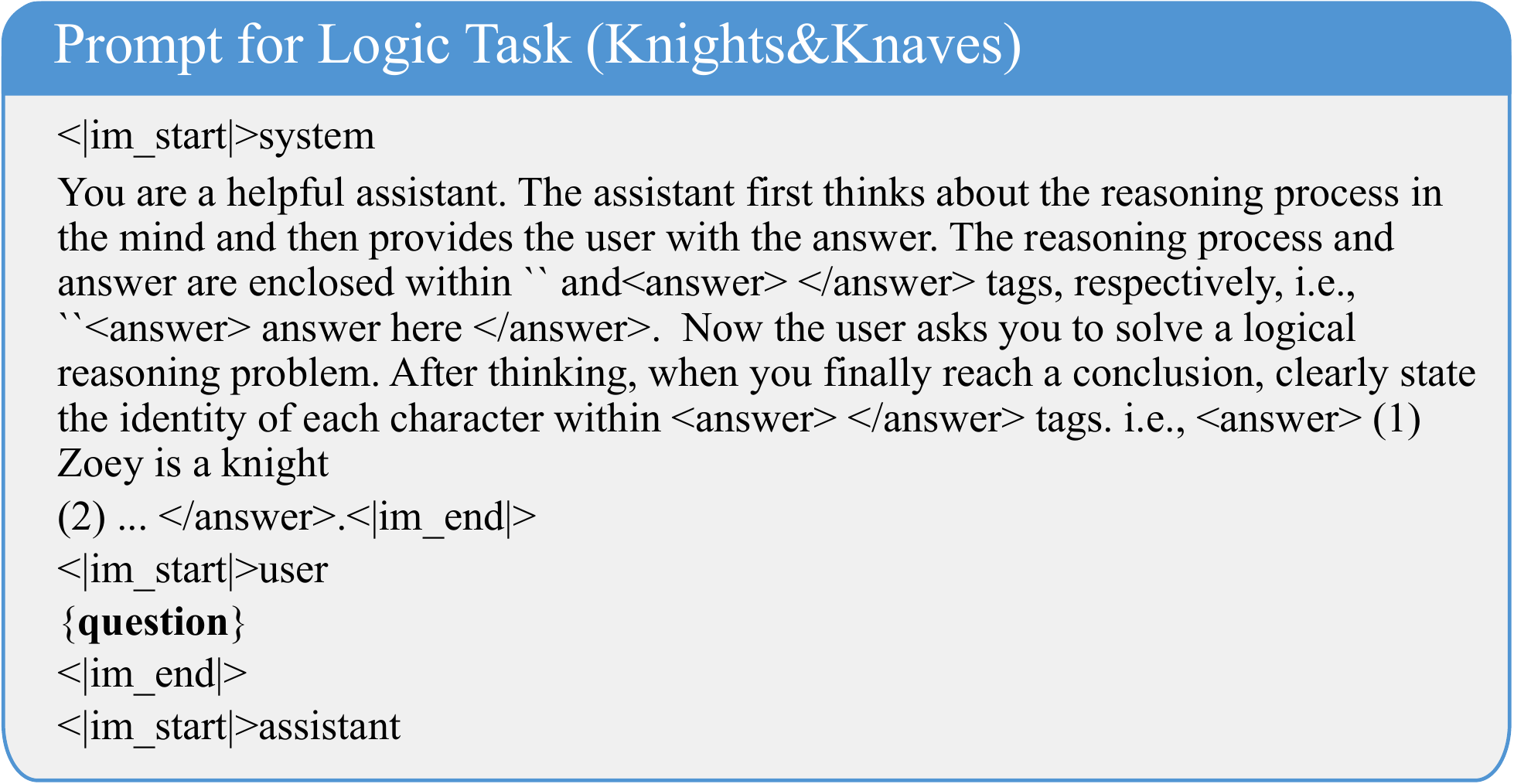}
  \end{subfigure}
\caption{
Prompt for \textbf{Logic Task Training and Evaluation}. We follow the prompt template settings in LogicRL.
}
\label{fig:logic_prompt}
\end{figure}

\begin{figure}[h]
\vspace{-5pt}
  \centering
  \begin{subfigure}[b]{0.8\textwidth}
    \includegraphics[width=\textwidth]{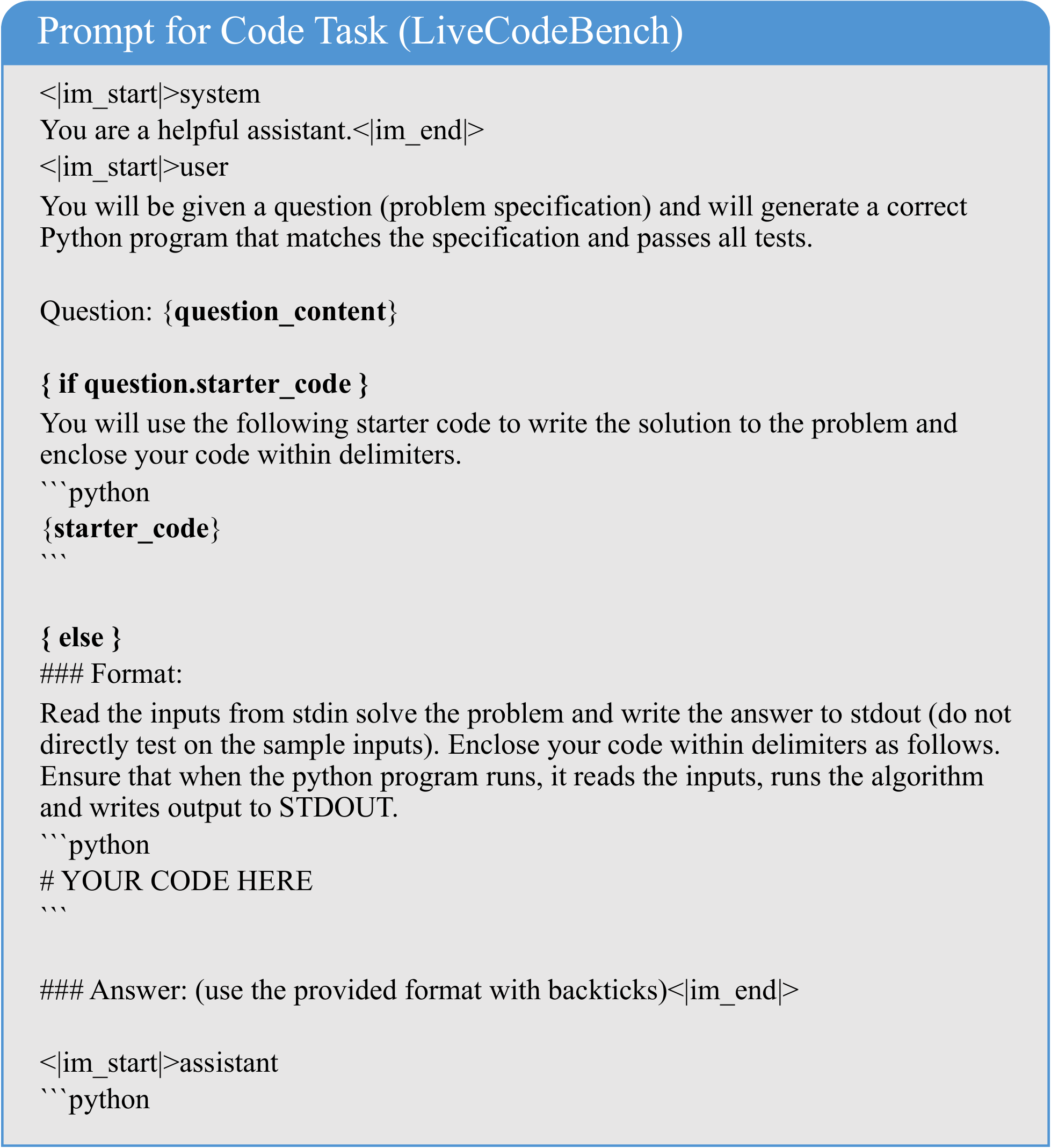}
  \end{subfigure}
\caption{
Prompt for \textbf{Code Task Training and Evaluation}. We follow the prompt template settings in LiveCodeBench and deepcoder.
}
\label{fig:code_prompt}
\end{figure}

\end{document}